\documentclass[english]{article}
\usepackage[T1]{fontenc}
\usepackage[latin9]{inputenc}
\usepackage{verbatim}
\usepackage{float}
\usepackage{amsmath}
\usepackage{amsthm}
\usepackage{amssymb}
\usepackage{stackrel}
\usepackage{setspace}
\usepackage{wasysym}
\usepackage{booktabs}

\makeatletter

\floatstyle{ruled}
\newfloat{algorithm}{tbp}{loa}
\providecommand{\algorithmname}{Algorithm}
\floatname{algorithm}{\protect\algorithmname}

\theoremstyle{plain}
\newtheorem{thm}{\protect\theoremname}
\theoremstyle{remark}
\newtheorem{rem}{\protect\remarkname}
\theoremstyle{definition}
\newtheorem{defn}{\protect\definitionname}
\theoremstyle{plain}
\newtheorem{prop}{\protect\propositionname}
\theoremstyle{definition}
 \newtheorem{example}{\protect\examplename}
\theoremstyle{plain}
\newtheorem{assumption}{\protect\assumptionname}
\theoremstyle{plain}
\newtheorem{lem}{\protect\lemmaname}
\theoremstyle{remark}
\newtheorem{claim}{\protect\claimname}
\usepackage{graphicx}
\usepackage{geometry}

\usepackage{color}   %May be necessary if you want t color links

\usepackage{hyperref}

\DeclareMathOperator*{\argmin}{argmin}

\AtBeginDocument{%
\let\ref\autoref
}

\def\equationautorefname~#1\null{{\color{black}(}#1{\color{black})}\null}

\hypersetup{
    colorlinks=true, %set true if you want colored links
    linktoc=all,     %set to all if you want both sections and subsections linked
    linkcolor=red,  %choose some color if you want links to stand out
  }

\usepackage[nameinlink]{cleveref}

\makeatother

\usepackage{babel}
\usepackage[authordate,bibencoding=auto,strict,backend=biber,natbib]{biblatex-chicago}

\let\cite\citet

\providecommand{\assumptionname}{Assumption}
\providecommand{\claimname}{Claim}
\providecommand{\definitionname}{Definition}
\providecommand{\examplename}{Example}
\providecommand{\lemmaname}{Lemma}
\providecommand{\propositionname}{Proposition}
\providecommand{\remarkname}{Remark}
\providecommand{\theoremname}{Theorem}

\begin{document}
\title{Robust Matrix Completion with Mixed Data Types}
\author{Daqian Sun, Martin T. Wells\thanks{
Wells$'$ research was partially supported by NIH U19 AI111143, NSF DMS-1611893, PCORI IHS-2017C3-8923, and Cornell's Institute for the Social Sciences project on Algorithms, Big Data, and Inequality.
} \\
    \{ds653, mtw1\}@cornell.edu \\
    Department of Statistics and Data Science \\
    Cornell University \\}
\maketitle

% rewrote abstract
\begin{abstract}
We consider the matrix completion problem of recovering a structured
low rank matrix with partially observed entries with mixed data types.
Vast majority of the solutions have proposed computationally feasible
estimators with strong statistical guarantees for the case where the
underlying distribution of data in the matrix is continuous.
A few recent approaches have extended using similar ideas these estimators
to the case where the underlying distributions belongs to the exponential
family. Most of these approaches assume that there is only one underlying
distribution and the low rank constraint is regularized by the matrix
Schatten Norm. We propose a computationally feasible
statistical approach with strong recovery guarantees along with an algorithmic framework
suited for parallelization to recover a low rank matrix with partially
observed entries for mixed data types in one step. We also provide
extensive simulation evidence that corroborate our theoretical
results.
\end{abstract}

\section{Introduction}

The matrix completion problem is related to recovering a low-rank
matrix from an observed subset of its entries and was initially shown to be solvable with strong theoretical guarantees \cite{candesPowerConvexRelaxation2010}, subsequently many algorithmic frameworks have been proposed for a variety of data settings  \cite{caiMatrixCompletionMaxnorm2016a,kloppNoisyLowrankMatrix2014,lafondLowRankMatrix2015,udellGeneralizedLowRank2016}.
However, few of these extensions address the matrix completion problem when
the underlying data are mixed data types. On the other hand, mixed
typed data matrices are quite common in real world applications. For example,
the data matrix could have count and binary data as well as continuous
entries. For instance, in recommended systems, the numerical ratings and like/dislike are two different data types but it is quite likely that both entries will be stored together. In this paper we propose a novel scalable algorithmic framework that solves the matrix completion problem for mixed data and provides provable recovery guarantees. 

The original problem formulation of matrix completion with a rank constraint
is computationally challenging and was in fact shown to be NP-hard \cite{vandenbergheSemidefiniteProgramming1996}. On the other
hand, a convex relaxation version of this problem which uses nuclear
norm as a surrogate for rank function gained attention because nuclear
norm was shown to be a convex envelop of the rank function \cite{maryamMatrixRankMinimization2002}.
A series of strong recovery guarantees were given by \cite{candesPowerConvexRelaxation2010,candesExactLowrankMatrix2008}.  Subsequent articles crystallized the canonical concepts and gave proofs of the main results that were simplified along with sharpened guarantee bounds \cite{rechtSimplerApproachMatrix2011}.
From that point on, an expanding literature proposed faster
algorithms. The primary bottleneck in the traditional
convex algorithm lies in the use of an eigenvalue decomposition or
singular value decomposition (SVD) in every iteration. Because of this constraint, non-convex algorithm have also been intensely studied. It has been
shown that with proper initialization (usually the SVD of the observed
matrix), one can obtain good recovery results with high probability
using an alternating minimization type algorithm which does not require eigenvalue decomposition
\cite{hardtUnderstandingAlternatingMinimization2014,jainLowrankMatrixCompletion2013}.

While fast computational methods abound, most of them only have provable theoretical recovery guarantees in the continuous data setting. Roughly speaking, for these non-convex methods, one implicitly assumes that the underlying distribution of the data is Gaussian. Whether the theoretical recovery guarantees of the current fast non-convex methods can be extended to the more
general case where the distribution of the matrix is not necessary
continuous is still an open question. On the other hand, the problem
of matrix completion in the more general setting has been partially
solved using a convex optimization perspective.  \cite{davenport1BitMatrixCompletion2014} showed that using the maximum likelihood principle, partially
observed binary low rank matrix could be recovered by optimizing a
convex objective using spectral gradient methods. More general results follow for more a more general family of distributions \cite{gunasekarExponentialFamilyMatrix2014,lafondLowRankMatrix2015},
it was shown that instead of binary data, one could recover with strong
theoretical guarantee a low rank matrix whose data follows a distribution in the exponential family. \cite{kloppAdaptiveMultinomialMatrix2015} considered the 
multinomial distribution and it was also shown to have a
theoretical recovery bound. \cite{caoPoissonMatrixCompletion2015}
showed that one could derive, using a different approach, a similar recovery
bound to that of \cite{gunasekarExponentialFamilyMatrix2014} and
\cite{lafondLowRankMatrix2015} in the Poisson distribution setting.

It is worth noting that most of the convex algorithms are based on
nuclear norm relaxation. Recent work by \cite{fangMaxnormOptimizationRobust2018,caiMatrixCompletionMaxnorm2016a}
have given empirical evidence that the matrix max norm works better
than nuclear norm when the sampling scheme is not uniform. Several
extensions of this result also appeared subsequently. \cite{caiMaxNormConstrainedMinimization2013}
presented a novel approach to recover, with theoretical guarantee, a
binary matrix using max-norm relaxation.  \cite{fangMaxnormOptimizationRobust2018}
showed that it is possible to use a hybrid of max-norm and nuclear
norm to recover a continuous valued matrix and in the same paper, it was shown that a 
Alternating Direction Method of Multipliers (ADMM) algorithm is a viable approach to solve max-norm related matrix completion problems with reasonably large input sizes. 
 
The problem of matrix completion when the observed matrix has mixed data types has been essentially overlooked. As mentioned previously, \cite{udellGeneralizedLowRank2016}
tried to use an alternating minimization approach to solve this problem. But under
such framework the only recovery guarantee that were known is the
Gaussian case. \cite{gunasekarConsistentCollectiveMatrix2015}
showed that such problem, while ill-formed in general, could be solved
when some extra conditions are imposed. Recently, \cite{alayaCollectiveMatrixCompletion2019} studied
the case of mixed data types with convex optimization
using nuclear norm regularization.  These authors considered the
mixed distributions first to be a mixture of exponential family distributions
and showed that it could relaxed to any distributions that satisfies
a certain Lipchitz condition. 

The problem we address in this article is the following.  Suppose we are given some matrix $M\in\mathbb{R}^{n_{1}\times n_{2}}$ and
some observed entries $(M_{ij})_{(i,j)\in\Omega}$, where $\Omega$
is the observed sets with $\left|\Omega\right|\ll n_{1}n_{2}.$ Also
the entries of $M$ has different data types by the columns (this condition could be relaxed to different deterministic index groups) $M_{\mathcal{C}_{i}}\sim T_{i}$, where $\mathcal{C}_{i}$
is the collection of columns of $M$ and $T_{i}$ represents a data
type, which is often chosen from the continuous, binary, or count
 types. Also we assume that
\begin{enumerate}
\item $M$ is approximately low rank.
\item The missing values could follow other schemes than missing at random, that is, the sampling scheme is can be non-uniform.
\end{enumerate}

Our goal in this article is to develop an analytic framework to study the recovery $M$ as accurately possible. To formulate it precisely, we solve the following optimization program:
\begin{align}
\text{minimize\ } & \left\Vert X-M\right\Vert _{F}\label{eq:rank min}\\
\text{subject to\ } & X\text{ is low rank}\text{ and }X_{\mathcal{C}_{i}}\sim T_{i}.
\end{align}
Furthermore, we would like to be able to control the rank of the recovered
matrix.

A brief overview our main results here and give an a formal
presentation in \Cref{sec:theoretical-properties}.
\begin{thm}
\label{thm:imprecise-recovery} If we choose $\lambda$ and $\lambda_{\text{max}}$,
where $\lambda$ and $\lambda_{\max}$ are regularization parameters as: 
\begin{equation}
\lambda_{*}=2c\frac{(U_{\gamma}\lor K)(\sqrt{n_{1}+N_{2}}+(\log(n_{1}\lor N_{2}))^{\frac{3}{2}}}{n_{1}N_{2}}
\end{equation}
and $\frac{\lambda_{\text{max}}}{\lambda_{*}}\leq\kappa,$ where $\kappa>0$
is some constant. Then the following recovery guarantees hold for
our proposed algorithm is
\begin{equation}
\frac{1}{n_{1}N_{2}}\left\Vert \widehat{\Theta}-\Theta\right\Vert _{F}^{2}\apprle\frac{\mathrm{rank}(\Theta)(n_{1}\lor N_{2})}{p^{2}n_{1}N_{2}}.\label{eq:imprecise-bound}
\end{equation}
\end{thm}

In view of \ref{thm:imprecise-recovery}, in order to get a small
estimation error, $p$ should be larger (up to multiplicative constant)
than $\text{rank}(\Theta)/(n_{1}\land N_{2})$ and the expected number
of observations $n$ should follow $n\geq C\text{rank}(\Theta)(n_{1}\land N_{2})$,
where $C$ is some large constant. The inequality in \ref{eq:imprecise-bound} means up to some term
that is $o(1)$.

Before stating and proving the main result, some pertinent topics will be reviewed.
A brief review of the exponential family of distributions is in \Cref{subsec:review-exponential-family}. One key property used implicitly frequently in our result is the mean parametrization, which was detailed in \Cref{subsec:mean-param}.  A
brief review of definitions of max-norm and nuclear norm is given
in \Cref{subsec:max-norm-nuclear-norm}. The specification, assumptions
and estimation procedure of our model is detailed in \Cref{sec:model-specification}.
In \Cref{sec:theoretical-properties} we state our main result whose
proof will be fully presented in \Cref{sec:appendix-theoreical-results-proof}.
Detailed description of our proposed algorithm along with its mathematical
properties are in the appendix. %\Cref{sec:appendix-theoreical-results-proof}%
The results of our extensive numerical experiments on simulated data are presented in \Cref{sec:numerical-experiments}. Finally, auxiliary lemmas and additional details are listed in the appendix. %\Cref{sec:appendix-lemmas}%

\section{Preliminaries}

In general, the two major approaches to handling data types in statistics are parametric
and non-parametric. In the parametric approach, we assume the data
follows a distribution that is specified up to a finite dimensional parameter, which in most cases could be represented
uniquely by its induced probability measure. Whereas nonparametric models are often indexed by a infinite dimensional family.  In this paper, we adopt the parametric approach to represent different data types. Specifically,
we restrict the distribution of the underlying data to be in the exponential family, of which we will provide a brief review.  Much of the review material in this section is adapted from \cite{wainwrightGraphicalModelsExponential2007}, we refer interested readers to the original paper for more details.

\subsection{Exponential Family\label{subsec:review-exponential-family}}

Given a random element $X=(X_{1},...,X_{n})\in\otimes_{i=1}^{n}\mathcal{X}_{i}$
where $\mathcal{X}_{i}$ is some arbitrary space with induce probability
measure $\mathbb{P}$. Let $\phi=(\phi_{\alpha},\alpha\in\mathcal{I})$
be a collection of functions of $\phi_{\alpha}:\mathcal{X}^{m}\rightarrow\mathbb{R},$ sometimes
known as \textbf{sufficient statistics} or \textbf{potential functions}.
Here $\mathcal{I}$ is an index set with $|\mathcal{I}|=d$ such that
$\phi(X):\otimes_{i=1}^{n}\mathcal{X}\rightarrow\mathbb{R}^{d}$ is
a vector valued mapping. For a given choice of $\phi(X)$, we further
associate with it another set of vector $\theta\in\mathbb{R}^{d}$,
which is often called \textbf{canonical parameters}. With this definition,
we can then have a relatively generalized definition of exponential
family.
\begin{defn}
The exponential family induced by $\phi$ is a family of probability
density functions (Radon Nikodym derivatives taken with respect to
$d\nu$, a base measure) of the form 
\[
\frac{d\mathbb{P}}{d\mathbb{\nu}}=p_{\theta}(x_{1},x_{2},...,x_{n})=\exp\left(\left\langle \theta,\phi(x)\right\rangle -A(\theta)\right).
\]
\end{defn}
The quantity $A(\theta)$, known as the log-partition function or
cumulant function, is defined by the following integral: 
\begin{equation} \label{partition_fun}
A(\theta)=\log\int_{\mathcal{X}^{m}}\exp\left\langle \theta,\phi(x)\right\rangle \nu(dx).
\end{equation}
This integral, if finite, acts as a normalizing factor for
the density function $p_{\theta}$. Holding the set of $\phi$ fixed,
each parameter vector $\theta$ corresponds to a particular member
$p_{\theta}$ of the family. Since $A(\theta)$ is not always finite,
the set of parameter $\theta$ of interests is the one corresponding
to a finite log-partition function, i.e. belong to the set 
\begin{equation}
\Omega:=\{\theta\in\mathbb{R}^{d}\vert A(\theta)<+\infty\}.
\end{equation}
From now on, unless it's otherwise defined, we use $\mathcal{E}$
to denote a exponential family distribution.

%\subsubsection{Topologies of $\Omega$}

While $\Omega$ is always well-defined because it can be viewed as
a pull back of a measurable function from a Borel set, it does not
always possess nice topological properties. One example is the case
where $\Omega$ is a closed set. This is not ideal because when $\theta$
is on the boundary, its $\epsilon$-ball is not properly contained
in the space, thus rendering the limiting behaviors irregular. Although
bad cases such as a closed $\Omega$ do exists, they are mostly for
pedological purposes. In turns out when $\Omega$ is an open set,
it behaves nicer analytically. Thus, we often say an exponential family
for which the domain $\Omega$ is an open set is a \textbf{regular
exponential family}. Almost all of the common distributions that we
encounter in the exponential family is regular. 

%\paragraph{Minimal representation}

Given an exponential family with a vector of sufficient statistics
$\phi,$ if there does not exists a non-zero vector $a\in\mathbb{R}^{d}$
such that the linear combination 
\begin{equation}
\left\langle a,\phi(x)\right\rangle =\sum_{\alpha\in\mathcal{I}}a_{\alpha}\phi_{\alpha}(x)
\end{equation}
is equal to a constant, then we say this exponential has a minimal
representation. The notion of minimal representation addresses the
problem of identifiability. In other words, with a minimal representation,
the canonical parameter $\theta$ associated with each distribution
is unique. 

%\paragraph{Over-complete representation }

The notion of over-complete representation is the analog to minimal
representation. With a over-complete representation, there is a non-zero
vector $a\in\mathbb{R}^{d}$ such that $\left\langle a,\phi(x)\right\rangle =c$
for some constant $c\in\mathbb{R}.$ As one might expect, this might
causes problems in identifiability. Indeed, for a member of the exponential
family with over-complete representation, the canonical parameter
$\text{\ensuremath{\theta}}$ associated to it is no longer unique, instead, there is an entire affine set of $\theta$ for it.

\subsubsection{Mean parameterization and the log partition function\label{subsec:mean-param}}

In turns out that many important parametric statistical inference
problems are related to the relationship between the canonical parameters
and mean parameters of distributions in the exponential family. In
the context of the problem at hand, the connection could formulated
as follows: suppose we have observed a low rank matrix $M$ with missing
entries with mixed exponential distributions. What is the most likely
recovery, $\hat{M}$? Assuming the underlying distribution does not
change much, then a natural way recover the matrix is to resample
it for many times and take its mean. However, it is not feasible to
sample the such matrix because we are not aware of the exact parameter
of the underlying distribution. Therefore, it is natural then to ask
for the most likely parameter value given the current observation,
which is a classical maximum likelihood estimation problem. For more
examples, see \cite{wainwrightGraphicalModelsExponential2007}. 

It is then natural to explore the relationship between the canonical
parameter of a exponential distribution and its corresponding mean.
To start with, we state the following result. 
\begin{prop}
\label{prop:exp-family-cumulant-mean}The cumulant function in \ref{partition_fun}
%\begin{equation}
%A(\theta):=\log\int_{\mathcal{X}^{m}}\exp\left\langle \theta,\phi(x)\right\rangle \nu(dx)
%\end{equation}
associated with any regular exponential family has the following properties.
\end{prop}
\begin{enumerate}
\item It has derivatives of all orders on its domain $\Omega$ and 
\begin{align}
\frac{\partial A}{\partial\theta_{\alpha}}(\theta) & =\mathbb{E}_{\theta}\left[\phi_{\alpha}(X)\right]=\int\phi_{\alpha}(x)p_{\theta}(x)d\nu,\\
\frac{\partial^{2}A}{\partial\theta_{\alpha}\partial\theta_{\beta}}(\theta) & =\mathbb{E}_{\theta}[\phi_{\alpha}(X)\phi_{\beta}(X)]-\mathbb{E}_{\theta}\left[\phi_{\alpha}(X)\right]\mathbb{E}_{\theta}\left[\phi_{\beta}(X)\right].
\end{align}
\item $A$ is a convex function of $\theta$ on its domain $\Omega$ and
strictly so if the representation if minimal. 
\end{enumerate}
The proof of \ref{prop:exp-family-cumulant-mean} is standard and
uses the dominated convergence theorem. This proposition builds a forward
mapping from the canonical parameter space to the mean parameter space,
which is the gradient map of $A$. In fact, the following result shows
the mapping is surjective with some mild regularity conditions. 
\begin{thm}
Given an exponential family distribution, $\mathcal{E}$, with a sufficient
statistic $\phi$, then
\begin{enumerate}
\item the gradient mapping $\nabla A:\Omega\rightarrow\mathcal{M}$ is injective
if and only $\mathcal{E}$ is minimal, and
\item the gradient mapping $\nabla A:\Omega\rightarrow\mathcal{M}$ is surjective
for $\mu\in\mathrm{Int}(\mathcal{M})$ if $\mathcal{E}$ is minimal.
\end{enumerate}
\end{thm}
%

%\subsubsection{Other Connections}
Another important connection between the mean parameterization and the log partition follows using duality theory.  The conjugate dual function to $A$, which we denote by $A^{*}$, is defined as follows:
\begin{equation}
A^{*}(\mu)=\sup_{\theta\in\Omega} \left\{ \left\langle \theta,\mu\right\rangle -A(\theta)\right\},
\end{equation}
where $\mu \in \mathbb{R}^{d}$ is a fixed vector of so-called dual variables of the same
dimension as $\theta$. These dual variables turn out to have a natural interpretation as mean parameters.  The theorems below connect the conjugate dual function of $A$ to the Shannon Entropy of $p_{\theta(\mu)}$, $H(p_{\theta(\mu)})$ and the mean parametrization.
\begin{thm}
For any $\mu\in\mathrm{Int}(\mathcal{M})$ let $\theta(\mu)$ denote
the unique canonical parameter satisfying the dual matching condition.
The conjugate dual function $A^{*}$ takes the form 
\begin{equation}
A^{*}(\mu)=\begin{cases}
-H(p_{\theta(\mu)}) & \text{if }\mu\in\mathrm{Int}(\mathcal{M})\\
+\infty & \text{if }\mu\notin\overline{\mathcal{M}}
\end{cases}.
\end{equation}
\end{thm}
\begin{thm}
The log-partition function has the following variational representation
\begin{equation}
A(\theta)=\sup_{\mu\in\mathcal{M}}\left\{ \left\langle \theta,\mu\right\rangle -A^{*}(\mu)\right\} .
\end{equation}
Moreover, for all $\theta\in\Omega$, the supremum is attained uniquely
at $\mathrm{Int}(\mathcal{M})$ and is specified by the moment matching conditions
\begin{equation}
\mu=\int_{\mathcal{X}^{m}}\phi(x)p_{\theta}(x)v(dx)=\mathbb{E}_{\theta}[\phi(X)].
\end{equation}
\end{thm}

%\subsubsection{Connection to Bregman Divergence }

The final connection with the log partition function connects Bregman and Kullback-Leibler divergences. We introduce the notion of Bregman Divergence since its connection
with the Kullback-Leibler divergence of exponential family which will be
used frequently in the proof of the recovery upper bound.
\begin{defn}
Let $S$ be a closed convex subset of $\mathbb{R}^{m}$ and $\Phi:S\subset\text{dom}(\Phi)\rightarrow\mathbb{R}$
a continuously differentiable and strictly convex function. The Bregman
divergence of $\Phi$, denoted as $d_{\Phi}:S\times S\rightarrow[0,\infty)$
is defined as 
\begin{equation}
d_{\Phi}(x,y)=\Phi(x)-\Phi(y)-\left\langle x-y,\nabla\Phi(y)\right\rangle .
\end{equation}
The next proposition from \cite{wainwrightGraphicalModelsExponential2007}
gives the connection of Kullback-Leibler divergence of distributions
in exponential family and Bregman divergence.
\end{defn}
\begin{prop}
For exponential family distributions, the Bregman divergence corresponds
to the Kullback-Leibler divergence with $\Phi=A.$
\end{prop}

\subsubsection{Common Examples of Members of the Exponential family}
\begin{example}[Gaussian]
The Gaussian distribution is widely used in modeling continuous data.
A Gaussian random variable with mean $\mu$ and variance $\sigma$
has the following form of density:
\begin{equation}
f_{X}(x\vert\mu,\sigma)=\frac{1}{\sqrt{2\pi\sigma^{2}}}\exp\left\{ -\frac{(x-\mu)^{2}}{2\sigma^{2}}\right\} \; (\mu,\sigma)\in\mathbb{R}\times\mathbb{R}^{+}
\end{equation}
which in its exponential family form, could be written as 
\begin{align}
f_{X}(x\vert\mu,\sigma) & =\frac{1}{\sqrt{2\pi\sigma^{2}}}\exp\left\{ -\log\sigma-\frac{x^{2}}{2\sigma^{2}}+\frac{\mu x}{\sigma^{2}}-\frac{\mu^{2}}{2\sigma^{2}}\right\} \\
 & =\frac{1}{\sqrt{2\pi\sigma^{2}}}\exp\left\{ \left\langle \theta,\phi(x)\right\rangle -A(\theta)\right\} ,
\end{align}
where $\theta=(\mu/\sigma^{2},-1/(2\sigma^{2}))$ and $\phi(x)=(x,x^{2}).$
The canonical parameterization of the Gaussian distributions is the mean parametrization. 
\end{example}
\begin{example}[Gamma]
The Gamma distribution is a two parameter continuous probability
distribution. It is often used to model the size of insurance claims
and rainfalls, see \cite{hewittMethodsFittingDistributions1979}, \cite{husakUseGammaDistribution2007}. In wireless communication, the Gamma distribution
is used to model the multi-path fading of signal power. It also has
wide application in the field of neuroscience, genomics, and oncology.
A random variable $X$ is said to follow Gamma distribution with parameter
$\alpha$ and scale $\theta$ if it has probability density function
\begin{equation}
f_{X}(x|\alpha,\theta)=\frac{1}{\Gamma(\alpha)\theta^{\alpha}}x^{\alpha-1}e^{-\frac{x}{\theta}}\text{ for }(\alpha,\theta)\in\mathbb{R}^{+}\times\mathbb{R}^{+}
\end{equation}
In its exponential family canonical form, we have 
\begin{align}
f_{X}(x\vert\alpha,\theta)  %=\exp\left(\log\left(\frac{1}{\Gamma(\alpha)\theta^{\alpha}}x^{\alpha-1}e^{-\fra%c{x}{\theta}}\right)\right)\\
  =\exp\left(-\log\Gamma(\alpha)-\alpha\log\theta+(\alpha-1)\log x-\frac{x}{\theta}\right).
\end{align}
Note that $\mathbb{E}[X]=\theta\alpha.$ Hence, the mean parameterization
for Gamma distribution can be written as 
\begin{align}
f_{X}(x\vert\mu,\alpha)  %=\exp\left(x\left(-\frac{1}{\mu/\alpha}\right)-\log\Gamma(\alpha)-\alpha\log\fra%c{\mu}{\alpha}+(\alpha-1)\log x\right)\\
  =\exp\left(x\left(-\frac{\alpha}{\mu}\right)+(\alpha-1)\log x-\log\Gamma(a)-\alpha\log\mu+\alpha\log\alpha\right).
\end{align}
\end{example}
\begin{example}[Bernoulli]
The Bernoulli distribution is the most used distribution to model binary data.
 %, which frequently occurs in classification problems, surveys. It could also
%be used to model black-and-white images.
In its most common form, the probability mass function (p.m.f) for a Bernoulli random variable as follows 
\begin{equation}
\mathbb{P}(X=x\vert p)=p^{x}(1-p)^{1-x},\text{ for }(x,p)\in\{0,1\}\times[0,1]
\end{equation}
which in its exponential family form, could be rewritten as 
\begin{equation} \label{bern_pmf}
\mathbb{P}(X=x\vert p)=\exp\left\{ x\log\frac{p}{1-p}+\log(1-p)\right\} .
\end{equation}
Since $\mathbb{E}[X]=p,$ \ref{bern_pmf} is also the mean parameterization with $p=\mu$.
%Bernoulli distribution could be re-parameterized
%in its mean as follows:
%\begin{equation}
%\mathbb{P}(X=x\vert\mu)=\exp\left\{ i\log\frac{\mu}{1-\mu}+\log(1-\mu)\right\} .
%\end{equation}
\end{example}
\begin{example}[Poisson]
The Poisson distribution is frequently used in fitting count data. A
Poisson random variable is often defined by its being equipped with
the following p.m.f 
\begin{equation}
\mathbb{P}(X=x\vert\lambda)=\lambda^{x}\frac{e^{-\lambda}}{x!}\text{ for }(x,\lambda)\in\mathbb{N}\times\mathbb{R}^{+}.
\end{equation}
In its exponential family canonical form, 
\begin{equation} \label{poisson_pmf}
\mathbb{P}(X=x\vert\lambda)=\frac{1}{x!}\exp\left\{ x\log\lambda-\lambda\right\} .
\end{equation}
Note that $\mathbb{E}[X]=\lambda,$ \ref{poisson_pmf} is also the mean parameterization with $\lambda=\mu$.
%\begin{equation}
%\mathbb{P}(X=x\vert\mu)=\frac{1}{x!}\exp\{x\log\mu-\mu\}\text{ for %}\mu\in\mathbb{R}^{+}.
%\end{equation}
\end{example}

\begin{example}[Negative Binomial]
The negative binomial distribution is often used to model the number of failures
before $r$th success in a stream of independent Bernoulli trials. It is also used as a alternative to the Poisson distribution for count data that accounts for possible over-dispersion via the representation as a Poisson-Gamma mixture. A most common way to parameterize negative binomial distribution is
by $r$, the number of success and the probability of success $p$,
namely, a random variable $X$ is said to have negative binomial distribution
if it has the following p.m.f 
\begin{equation}
\mathbb{P}(X=k)=\binom{k+r-1}{k}p^{r}(1-p)^{k},\text{ for }p\in[0,1],k\in\mathbb{N}.
\end{equation}
We note that such definition could be extended to $k\in\mathbb{R}$
with a slight extension of the definition: 
\begin{equation}
\mathbb{P}(X=k)=\frac{\Gamma(k+r)}{k!\Gamma(r)}p^{r}(1-p)^{k}\text{ for }p\in[0,1],k\in\mathbb{R}.\label{eq:extended-nb-pmf}
\end{equation}
It's easy to check that when $k\in\mathbb{N},$ the extended version
falls reduces to the original negative binomial distribution. It is well-known
that extended definition \ref{eq:extended-nb-pmf} could be interpreted
as a Poisson-Gamma mixture, which is sometimes useful in some model
fitting problems. Its mean parameterization of the following form (cf.
\ref{lem:mean-parametrization-negative-binomial}): 
\begin{equation}
\mathbb{P}(X=k\vert\mu,r)=\frac{\Gamma(k+r)}{\Gamma(r)k!}\left(\frac{r}{\mu+r}\right)^{r}\left(\frac{\mu}{\mu+r}\right)^{k},\text{ for }\mu\in\mathbb{R}^{+},r\in\mathbb{N}.
\end{equation}
\end{example}
% revised: bridge between scalar and matrix case
In the previous examples we have seen how the definition of exponential
family manifests in simple scalar random variables. To further illustrate
the usage of concept of exponential family in the setting of our problem
interest, we now showcase the random exponential family matrices. 
% revised: added pure form random matrix def
\begin{example}[Independent Exponential Family Random Matrix]
 \label{exa:exponential random matrix} Let $X\in\mathbb{R}^{m\times n}$
be a random matrix, where each entry $X_{ij}$ is drawn independently
from the same exponential family characterized by the density $\mathbb{P}(X_{ij}\vert\Theta_{ij})=h(X_{ij})\exp(X_{ij}\Theta_{ij}-G(\Theta_{ij}))$,
then it follows from factorization theorem, 
\[
\mathbb{P}(X\vert\Theta)=\prod_{i,j}h(X_{ij})\exp(X_{ij}\Theta_{ij}-G(\Theta_{ij})=h(X)\exp\left(\left\langle X,\Theta\right\rangle -G(\Theta)\right),
\]
where by slightly abuse of notation we denote $G:\mathbb{R}^{m\times n}\mapsto\mathbb{R}$
as $G(\Theta)=\sum_{ij}G(\Theta_{ij})$ and $h(X)=\prod_{ij}h(X_{ij}).$
\end{example}
The next example is important as it serves as the basis of our model formulation.
% revised: added heterogenous form random matrix
\begin{example}[Structurally Heterogeneous Exponential Family Random Matrix]
 Let $X=[X^{(1)}\ X^{(2)}\ ...\ X^{(k)}]$ be $\mathbb{R}^{m\times n},$
where $X^{(i)}\in\mathbb{R}^{m\times n_{i}}$ such that $\sum_{i=1}^{k}n_{i}=n$
be a random matrix consisting of column-wise disjoint sub-matrices
$X^{(i)},$ and for each $X^{(s)}$ is an independent exponential
family random matrix as in \ref{exa:exponential random matrix} with
density $\mathbb{P}(X^{(s)}\vert\Theta^{(s)})=h^{(s)}(X)\exp(\left\langle X^{(s)},\Theta^{(s)}\right\rangle -G^{(s)}(\Theta^{(s)})).$
Again, by independence and factorization theorem, we have that the
density for $X$ as 
\begin{equation}
\mathbb{P}(X\vert\Theta)=\prod_{s=1}^{k}h^{(s)}(X^{(s)})\exp\left(\left\langle X^{(s)},\Theta^{(s)}\right\rangle -G^{(s)}(\Theta^{(s)})\right).
\end{equation}
\end{example}

\subsection{Matrix Norms \label{subsec:max-norm-nuclear-norm}}

Both the nuclear norm and max norm will serve as important tools in the
derivation of a tractable formulation of the matrix completion problem.
Before we delve deeper, a few definitions are needed.
\begin{defn}
Let $A\in\mathbb{R}^{n_{1}\times n_{2}}$ and $A=U\Sigma V^{T}=\sum_{i=1}^{n_{1}\land n_{2}}\sigma_{i}u_{i}v_{i}^{T}$
be its singular value decomposition, then the nuclear norm of $A$
is defined as 
\begin{equation}
\left\Vert A\right\Vert _{*}=\sum_{i=1}^{n_{1}\land n_{2}}\sigma_{i}.
\end{equation}
\end{defn}

%\paragraph{Nuclear norm and rank}

Since the rank function could be defined as the \textquotedbl $\ell_{0}$  norm\textquotedbl{}
of the vector of singular values of a matrix, the nuclear norm, which
can be seen as the $\ell_{1}$ counterpart of the same concept, intuitively
should be a good approximation of the rank function. Formally, \cite{maryamMatrixRankMinimization2002}
showed that the convex envelop of $\text{rank}(X)$ for $X\in\left\{ X\in\mathbb{R}^{n\times m}:\left\Vert X\right\Vert \leq1\right\} $
is the nuclear norm.
\begin{defn}
Let $A\in\mathbb{R}^{n_{1}\times n_{2}}$. The max norm of $A$ is
defined as 
\begin{equation}
\left\Vert A\right\Vert _{\text{max}}=\min_{U,V\text{ s.t.}A=UV^{T}}\left\Vert U\right\Vert _{\ell_{2}\rightarrow\ell_{\infty}}\left\Vert V\right\Vert _{\ell_{2}\rightarrow\ell_{\infty}},
\end{equation}
where $\left\Vert \cdot\right\Vert _{\ell_{2}\rightarrow\ell_{\infty}}$
is the operator norm from $\ell_{2}$ to $\ell_{\infty}$ defined
by 
\begin{equation}
\left\Vert A\right\Vert _{\ell_{2}\rightarrow\ell_{\infty}}=\sup_{\left\Vert x\right\Vert _{2}\leq1}\left\Vert Ax\right\Vert _{\infty}.
\end{equation}
\end{defn}

%\paragraph{Max nom and rank}

The direct connection between max norm and rank is a bit technical.
So we avoid it here. Instead, we observe this connection by taking
a look at the connection between the max norm and the nuclear norm.
It is well known that nuclear norm has the following alternative representation:
\begin{equation}
\left\Vert A\right\Vert _{*}=\min_{\substack{\left\Vert u_{j}\right\Vert _{2}=\left\Vert v_{j}\right\Vert _{2}=1,\\
M=\sum_{j}\sigma_{j}u_{j}v_{j}
}
}\sum_{j}\left|\sigma_{j}\right|.
\end{equation}
On the other hand, \cite{jamesonSummingNuclearNorms1987} showed that
\begin{equation}
\left\Vert A\right\Vert _{\max}\asymp\min_{\substack{\left\Vert u_{j}\right\Vert _{\infty}=\left\Vert v_{j}\right\Vert _{\infty}=1,\\
M=\sum_{j}\sigma_{j}u_{j}v_{j}^{T}
}
}\sum_{j}\left|\sigma_{j}\right|,
\end{equation}
where the factor of equivalence is the Grothendieck's constant $K\in(1.67,1.79).$
Roughly, the similarity in representations suggests that
max norm may be a good approximation of rank function. For a more
precise characterization of the relationship between rank and the
max norm, see \cite{srebroRankTracenormMaxnorm2005}.

%\paragraph{Computational aspect}

Both the max and nuclear norms have their respective semi-definite program
representations. This makes the numerical computation easier, especially
in the case of max norm, where computing from the original definition
is NP-hard. Let $A\in\mathbb{R}^{n\times m}$ be an arbitrary matrix.
Then the nuclear norm of $A$ was be represented by \cite{maryamMatrixRankMinimization2002}
as the solution to the following semi-definite programs:
\begin{align}
\max\quad & A\bullet X\\
\text{subject to\ensuremath{\quad}} & \begin{bmatrix}I_{m} & Y\\
Y^{T} & I_{n}
\end{bmatrix}\succeq0
\end{align}
and its dual
\begin{align}
\min\quad & \text{tr}(W_{1})+\text{tr}(W_{2})\\
\text{subject to}\quad & \begin{bmatrix}W_{1} & -\frac{1}{2}A\\
-\frac{1}{2}A^{T} & W_{2}
\end{bmatrix}\succeq0.
\end{align}
Furthermore, $\left\Vert A\right\Vert _{*}\leq t$ if and only if
there exists $Y\in\mathcal{S}^{m}$ and $Z\in\mathcal{S}^{n}$ such that
\begin{equation}
\text{tr}(Y)+\text{tr}(Z)\leq2t,\begin{bmatrix}Y & X\\
X^{T} & Z
\end{bmatrix}\succeq0.
\end{equation}
On the other hand, \cite{srebroMaximummarginMatrixFactorization2004}
showed that $\left\Vert A\right\Vert _{\max}$ can be represented
as the solution to the following semi-definite program: 
\begin{align}
\min\quad & R\\
\text{subject to\ensuremath{\quad}} & \begin{bmatrix}W_{1} & A\\
A^{T} & W_{2}
\end{bmatrix}\succeq0,\\
 & \left\Vert \text{diag}(W_{1})\right\Vert _{\infty}\leq R,\left\Vert \text{diag}(W_{2})\right\Vert _{\infty}\leq R.
\end{align}
We will see later that this representation will facilitate the reformulation of our
objective function.

\section{Model Specification \label{sec:model-specification}}

\subsection{Set up and assumptions}

Let $\Theta=[\Theta_{\tau_{1}}\ \Theta_{\tau_{2}}\ ...\ \Theta_{\tau_{\left|\mathcal{T}\right|}}]\in\mathbb{R}^{n_{1}\times N_{2}}$
be the full matrix (unknown truth), where $[\Theta_{\tau}]_{\tau\in\mathcal{T}}\in\mathbb{R}^{n_{1}\times n_{2}^{\tau}}$
represent sub-matrices whose entries follow different distributions
in the exponential family, that is, 
\begin{equation}
\frac{d\mathbb{P}_{\Theta_{ij}^{\tau}}}{d\lambda}=h^{\tau}(\Theta_{ij}^{\tau})\exp(X_{ij}^{\tau}\eta_{ij}^{\tau}-G^{\tau}(\eta_{ij}^{\tau}))
\end{equation}
and $N_{2}:=\sum_{\tau\in\mathcal{T}}n_{2}^{\tau}.$ Let observed
matrix be $Y\in\mathbb{R}^{n_{1}\times N_{2}}.$ Additionally, we
assume the entries of $\Theta$ are uniformly bounded, that is, $\left\Vert \Theta\right\Vert _{\infty}\leq\gamma$
for some $\gamma\in\mathbb{R}.$ For ease of notation, let $\mathcal{C}(\gamma):=\left\{ X\in\mathbb{R}^{n_{1}\times N_{2}}:\left\Vert X\right\Vert _{\infty}\leq\gamma\right\} $
be the $\ell_{\infty}$ball with radius $\gamma$.

%\subsection{The Bernoulli Missingness Model}

Before we delve into estimation, we first address the procedure
with which the observed incomplete matrix, $Y$, is determined. Formally, $Y$
is generated by associating each full matrix $X_{ij}^{\tau}$ with a Bernoulli
random variable $\delta_{ij}^{\tau}\sim\text{Bin}(\pi_{ij}^{\tau})$
and let $Y_{ij}^{\tau}=\delta_{ij}^{\tau}X_{ij}^{\tau}.$ Here, $\pi_{ij}^{\tau}$
can be thought of as the sampling rate. In the easiest case, uniform
sampling scheme, we could consider $\pi_{ij}^{\tau}=\alpha\in(0,1)$,
where $\alpha$ is some constant. For an intuitive understanding for
this scheme, imagine we are scanning through $X$ entry by entry in
a row-major manner, for each entry, we stop and toss a coin which
has a probability of $\alpha$ landing on a head, and probability
of $1-\alpha$ on a tail. If landed on a head, we keep $X_{ij}^{\tau}$
the same; otherwise we let $X_{ij}^{\tau}=0.$ Notice that in this
example, we used the same coin throughout the double loop. In the
non-uniform sampling scheme, the same coin analog still holds with
one simple modification, we can potentially use a different coin with
different head probability for every entry at which we stop.

%\subsection{Assumptions}

We now introduce two mild but necessary assumptions for
our model. These two assumptions are common in previous literature,
cf. \cite{alayaCollectiveMatrixCompletion2019,gunasekarExponentialFamilyMatrix2014,lafondLowRankMatrix2015,kloppNoisyLowrankMatrix2014}
\begin{assumption}
\label{assu:sample}Each entry has a positive probability of being
observed, that is,
\begin{equation}
\min_{\tau}\min_{i,j\in[n_{1}]\times[n_{2}^{\tau}]}\pi_{ij}^{\tau}\geq p
\end{equation}
for $p\in(0,1).$
\end{assumption}
\begin{assumption}
\label{assu:curvature} The curvature of $A^{\tau}(x)$ is bounded,
that is 
\begin{align}
\sup_{\eta\in[-\gamma-\frac{1}{K},\gamma+\frac{1}{K}]}\left[\nabla^{2}G^{\tau}\right](\eta) & \leq U_{\gamma}^{2}\\
\inf_{\eta\in[-\gamma-\frac{1}{K},\gamma+\frac{1}{K}]}\left[\nabla^{2}G^{\tau}\right](\eta) & \leq L_{\gamma}^{2}.
\end{align}
\end{assumption}
\noindent Note that \ref{assu:sample} is natural in the sense that if there
are some entries with $0$ probability of being sampled, then the
problem could become completely intractable in the sense that if we
let a whole row to be unobserved then it would be possible that the
matrix is full rank and thus non-recoverable.  In addition,  \ref{assu:curvature} is an sufficient condition for $\Theta_{ij}^{\tau}$
to have uniformly bounded variance and sub-exponential tails, which
serve as a license that enables us to invoke concentration inequalities
in our proof. \cite{alayaCollectiveMatrixCompletion2019}
shows that a wide range of distributions satisfy Assumptions 1 and 2, some of these are reproduced in Table \ref{table:curvature}.

\begin{table}
\centering
\begin{tabular}[h]{ccc}
\toprule 
Model & $L_{\gamma}$ & $U_{\gamma}$ \\
\midrule
Normal & $\sigma^2$ & $\sigma^2$ \\ 
Binomial & $\frac{Ne^{-(\gamma+\frac{1}{K})}}{(1+e^{\gamma+\frac{1}{K}})^{2}}$ & $\frac{N}{4}$ \\ 
Gamma &$\frac{\alpha}{(\gamma+\frac{1}{K})^{2}}$ &$\frac{\alpha}{(\left|\gamma_{1}\right|\land\left|\gamma_{2}\right|)^{2}}$ \\ 
Negative binomial & $\frac{re^{-(\gamma+\frac{1}{K})}}{(1-e^{-(\gamma+\frac{1}{K})})^{2}}$ &$\frac{re^{(\gamma+\frac{1}{K})}}{(1-e^{(\gamma+\frac{1}{K})})^{2}}$ \\ 
Poisson & $e^{-(\gamma+\frac{1}{K})}$ & $e^{-(\gamma+\frac{1}{K})}$ \\
\bottomrule
\end{tabular}
\label{table:curvature}
\caption{Examples of $L_{\gamma}$ and $U_{\gamma}$ functions in Assumptions 1 and 2 for various member of the exponential family.}
\end{table}

\subsection{Estimation Procedure}

Since $\{\Theta_{ij}^{\tau}\}$ are independent, by construction, we
can write out the (normalized) negative log-likelihood function, $\ell(\Theta)$,
as 
\begin{equation}
-\frac{1}{n_{1}N_{2}}\sum_{\tau\in[\mathcal{T}]}\sum_{(i,j)\in[n_{1}]\times[n_{2}^{\tau}]}\delta_{ij}^{\tau}(Y_{ij}^{\tau}\Theta_{ij}^{\tau}-G^{\tau}(\Theta_{ij}^{\tau})).
\end{equation}
Using maximum likelihood principle, the straightforward approach is to let our estimator $\widehat{\Theta}$ be the solution to the
following program:
\begin{align}
\text{\ensuremath{\min_{\Theta\in\mathbb{R}^{n_{1}\times N_{2}}}\quad}} & \ell(\Theta)\\
\text{subject to}\quad & \text{rank}(\Theta)\text{ is low},\left\Vert \Theta\right\Vert _{\infty}\leq\gamma.
\end{align}

Since this program is non-convex, due to the nature of the
rank function, we consider a convex relaxation of the original problem
by nuclear norm:
\begin{align}
\text{\ensuremath{\min_{\Theta\in\mathbb{R}^{n_{1}\times N_{2}}}\quad}} & \ell(\Theta)\\
\text{subject to}\quad & \left\Vert \Theta\right\Vert _{*}\leq\gamma_{1},\left\Vert \Theta\right\Vert _{\infty}\leq\gamma.
\end{align}
Recent works have shown that nuclear norm alone doesn't perform well
in non-uniform sampling schemes. A max-norm regularization approach
is often used to address this issue  \cite{fangMaxnormOptimizationRobust2018,caiMaxNormConstrainedMinimization2013,caiMatrixCompletionMaxnorm2016a}.
However, using max-norm alone could lead to suboptimal recovery result;
therefore, we propose using a hybrid norm which combines the max norm
and nuclear, in our convex relaxation set up:
\begin{align}
\text{\ensuremath{\min_{\Theta\in\mathbb{R}^{n_{1}\times N_{2}}}\quad}} & \ell(\Theta)\label{eq:max nuc constrained program}\\
\text{subject to}\quad & \left\Vert \Theta\right\Vert _{*}\leq\gamma_{1},\left\Vert \Theta\right\Vert _{\text{max}}\leq\gamma_{2},\left\Vert \Theta\right\Vert _{\infty}\leq\gamma.
\end{align}
By convexity and strong duality, the admissible solution $\widehat{\Theta}$ of \ref{eq:max nuc constrained program} can also be obtained
by the following unconstrained program
\begin{align}
\widehat{\Theta} & =\stackrel[\Theta\in\mathcal{C}(\gamma)]{}{\text{argmin}}\ell(\Theta)+\lambda_{*}\left\Vert \Theta\right\Vert _{*}+\lambda_{\max}\left\Vert \Theta\right\Vert _{\text{max}}.\label{eq:p}
\end{align}

\section{Theoretical Properties\label{sec:theoretical-properties}}

\global\long\def\norm#1{\|#1\|}%

We now state the main result regarding the recovery of $\Theta.$
Due to spacing limitation, we state an imprecise version of the theorem
(ignoring multiplicative constant) and defer the precise versions
to Appendix A. 
\begin{thm}
\label{thm:2}Under \ref{assu:sample} and \ref{assu:curvature}, if we
set $\lambda_{*}\geq2\left\Vert \nabla\ell(\Theta|Y)\right\Vert $
and $\lambda_{\text{max}}\geq 0$, then with probability $1-\frac{4}{(n_{1}+N_{2})}$,
we have the following upper bounds, 
\begin{align}
 & \frac{1}{n_{1}N_{2}}\norm{\widehat{\Theta}}-\Theta_{\Pi,F}^{2}\\
\leq\  & \frac{C}{p}\max\Bigg\{ n_{1}N_{2}\mathrm{rank}(\Theta)\left(\frac{\lambda_{*}^{2}+\lambda_{*}\lambda_{\max}L_{\gamma}^{2}+L_{\gamma}^{4}\lambda_{\max}^{2}}{L_{\gamma}^{4}}+\left(1+\frac{\lambda_{\max}}{\lambda_{*}}+\frac{\lambda_{\max}^{2}}{\lambda_{*}^{2}}\right)\gamma^{2}(\mathbb{E}[\left\Vert \Sigma_{R}\right\Vert ])^{2}\right),\frac{\gamma^{2}\log(n_{1}+N_{2})}{n_{1}N_{2}}\Bigg\}
\end{align}
and 
\begin{align}
 & \frac{1}{n_{1}N_{2}}\norm{\widehat{\Theta}-\Theta}_{F}^{2}\\
\leq\  & \frac{C}{p^{2}}\max\Bigg\{ n_{1}N_{2}\mathrm{rank}(\Theta)\left(\frac{\lambda_{*}^{2}+\lambda_{*}\lambda_{\max}L_{\gamma}^{2}+L_{\gamma}^{4}\lambda_{\max}^{2}}{L_{\gamma}^{4}}+\left(1+\frac{\lambda_{\max}}{\lambda_{*}}+\frac{\lambda_{\max}^{2}}{\lambda_{*}^{2}}\right)\gamma^{2}(\mathbb{E}[\left\Vert \Sigma_{R}\right\Vert ])^{2}\right),\frac{\gamma^{2}\log(n_{1}+N_{2})}{n_{1}N_{2}}\Bigg\}.
\end{align}
\end{thm}
\begin{thm}
\label{thm:3}Under \ref{assu:sample} and \ref{assu:curvature}, if
we let choose $\lambda$ and $\lambda_{\text{max}}$ in the following
ways: 
\begin{equation}
\lambda_{*}=2c\frac{(U_{\gamma}\lor K)(\sqrt{n_{1}+N_{2}}+(\log(n_{1}\lor N_{2}))^{\frac{3}{2}}}{n_{1}N_{2}}\quad\text{and}\quad\frac{\lambda_{\text{max}}}{\lambda_{*}}\leq\kappa,
\end{equation}
where $\kappa>0$ is some constant, then the following recovery guarantees
hold:
\begin{equation}
\frac{1}{n_{1}N_{2}}\norm{\widehat{\Theta}-\Theta}_{\Pi,F}^{2}\leq\frac{C\mathrm{rank}(\Theta)(n_{1}\lor N_{2})}{pn_{1}N_{2}}\left(1+\frac{\log^{3}(n_{1}\lor N_{2})}{n_{1}\lor N_{2}}\right)\left((U_{\gamma}\lor K)^{2}\frac{1+\kappa L_{\gamma}^{2}+\kappa^{2}L_{\gamma}^{4}}{L_{\gamma}^{4}}+\gamma^{2}(1+\kappa+\kappa^{2})\right)
\end{equation}
and 
\begin{equation}
\frac{1}{n_{1}N_{2}}\left\Vert \widehat{\Theta}-\Theta\right\Vert _{F}^{2}\leq\frac{C\mathrm{rank}(\Theta)(n_{1}\lor N_{2})}{p^{2}n_{1}N_{2}}\left(1+\frac{\log^{3}(n_{1}\lor N_{2})}{n_{1}\lor N_{2}}\right)\left((U_{\gamma}\lor K)^{2}\frac{1+\kappa L_{\gamma}^{2}+\kappa^{2}L_{\gamma}^{4}}{L_{\gamma}^{4}}+\gamma^{2}(1+\kappa+\kappa^{2})\right).
\end{equation}
\end{thm}

\section{Algorithm Framework \label{sec:algorithm}}

In general, there are two dominant approaches on how to solve \ref{eq:p},
namely, proximal gradient method and ADMM. We use the latter one mainly
because the gradient of the max norm is quite hard to calculate. We propose as our solution to
\ref{algo:1}, which is based on the previous work by \cite{fangMaxnormOptimizationRobust2018}.
We present all of the details of this algorithm in the next section.
We note that compared to traditional gradient based method, ADMM has
the advantage of easy parallelization, which is powerful in solving
large scale inputs.

\begin{algorithm}[t]
\caption{ADMM mixed data matrix completion}

\label{algo:1}

\textbf{Input:} $X^{0}$, $Z^{0},W^{0},Y_{\Omega},\lambda,\mu,\alpha,\rho,\tau,t=0$

\textbf{while }Stopping criterion is not satisfied \textbf{do}

\quad $X^{t+1}\text{\ensuremath{\leftarrow\text{proj}_{\mathcal{S}_{+}^{d}}\left\{ Z^{t}-\rho^{-1}(W^{t}+\mu I)\right\} }}$

\quad $Z^{t+1}\leftarrow\mathcal{Z}(X^{t+1}+\rho^{-1}W^{t})$ by
\ref{prop:admm1}.

\quad $W^{t+1}\leftarrow W^{t}+\gamma\rho(X^{t+1}-Z^{t+1}).$

\quad $t\leftarrow t+1$

\textbf{end while}

\textbf{Output: $\widehat{Z}=Z^{t}$, $\widehat{\Theta}=\widehat{Z}^{12}$.}
\end{algorithm}

%\subsection{Derivation Details}

Recall that our estimator is defined as 

\begin{equation}
\widehat{\Theta}:=\underset{\Theta\in\mathbb{R}^{n_{1}\times N_{2}}}{\text{argmin}}\sum_{\tau\in\mathcal{T}}\sum_{i,j\in\Omega_{\tau}}(Y_{ij}^{\tau}\Theta_{ij}^{\tau}-G^{\tau}(\Theta_{ij}^{\tau}))+\lambda_{\max}\left\Vert \Theta\right\Vert _{\max}+\lambda_{*}\left\Vert \Theta\right\Vert _{*},\text{ subject to }\left\Vert \Theta\right\Vert _{\infty}\leq K.\label{eq:max-norm formulation}
\end{equation}

Using definitions of the max-norm and nuclear norm in terms of semi-definite
programs, one can get the following equivalent representation.
\begin{lem}
$\widehat{\Theta}$ is has the following equivalent representation.

\begin{equation}
\widehat{M}=\underset{Z\in\mathbb{R}^{d\times d}}{\argmin}\sum_{\tau\in\mathcal{T}}\sum_{i,j\in\Omega^{\tau}}(Y_{ij}^{\tau}Z_{ij,\tau}^{12}-G^{\tau}(Z_{ij,\tau}^{12}))+\lambda\left\Vert \mathrm{diag}(Z)\right\Vert _{\infty}+\mu\left\langle I,Z\right\rangle \ \mathrm{subject\ to}\left\Vert Z^{12}\right\Vert _{\infty}\leq\alpha,Z\succeq0,\label{eq:max-norm original formulation - split}
\end{equation}
where $d=n_{1}+N_{2}.$
\end{lem}

\paragraph{ADMM formulation.}

Now we formulate the objective function described in \ref{eq:max-norm original formulation - split}
in a way such that ADMM, a popular algorithm with strong convergence
guarantees, could be applied. Note that we can transform the objective
function in the following way: 

\begin{align*}
 & \min_{Z\in\mathbb{R}^{d\times d}}\sum_{\tau\in\mathcal{T}}\sum_{i,j\in\Omega^{\tau}}(Y_{ij}^{\tau}Z_{ij,\tau}^{12}-G^{\tau}(Z_{ij,\tau}^{12}))+\lambda\left\Vert \text{diag}(Z))\right\Vert _{\infty}+\mu\left\langle I,Z\right\rangle \\
\iff & \min_{X,Z\in\mathbb{R}^{d\times d}}\underbrace{\sum_{\tau\in\mathcal{T}}\sum_{i,j\in\Omega^{\tau}}(Y_{ij}^{\tau}Z_{ij,\tau}^{12}-G^{\tau}(Z_{ij,\tau}^{12}))+\lambda\left\Vert \text{diag}(Z)\right\Vert _{\infty}}_{:=\mathcal{L}(Z)}+\mu\left\langle I,X\right\rangle  & \text{ s.t }X,Z\succeq0,\left\Vert Z^{12}\right\Vert _{\infty}\leq\alpha,X-Z=0.\\
\iff & \min_{X,Z\in\mathbb{R}^{d\times d}}\mathcal{L}(Z)+\mu\left\langle I,X\right\rangle.  & \text{s.t }X,Z\succeq0,\left\Vert Z^{12}\right\Vert _{\infty}\leq\alpha,X-Z=0.
\end{align*}

Now we can write the augmented Lagrangian function as 
\[
L(X,Z;W)=\mathcal{L}(Z)+\mu\left\langle I,X\right\rangle +\left\langle W,X-Z\right\rangle +\frac{\rho}{2}\left\Vert X-Z\right\Vert _{F}^{2},\ X\in\mathbb{S}_{+}^{d},Z\in\mathcal{P}:=\{Z\in\mathbb{S}^{d}:\left\Vert Z^{12}\right\Vert _{\infty}\leq\alpha\}.
\]
Hence, the $t+1$th update step of the algorithm is 
\begin{align}
X^{t+1} & =\argmin_{X\in\mathbb{S}_{+}^{d}}L(X,Z^{t};W^{t})=\text{proj}_{\mathbb{S}_{+}^{d}}\left\{ Z^{t}-\rho^{-1}(W^{t}+\mu I)\right\} ,\label{eq:max norm original ADMM X step}\\
Z^{t+1} & =\argmin_{Z\in\mathcal{P}}L(X^{t+1},Z;W^{t})=\argmin_{Z\in\mathcal{P}}\mathcal{L}(Z)+\frac{\rho}{2}\left\Vert Z-X^{t-1}-\frac{1}{\rho}W^{t}\right\Vert _{F}^{2},\label{eq: max norm original ADMM Z step}\\
W^{t+1} & =W^{t}+\tau\rho(X^{t+1}-Z^{t+1}),\label{eq: max norm original ADMM dual step}
\end{align}
where $\tau\in(0,(1+\sqrt{5})/2)$ is a step length operator. Empirical
evidence suggests $\tau$ = 1.618 (the golden ratio) works best. 
\begin{rem}
The rate of convergence of ADMM algorithm in the worse case has been
established to be $O(t^{-1})$, see \cite{fangGeneralizedAlternatingDirection2015}.
\end{rem}

\paragraph{Details for \ref{eq:max norm original ADMM X step}}

Note that 
\[
L(X,Z^{t};W^{t})=\sum_{t=1}^{n}\left(Y_{i_{t},j_{t}}-Z_{i_{t},j_{t}}^{(t)12}\right)^{2}+\lambda\left\Vert \text{diag}(Z)\right\Vert _{\infty}+\mu\left\langle I,X\right\rangle +\left\langle W,X-Z^{t}\right\rangle +\frac{\rho}{2}\left\Vert X-Z\right\Vert _{F}^{2},
\]
where $X\in\mathbb{S}_{+}^{d}$ and $Z\in\mathcal{P}.$ Differentiating
with respect to $X$, 
\begin{align*}
\nabla_{X}L(X,Z^{t};W^{t}) & =\nabla_{X}\left[\mu\left\langle I,X\right\rangle +\left\langle W,X\right\rangle +\frac{\rho}{2}\left\Vert X-Z\right\Vert _{F}^{2}\right]\\
 & =\nabla_{X}\left[\mu\left\langle I,X\right\rangle \right]+\nabla_{Y}\left[\left\langle W,X\right\rangle \right]+\nabla_{X}\left[\frac{\rho}{2}\left\Vert X-Z\right\Vert _{F}^{2}\right]\\
 & =\mu I+W^{t}+\rho(X-Z).
\end{align*}
The critical point is then found by setting gradient to zero: 
\[
\nabla_{X}L(X,Z^{t};W^{t})=0\iff\mu I+W^{t}+\rho X-\rho Z=0\iff Z-\rho^{-1}(W^{t}+\mu I).
\]
To ensure feasibility of $X$, we need to project the critical point
onto the positive semi-definite cone (\cite{Boyd:2011:DOS:2185815.2185816}).
Hence, combined we get 
\[
X^{t+1}=\text{proj}_{\mathbb{S}_{+}^{d}}(Z-\rho^{-1}(W-\mu I)).
\]

\paragraph{Details for \ref{eq: max norm original ADMM Z step}}

Note that 
\[
L(X^{t+1},Z;W^{t})=\mathcal{L}(Z)+\mu\left\langle I,X^{t+1}\right\rangle +\left\langle W^{t},X^{t+1}-Z\right\rangle +\frac{\rho}{2}\left\Vert X^{t+1}-Z\right\Vert _{F}^{2}.
\]
First, we show that $\argmin_{Z\in\mathcal{P}}L(X^{t+1},Z;W^{t})=\argmin_{Z\in\mathcal{P}}\mathcal{L}(Z)+\frac{\rho}{2}\left\Vert Z-X^{t+1}-\frac{1}{\rho}W^{t}\right\Vert _{F}^{2}$
. We note that 
\begin{align}
 & \mathcal{L}(Z)+\frac{\rho}{2}\left\Vert Z-X^{t+1}-\frac{1}{\rho}W^{t}\right\Vert _{F}^{2}\\
=\  & \mathcal{L}(Z)+\frac{\rho}{2}\left\langle Z-X^{t+1}-\frac{1}{\rho}W^{t},Z-X^{t+1}-\frac{1}{\rho}W^{t}\right\rangle \\
=\  & \mathcal{L}(Z)+\frac{\rho}{2}\left[\left\langle Z-X^{t+1}-\frac{1}{\rho}W^{t},Z-X^{t+1}\right\rangle -\left\langle Z-X^{t+1}-\frac{1}{\rho}W^{t},\frac{1}{\rho}W^{t}\right\rangle \right]\\
=\  & \mathcal{L}(Z)+\frac{\rho}{2}\left[\left\langle Z-X^{t+1},Z-X^{t+1}\right\rangle -\frac{1}{\rho}\left\langle W^{t},Z\right\rangle +\frac{1}{\rho}\left\langle W^{t},X^{t+1}\right\rangle -\frac{1}{\rho}\left\langle Z,W^{t}\right\rangle +\frac{1}{\rho}\left\langle W^{t},X^{t+1}\right\rangle +\frac{1}{\rho^{2}}\left\langle W^{t},W^{t}\right\rangle \right]\\
=\  & \mathcal{L}(Z)+\frac{\rho}{2}\left\Vert Z-X^{t+1}\right\Vert _{F}^{2}-\left\langle W^{t},Z\right\rangle +\left\langle W^{t},X^{t+1}\right\rangle +\frac{1}{\rho^{2}}\left\langle W^{t},W^{t}\right\rangle \\
=\  & \mathcal{L}(Z)+\left\langle W^{t},X^{t+1}-Z\right\rangle +\frac{\rho}{2}\left\Vert Z-X^{t+1}\right\Vert _{F}^{2}+\frac{1}{\rho^{2}}\left\Vert W^{t}\right\Vert _{F}^{2}.
\end{align}
Since $\frac{1}{\rho^{2}}\left\Vert W^{t}\right\Vert _{F}^{2}$ in
the last equation is not related to $Z$, it follows that 
\begin{align}
\argmin L(X^{t+1},Z;W^{t}) & =\argmin\mathcal{L}(Z)+\mu\left\langle I,X^{t+1}\right\rangle +\left\langle W^{t},X^{t+1}-Z\right\rangle +\frac{\rho}{2}\left\Vert X^{t+1}-Z\right\Vert _{F}^{2}\\
 & =\argmin\mathcal{L}(Z)+\left\langle W^{t},X^{t+1}-Z\right\rangle +\frac{\rho}{2}\left\Vert X^{t+1}-Z\right\Vert _{F}^{2}\\
 & =\argmin\mathcal{L}(Z)+\left\langle W^{t},X^{t+1}-Z\right\rangle +\frac{\rho}{2}\left\Vert X^{t+1}-Z\right\Vert _{F}^{2}+\frac{1}{\rho^{2}}\left\Vert W^{t}\right\Vert _{F}^{2}\\
 & =\argmin\mathcal{L}(Z)+\frac{\rho}{2}\left\Vert Z-X^{t+1}-\frac{1}{\rho}W^{t}\right\Vert _{F}^{2},\label{eq: max norm original ADMM Z step alternative form}
\end{align}
where the second to last equality is justified by our previous calculation. 

Next, we introduce a result that will help us get a closed form of
the $Z$-step update. 
\begin{prop}
\label{prop:admm1}Let $\Omega=\{(i,j)\}_{t=1}^{n}$ be the index
set of observed entries and let 
\begin{equation}
f(Z)=\sum_{\tau\in\mathcal{T}}\sum_{i,j\in\Omega^{\tau}}(Y_{ij}^{\tau}Z_{ij,\tau}^{12}-G^{\tau}(Z_{ij,\tau}^{12}))+\lambda\left\Vert \mathrm{diag}(Z)\right\Vert _{\infty}+\frac{\rho}{2}\left\Vert Z-C\right\Vert _{F}^{2}.\label{eq:max norm separation lemma 1}
\end{equation}
 Then it follows that $\argmin_{Z\in\mathcal{P}}f(Z)=\mathcal{Z}(C),$
where 
\begin{align}
\mathcal{Z}(C) & =\begin{bmatrix}\mathcal{Z}^{11}(C) & \mathcal{Z}^{12}(C)\\
\mathcal{Z}^{21}(C) & \mathcal{Z}^{22}(C)
\end{bmatrix},\\
Z_{kl}^{12}(C) & =\begin{cases}
\text{proj}_{\left[-\alpha,\alpha\right]}\argmin(Y_{ij}^{\tau}Z_{ij,\tau}^{12}-G^{\tau}(Z_{ij,\tau}^{12}))+\rho(Z_{ij,\tau}^{12}-C_{ij,\tau})^{2}, & \text{if }(k,\ell)\in\Omega,\\
\mathrm{proj}_{[-\alpha,\alpha]}C_{k\ell}^{12} & \mathrm{otherwise},
\end{cases}\\
Z_{kl}^{11}(C) & =C_{kl}^{11}\text{ if }k\neq\ell,\\
Z_{kl}^{22}(C) & =C_{kl}^{22}\text{ if }k\neq\ell,\\
\mathrm{diag}(\mathcal{Z}(C)) & =\argmin_{z\in\mathbb{R}^{d}}\lambda\left\Vert z\right\Vert _{\infty}+\frac{\rho}{2}\left\Vert \mathrm{diag}(C)-z\right\Vert _{2}^{2}.
\end{align}
\end{prop}
\begin{proof}
The idea for the proof is to decompose \ref{eq:max norm separation lemma 1}
into separate disjoint parts based on the blocks of $Z$. Recall that
$Z=\begin{bmatrix}Z_{11} & Z_{12}\\
Z_{12} & Z_{22}
\end{bmatrix}.$ First, we set the notation $X\vert_{\mathrm{ND}}$ to be the $X$
but with its diagonal terms forced to be zero, that is, $X\vert_{\mathrm{ND}}=X-\mathrm{diag}(X)\cdot I$.
Then we note that we can write 
\begin{align}
\left\Vert Z-C\right\Vert _{F}^{2} & =\sum_{i}\sum_{j}\left|z_{ij}-c_{ij}\right|^{2}=\sum_{(i,j)\in Z_{11},i\neq j}\left|z_{ij}-c_{ij}\right|^{2}+\sum_{(i,j)\in Z_{22},i\neq j}\left|z_{ij}-c_{ij}\right|^{2}\\
 & \ \ \ +2\sum_{(i,j)\in Z_{12}}\left|z_{ij}-c_{ij}\right|^{2}+\sum_{(i,j),i=j}\left|z_{ij}-c_{ij}\right|^{2}\\
 & =\left\Vert Z_{11}\vert_{\mathrm{ND}}-C_{11}\vert_{\mathrm{ND}}\right\Vert _{F}^{2}+\left\Vert Z_{22}\vert_{\mathrm{ND}}-C_{22}\vert_{\mathrm{ND}}\right\Vert _{F}^{2}+\left\Vert \text{diag}\left(Z-C\right)\right\Vert _{2}^{2}+2\left\Vert Z_{12}-C_{12}\right\Vert _{F}^{2}.
\end{align}
Hence, it follows that 
\begin{align}
\argmin_{Z\in\mathbb{S}_{d}}f(Z) & =\argmin_{\substack{Z_{11}\vert_{\mathrm{ND}}\in\mathbb{S}_{d_{1}}\\
Z_{22}\vert_{\mathrm{ND}}\in\mathbb{S}_{d_{2}}\\
Z_{12}\in\mathbb{R}^{d_{1}\times d_{2}}\\
\left\Vert Z_{12}\right\Vert _{\infty}^{2}\leq\alpha\\
\mathrm{diag}(Z_{11})\\
\mathrm{diag}(Z_{22})
}
}f_{11}\left(Z_{11}\vert_{\mathrm{ND}}\right)+f_{22}\left(Z_{22}\vert_{\mathrm{ND}}\right)+f_{12}\left(Z_{12}\right)+f_{\mathrm{diag}}(\mathrm{diag}(Z_{11}),\mathrm{diag}(Z_{22})),\nonumber \\
 & =\argmin_{Z_{11}\vert\mathrm{ND}\in\mathbb{S}_{d_{1}}}f_{11}(Z_{11}\vert_{\mathrm{ND}})+\argmin_{Z_{22}\vert\mathrm{ND}\in\mathbb{S}_{d_{2}}}f_{22}(Z_{22}\vert_{\mathrm{ND}})+\argmin_{\substack{Z_{12}\in\mathbb{R}^{d_{1}\times d_{2}}\\
\left\Vert Z_{12}\right\Vert _{\infty}^{2}\leq\alpha
}
}f_{12}(Z_{12})+\argmin_{\substack{\mathrm{diag}(Z_{11})\\
\mathrm{diag}(Z_{22})
}
}(\mathrm{diag}(Z_{11}),\mathrm{diag}(Z_{22})).
\end{align}
where 
\begin{align}
f_{11}(Z_{11}\vert_{\mathrm{ND}}) & =\frac{\rho}{2}\left\Vert Z_{11}\vert_{\mathrm{ND}}-C_{11}\vert_{\mathrm{ND}}\right\Vert _{F}^{2}=\frac{\rho}{2}\sum_{(i,j)\in Z_{11},i\neq j}\left|z_{ij}-c_{ij}\right|^{2},\\
f_{22}(Z_{22}\vert_{\mathrm{ND}}) & =\frac{\rho}{2}\left\Vert Z_{22}\vert_{\mathrm{ND}}-C_{22}\vert_{\mathrm{ND}}\right\Vert _{F}^{2}=\frac{\rho}{2}\sum_{(i,j)\in Z_{22},i\neq j}\left|z_{ij}-c_{ij}\right|^{2},\\
f_{12}(Z_{12}) & =\sum_{\tau\in\mathcal{T}}\sum_{i,j\in\Omega^{\tau}}(Y_{ij}^{\tau}Z_{ij,\tau}^{12}-G^{\tau}(Z_{ij,\tau}^{12}))+\rho\left\Vert Z_{12}-C_{12}\right\Vert _{F}^{2},\\
f_{\text{diag}}(Z_{11},Z_{22}) & =\lambda\left\Vert \mathrm{diag}(Z)\right\Vert _{\infty}+\frac{\rho}{2}\left\Vert \mathrm{diag}(Z-C)\right\Vert _{2}^{2}.
\end{align}
\end{proof}

\paragraph{Optimality of $f_{11}$ and $f_{22}$}

Then note that it is obvious that $f_{11}(Z_{11}\vert_{\mathrm{ND}})\geq0$
for any possible candidate of $Z_{11}\vert_{\mathrm{ND}}$ and takes
equality sign when $Z_{11}\vert_{\mathrm{ND}}=C_{11}\vert_{\mathrm{ND}}$.
The same argument can be made for $f_{22}(Z_{22}\vert_{\mathrm{ND}})$.
Then, it follows that 
\[
\argmin f_{11}(Z_{11}\vert_{\mathrm{ND}})=C_{11}\vert_{\mathrm{ND}},\ \text{and }\argmin f_{22}(Z_{22}\vert_{\mathrm{ND}})=C_{22}\vert_{\mathrm{ND}}.
\]

\paragraph{Optimality of $f_{12}$}

First, we rewrite 
\begin{align}
f_{12}(Z_{12}) & =\sum_{\tau\in\mathcal{T}}\sum_{i,j\in\Omega^{\tau}}(Y_{ij}^{\tau}Z_{ij,\tau}^{12}-G^{\tau}(Z_{ij,\tau}^{12}))+\rho\left\Vert Z^{12}-C^{12}\right\Vert _{F}^{2}\\
 & =\sum_{\tau\in\mathcal{T}}\sum_{i,j\in\Omega^{\tau}}(Y_{ij}^{\tau}Z_{ij,\tau}^{12}-G^{\tau}(Z_{ij,\tau}^{12}))+\rho\sum_{\tau\in\mathcal{T}}\sum_{(i,j)\in\Omega^{\tau}}(Z_{ij,\tau}^{12}-C_{ij,\tau}^{12})^{2}+\rho\sum_{\tau\in\mathcal{T}}\sum_{(i,j)\notin\Omega^{\tau}}(Z_{ij,\tau}^{12}-C_{ij,\tau}^{12})^{2}\\
 & =\sum_{\tau\in\mathcal{T}}\sum_{i,j\in\Omega^{\tau}}\left[(Y_{ij}^{\tau}Z_{ij,\tau}^{12}-G^{\tau}(Z_{ij,\tau}^{12}))+\rho(Z_{ij,\tau}^{12}-C_{ij,\tau})^{2}\right]+\rho\sum_{\tau\in\mathcal{T}}\sum_{(i,j)\notin\Omega^{\tau}}(Z_{ij,\tau}^{12}-C_{ij,\tau}^{12})^{2}.
\end{align}
Note that 
\begin{align*}
\frac{\partial f_{12}}{\partial Z_{(i,j)\in\Omega^{\tau}}} & =2\rho(Z_{ij,\tau}^{12}-C_{ij}^{12})=0\implies Z_{(i,j)\notin\Omega^{\tau}}=C_{ij,\tau}^{12}.
\end{align*}
Since $Z_{12}$ has constraint $Z_{12}\in\mathcal{B}_{\left\Vert \cdot\right\Vert _{\infty}}(\alpha)$,
we need to project it to the constrained space: 
\[
Z_{ij,\tau}^{12}=\begin{cases}
\text{proj}_{\left[-\alpha,\alpha\right]}\argmin(Y_{ij}^{\tau}Z_{ij,\tau}^{12}-G^{\tau}(Z_{ij,\tau}^{12}))+\rho(Z_{ij,\tau}^{12}-C_{ij,\tau})^{2} & \text{if }(i,j)\in\Omega\\
\mathrm{proj}_{\left[-\alpha,\alpha\right]}C_{ij}^{12} & \text{otherwise}
\end{cases}.
\]

\paragraph{Optimality of $f_{\mathrm{diag}}$}

Note that $\argmin f_{\mathrm{diag}}$ can be cased into the following
program: 
\begin{equation}
\min_{z\in\mathbb{R}^{d}}\beta\left\Vert z\right\Vert _{\infty}+\frac{1}{2}\left\Vert c-z\right\Vert _{2}^{2},
\end{equation}
where $c=(c_{1},\dots,c_{d})^{T}=\mathrm{diag}(C)$ and $\beta=\frac{\lambda}{\rho}.$
A closed form solution could be formed by laying out the KKT condition,
see \ref{lem:kkt helper}.

\paragraph{Duality on $X$}

Assume that $X^{t+1}$ reaches that optimality, then we have 
\begin{align}
0 & \in\partial\delta_{\mathbb{S}_{+}^{d}}(X^{t+1})+\mu I+W^{t},\label{eq:kkt-stationary-x}
\end{align}
where we can rewrite the RHS as 
\begin{align}
(\mathrm{RHS}) & =\partial\delta_{\mathbb{S}_{+}^{d}}(X^{t+1})+\mu I+W^{t}+\rho(X^{t+1}-Z)\\
 & =\partial\delta_{\mathbb{S}_{+}^{d}}(X^{t+1})+\rho(\rho^{-1}(\mu I+W^{t})+(X^{t+1}-Z)).
\end{align}
Therefore, we can rewrite \ref{eq:kkt-stationary-x} as 
\begin{align}
 & \rho(Z^{t}-X^{t+1})-W^{t}\in\partial\delta_{\mathbb{S}_{+}^{d}}(X^{t+1})+\mu I\\
\iff & \rho(Z^{t}-X^{t+1})-W^{t}\in\partial\delta_{\mathbb{S}_{+}^{d}}(X^{t+1})+\mu I\\
\iff & \rho(Z^{t}-X^{t+1})-W^{t}+W^{t+1}\in\partial\delta_{\mathbb{S}_{+}^{d}}(X^{t+1})+\mu I+W^{t+1}\\
\iff & \rho(Z^{t}-Z^{t+1})-W^{t}+W^{t+1}+\rho Z^{t+1}-\rho X^{t+1}\in\partial\delta_{\mathbb{S}_{+}^{d}}(X^{t+1})+\mu I+W^{t+1}\\
\iff & \rho(Z^{t}-Z^{t+1})+W^{t+1}-(W^{t}+\rho(X^{t+1}-Z^{t+1}))\in\partial\delta_{\mathbb{S}_{+}^{d}}(X^{t+1})+\mu I+W^{t+1},\label{eq:kkt-stationary-x-form1}
\end{align}
Let $\widetilde{W}^{t+1}=W^{t}+\rho(X^{t+1}-Z^{t+1}),$ then \ref{eq:kkt-stationary-x-form1}
could be written as 
\begin{equation}
\rho(Z^{t}-Z^{t+1})+W^{t+1}-\widetilde{W}^{t+1}\in\partial\delta_{\mathbb{S}_{+}^{d}}(X^{t+1})+\mu I+W^{t+1}.
\end{equation}

\paragraph{Duality on $Z$}

Note that we originally have the optimality condition as 
\begin{equation}
0\in\partial\delta_{\mathcal{P}}(Z)+\nabla\mathcal{L}(Z)-W.\label{eq:kkt-stationary-z}
\end{equation}
At iteration $t+1,$ if $Z^{t+1}$ satisfies reaches the optimality
condition, we would have $Z^{t+1}-X^{t+1}$ since we have updated
in the first step. Then we have \ref{eq:kkt-stationary-z} is equivalent
to the following 
\begin{align}
 & 0\in\partial\delta_{\mathcal{P}}(Z^{t+1})+\nabla\mathcal{L}(Z^{t+1})-W^{t}+\rho(Z^{t+1}-X^{t+1})\\
\iff & W^{t}+\rho(X^{t+1}-Z^{t+1})\in\partial\delta_{\mathcal{P}}(Z^{t+1})+\nabla\mathcal{L}(Z^{t+1})\\
\iff & \widetilde{W}^{t+1}-W^{t+1}\in\partial\delta_{\mathcal{P}}(Z^{t+1})+\nabla\mathcal{L}(Z^{t+1})-W^{t+1}.
\end{align}

\begin{rem}
\label{rem:early-stop-dual-gap}The purpose of the rewriting above
is to create a get condition for early stopping. Namely, once we have
updated all of $X,Z,W$ in the $(t+1)$th iteration and hypothetically
we have reached the optimality condition 
\[
\begin{cases}
0\in\partial\delta_{\mathbb{S}_{+}^{d}}(X^{t+1})+\mu I+W^{t+1}\\
0\in\partial\delta_{\mathcal{P}}(Z^{t+1})+\nabla\mathcal{L}(Z^{t+1})-W^{t+1}
\end{cases},
\]
then by the equivalent formulation above, the pair $\begin{cases}
\widetilde{W}^{t+1}-W^{t+1}\\
\rho(Z^{t}-Z^{t+1})+W^{t+1}-\widetilde{W}^{t+1}
\end{cases}$ should be close to $0$, i.e. the value $R_{D}$ defined as 
\begin{equation}
\max\left\{ \left\lVert \tilde{W}^{t+1}-W^{t+1}\right\rVert ,\left\lVert \rho(Z^{t}-Z^{t+1})+W^{t+1}-\widetilde{W}^{t+1}\right\rVert \right\} 
\end{equation}
should be small. We also note any choice should norm should work for
$R_{D}$ due to the equivalence of norms in finite dimensional vector
spaces; however, difference norm might induce a difference convergence
rate and as a result impact the effectiveness of the early stopping
predicate. Empirically, Frobenous norm works quite well in most cases. 
\end{rem}
\begin{rem}
\label{rem:early-stop-primal-gap}If $X^{t+1},Z^{t+1},W^{t+1}$ produces
the optimal solution, aside from satisfying the condition in the previous
remark, $X^{t+1}$ and $Z^{t+1}$ should also satisfy the primal feasibility
condition, i.e. $X^{t+1}=Z^{t+1}.$ Numerically, this means that the
value $R_{P}:=\left\lVert X^{t+1}-Z^{t+1}\right\rVert $ should be
small. 
\end{rem}

\paragraph{Early stopping}

Based on \ref{rem:early-stop-dual-gap} and \ref{rem:early-stop-primal-gap},
we propose the following early stopping predicate to speed up our
main algorithm. 

\begin{algorithm}[t]
\caption{Early Stopping Predicate}

\label{algo:early-stopping-predicate}

\textbf{function }\textsc{EarlyStopPredicate}$(X,Z,W,\mathrm{tol})$

$\quad$$R_{P}\leftarrow\left\lVert X^{t+1}-Z^{t+1}\right\rVert $$_{F}$

$\quad R_{D}\leftarrow\max\left\{ \left\lVert \tilde{W}^{t+1}-W^{t+1}\right\rVert _{F},\left\lVert \rho(Z^{t}-Z^{t+1})+W^{t+1}-\widetilde{W}^{t+1}\right\rVert _{F}\right\} $

$\quad$\textbf{if} $\max(R_{p},R_{d})\ <\ $$\mathrm{tol}$

$\qquad$\textbf{return} true

$\quad$\textbf{return} false

\textbf{end function}
\end{algorithm}

\paragraph{Adjust $\rho$ dynamically}

According to \cite{fangMaxnormOptimizationRobust2018}, dynamically
adjusting $\rho$ according to helps speed up the convergence of the
ADMM algorithm. We remark that in the mixed data setting this speed-up
procedure still works. 

\begin{algorithm}[t]
\caption{Balance Gap}

\label{algo:admm-adjust-rho}

\textbf{function }\textsc{BalanceGap}$(\rho)$

$\quad$\textbf{if} $\left\lVert R_{P}^{t+1}\right\rVert <0.5\left\lVert R_{D}^{t+1}\right\rVert $

$\qquad$$\text{\ensuremath{\rho\leftarrow0.7\rho}}$

$\quad$\textbf{if} $\left\lVert R_{D}^{t+1}\right\rVert <0.5\left\lVert R_{P}^{t+1}\right\rVert $

$\qquad$$\text{\ensuremath{\rho\leftarrow1.3\rho}}$

\textbf{end function}
\end{algorithm}

%\subsection{Implementation details}

Due the the fact that eigen-decomposition is performed in every iteration
of ADMM, we left a few flags in the implemented package for users
to choose the eigen-decomposition procedure. For a reasonably large
matrix of size $5000\times5000$ full eigen decomposition is costly
and as we will show in simulation result that the non-dominate eigen
values/vector pairs have negligible effects on the final output, a
sparse eigen routine is often enough to get the desired recovery.

\section{Numerical Experiments\label{sec:numerical-experiments}}

In this section, we present several numerical simulation on random
generated low rank matrix data to verify the validity of our proposed
model. In additional to tracking recovery rates, we will also focus
on 

Due to the fact that our computational package is still in development
and stability needs further improvement (some of the large scale simulation
could not be 100\% reproduced), we present a small scale numerical
result for the purpose of verifying the correctness of our proposed
algorithm.

\paragraph{Small Scale Pure Data 1}

In this experiment, we randomly generate $500\times500$ matrix of
one single distribution (Normal, Gamma, Poisson, Bernoulli and Negative
Negative Binomial) and keep its rank fixed while measure the recovery
result under different sample rate. The results are shown in \Cref{fig:single_gaussian_small,fig:single_bernoulli_small,fig:single_poisson_small,fig:single_gamma_small,fig:single_negbin_small,fig:single_gaussian_medium,fig:single_bernoulli_medium,fig:single_poisson_medium,fig:single_gamma_medium,fig:single_negbin_medium}.

\begin{figure}[H]
\minipage{0.2\textwidth}
  \includegraphics[width=\linewidth]{figures/single_gaussian_small}
  \caption{}\label{fig:single_gaussian_small}
\endminipage\hfill
\minipage{0.2\textwidth}
  \includegraphics[width=\linewidth]{figures/single_bernoulli_small}
  \caption{}\label{fig:single_bernoulli_small}
\endminipage\hfill 
\minipage{0.2\textwidth}
  \includegraphics[width=\linewidth]{figures/single_poisson_small}
  \caption{}\label{fig:single_poisson_small}
\endminipage\hfill 
\minipage{0.2\textwidth}
  \includegraphics[width=\linewidth]{figures/single_gamma_small}
  \caption{}\label{fig:single_gamma_small}
\endminipage\hfill 
\minipage{0.2\textwidth}
  \includegraphics[width=\linewidth]{figures/single_negbin_small}
  \caption{}\label{fig:single_negbin_small}
\endminipage\hfill 
\minipage{0.2\textwidth}
  \includegraphics[width=\linewidth]{figures/single_gaussian_medium}
  \caption{}\label{fig:single_gaussian_medium}
\endminipage\hfill
\minipage{0.2\textwidth}
  \includegraphics[width=\linewidth]{figures/single_bernoulli_medium}
  \caption{}\label{fig:single_bernoulli_medium}
\endminipage\hfill
\minipage{0.2\textwidth}
  \includegraphics[width=\linewidth]{figures/single_poisson_medium}
  \caption{}\label{fig:single_poisson_medium}
\endminipage\hfill
\minipage{0.2\textwidth}
  \includegraphics[width=\linewidth]{figures/single_gamma_medium}
  \caption{}\label{fig:single_gamma_medium}
\endminipage\hfill
\minipage{0.2\textwidth}
  \includegraphics[width=\linewidth]{figures/single_negbin_medium}
  \caption{}\label{fig:single_negbin_medium}
\endminipage\hfill
\end{figure}

\paragraph{Small Scale Mixed Data 1}

In this experiment, we randomly generate $500\times500$ matrix of
five mixed distributions (Normal, Gamma, Poisson, Bernoulli and Negative
Binomial) and keep its rank fixed while measure the recovery result
under different sample rate. The results are shown in \Cref{fig:SSMD1,fig:SSMD10,fig:SSMD20,fig:SSMD30,fig:SSMD40,fig:SSMD50,fig:SSMD60,fig:SSMD70,fig:SSMD80},
where each colored line represents the relative error compared to
the truth matrix for its corresponding distributions types. The X-axis
represents the sampling rate. An averaged relative error over all
distributions is shown in figure 

\begin{figure}[H]
\minipage{0.33\textwidth}
  \includegraphics[width=\linewidth]{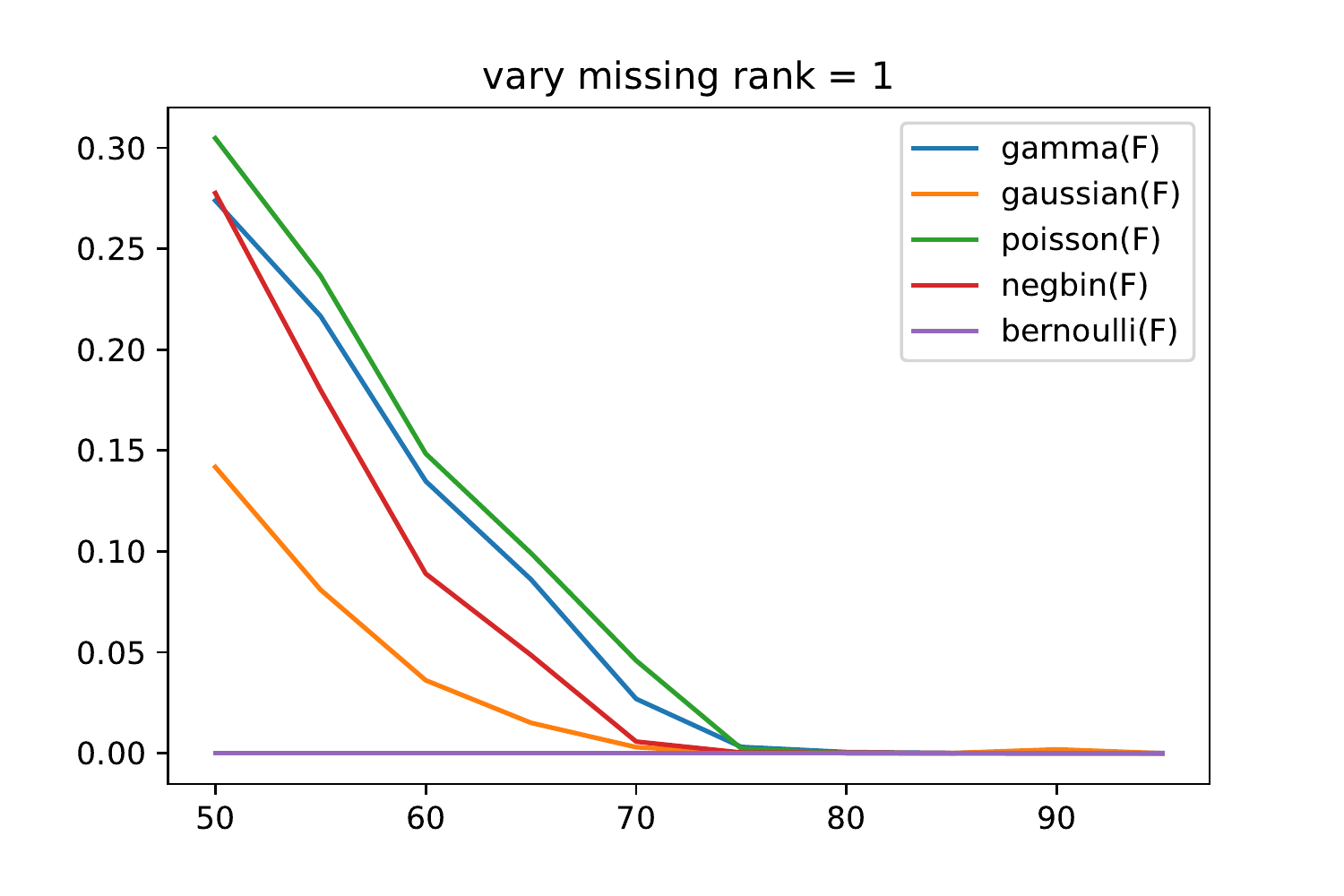}
  \caption{}\label{fig:SSMD1}
\endminipage\hfill
\minipage{0.33\textwidth}
  \includegraphics[width=\linewidth]{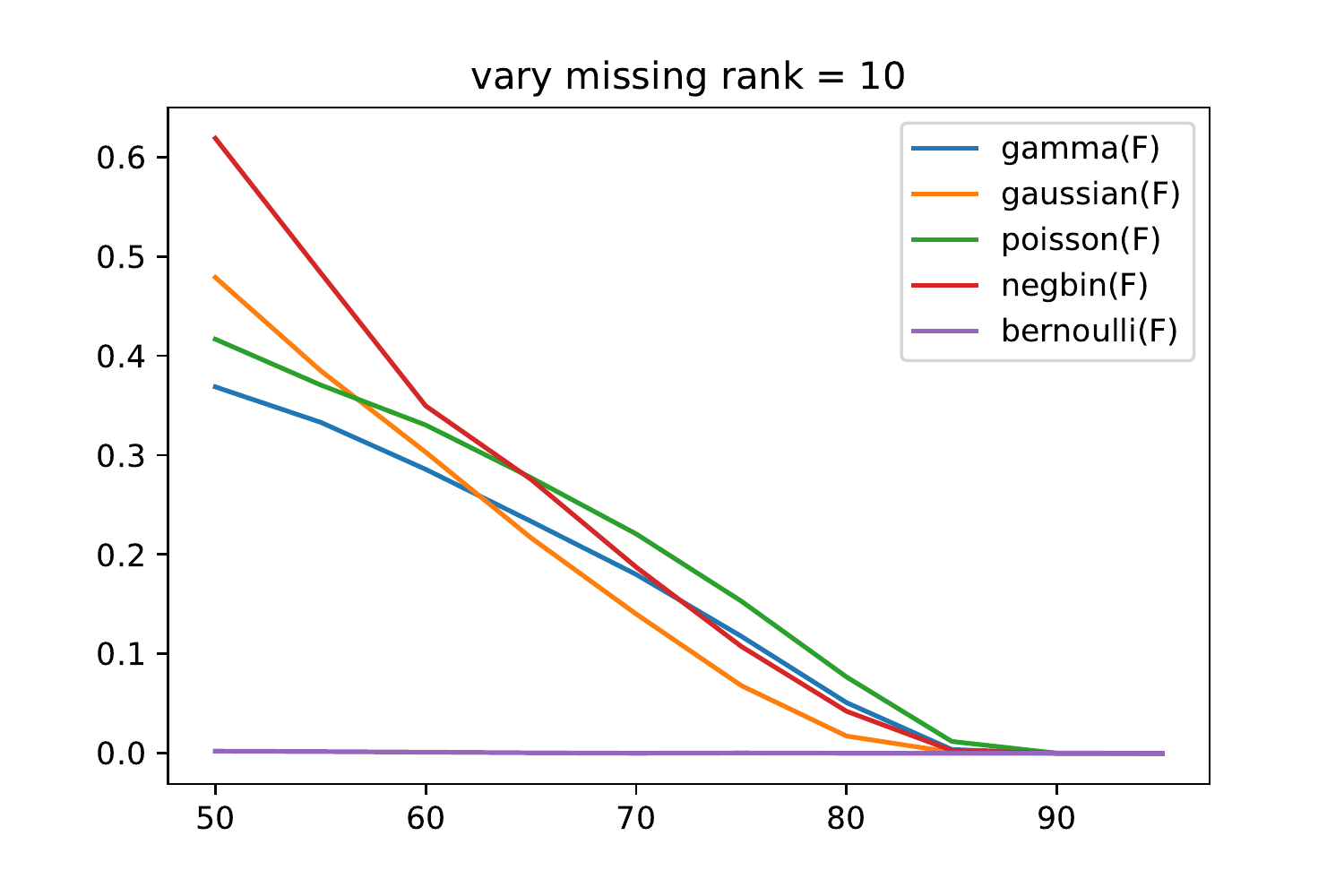}
  \caption{}\label{fig:SSMD10}
\endminipage\hfill
\minipage{0.33\textwidth}
  \includegraphics[width=\linewidth]{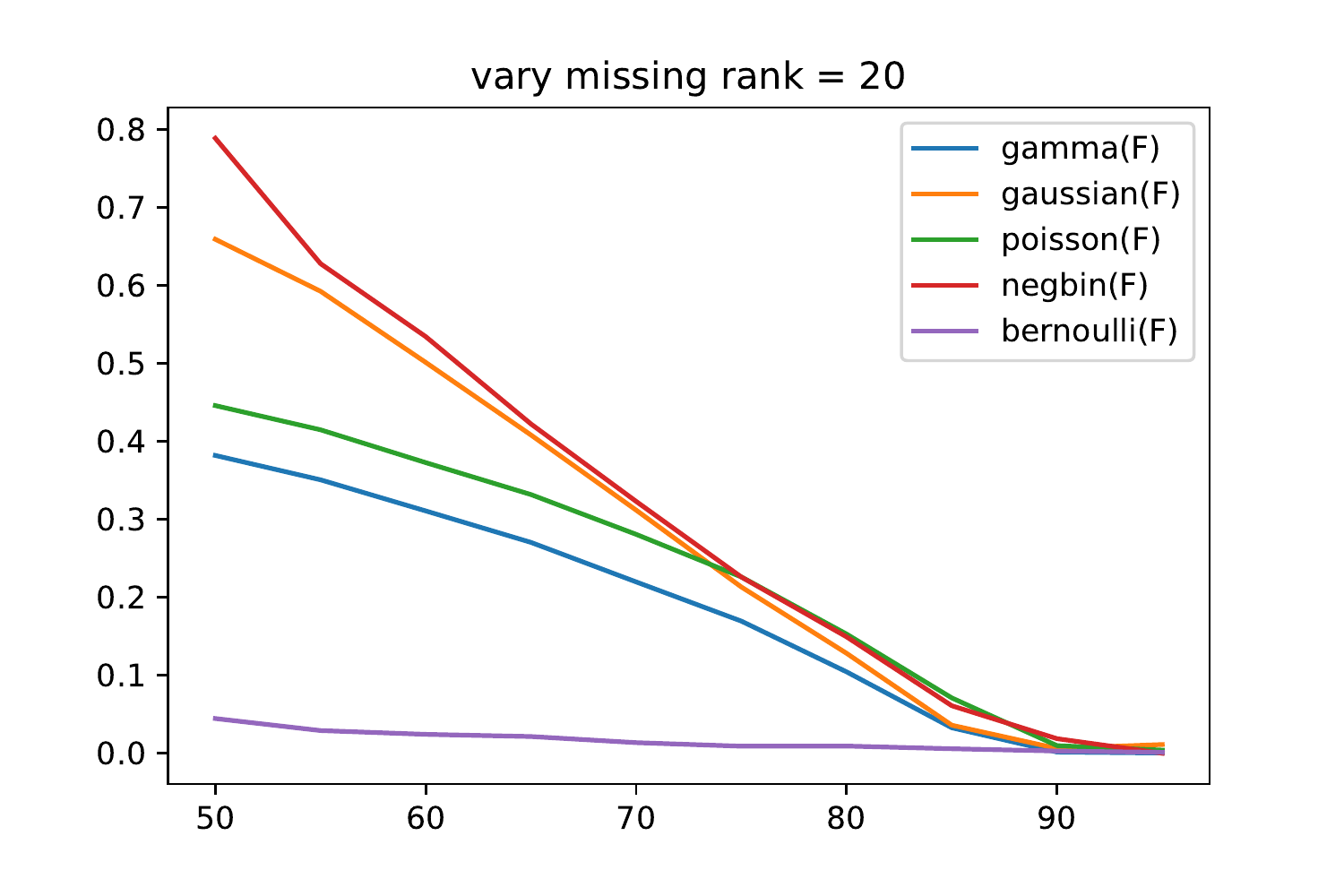}
  \caption{}\label{fig:SSMD20}
\endminipage\hfill 
\newline
\minipage{0.33\textwidth}
  \includegraphics[width=\linewidth]{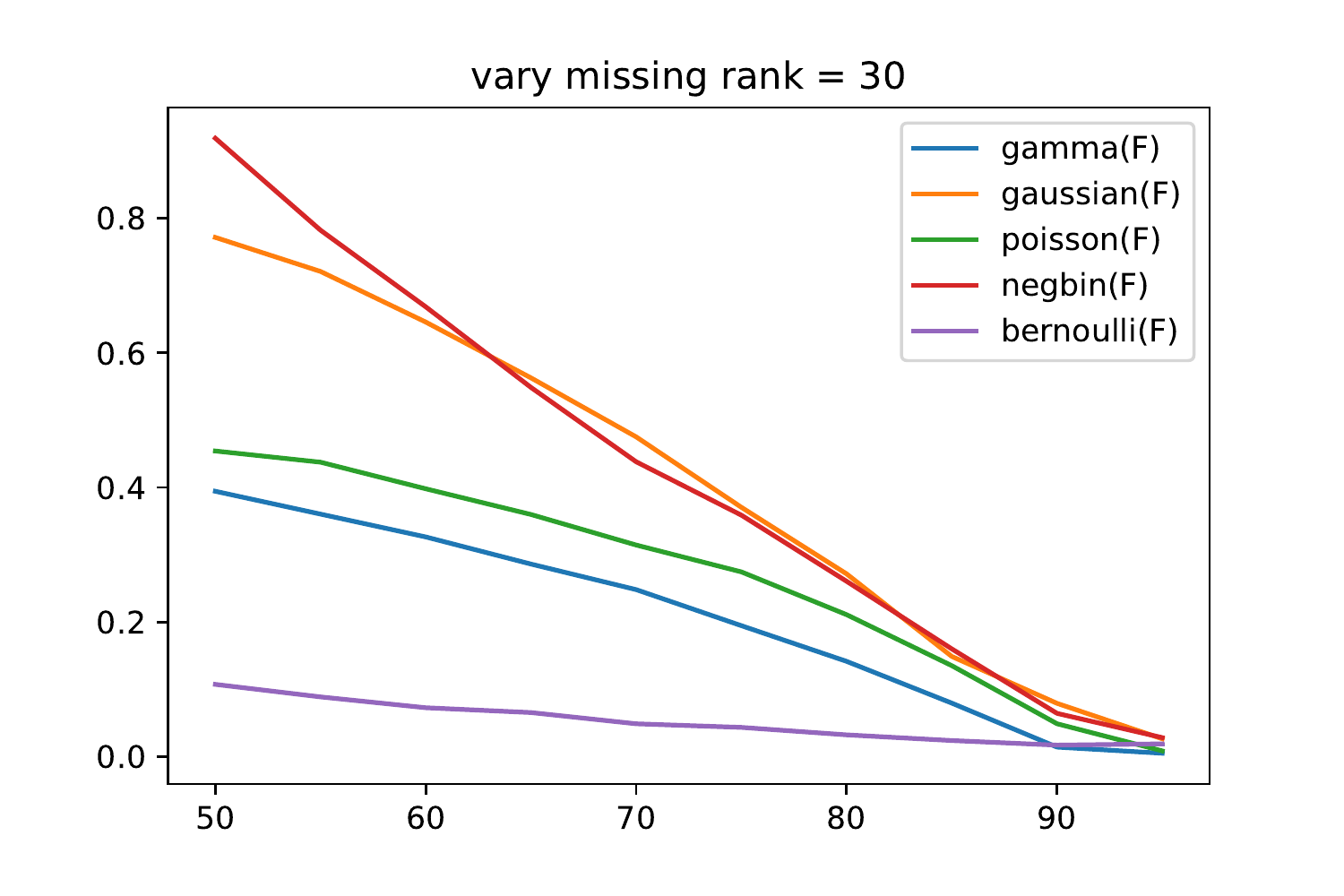}
  \caption{}\label{fig:SSMD30}
\endminipage\hfill
\minipage{0.33\textwidth}
  \includegraphics[width=\linewidth]{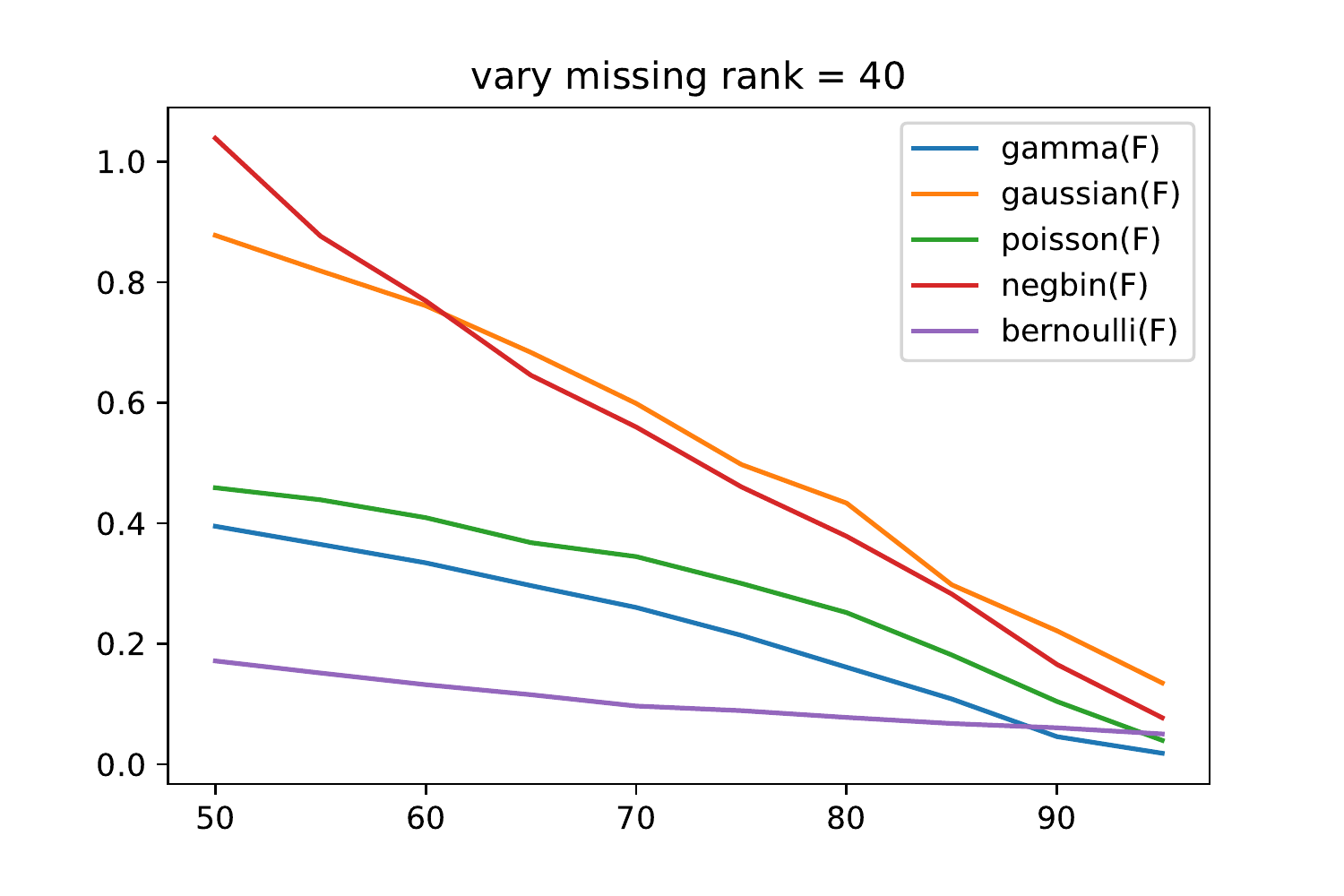}
  \caption{}\label{fig:SSMD40}
\endminipage\hfill
\minipage{0.33\textwidth}
  \includegraphics[width=\linewidth]{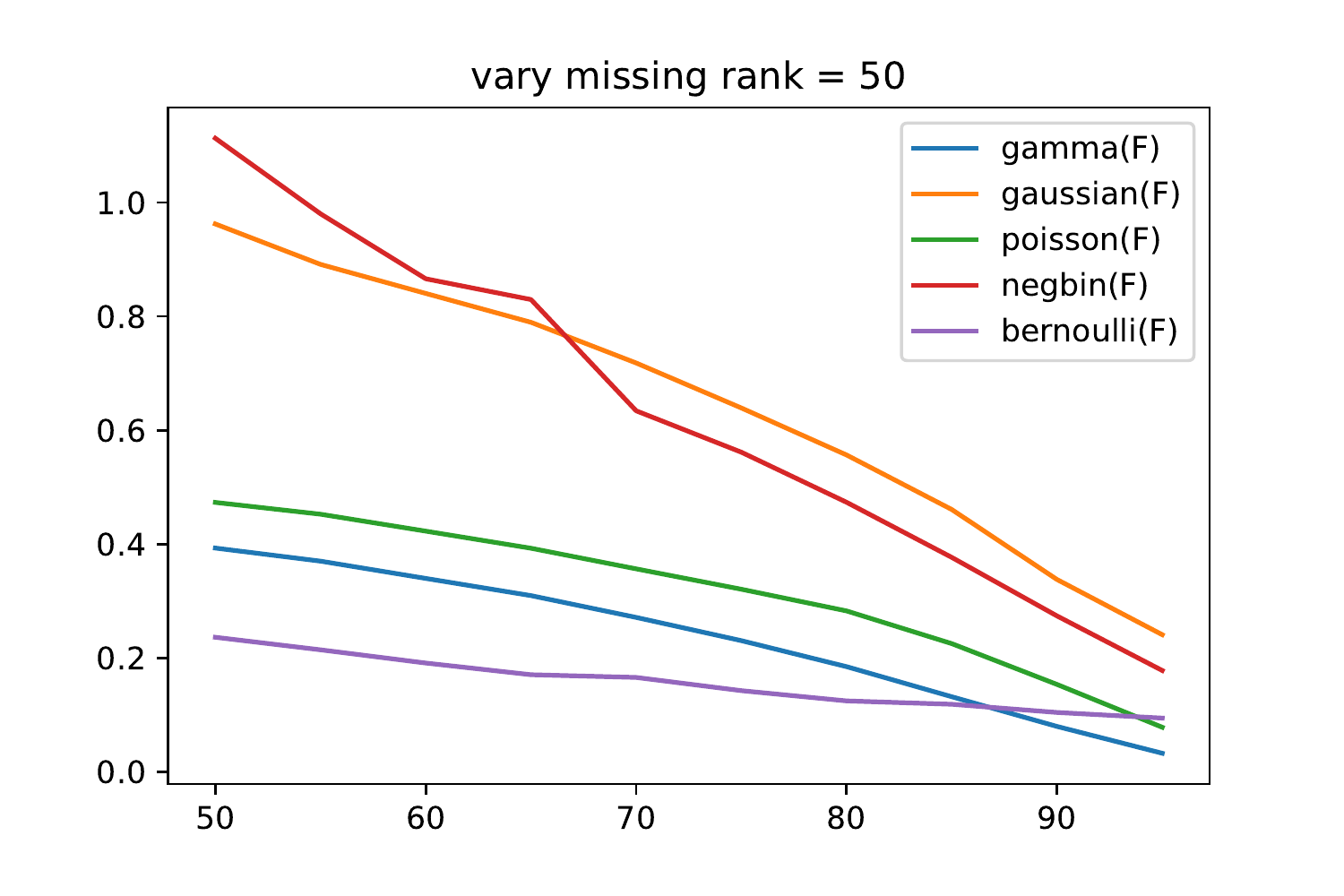}
  \caption{}\label{fig:SSMD50}
\endminipage\hfill 
\newline
\minipage{0.33\textwidth}
  \includegraphics[width=\linewidth]{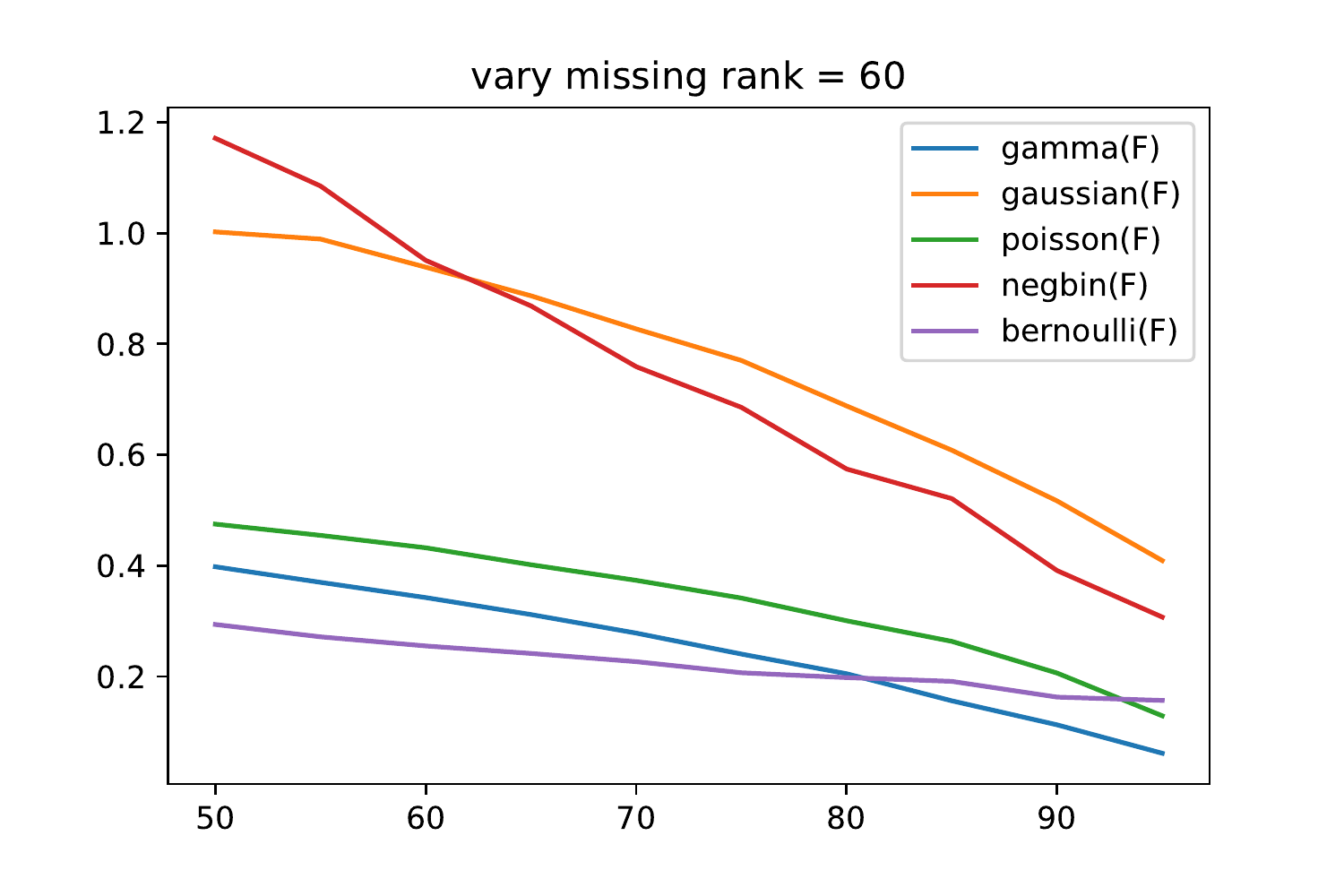}
  \caption{}\label{fig:SSMD60}
\endminipage\hfill
\minipage{0.33\textwidth}
  \includegraphics[width=\linewidth]{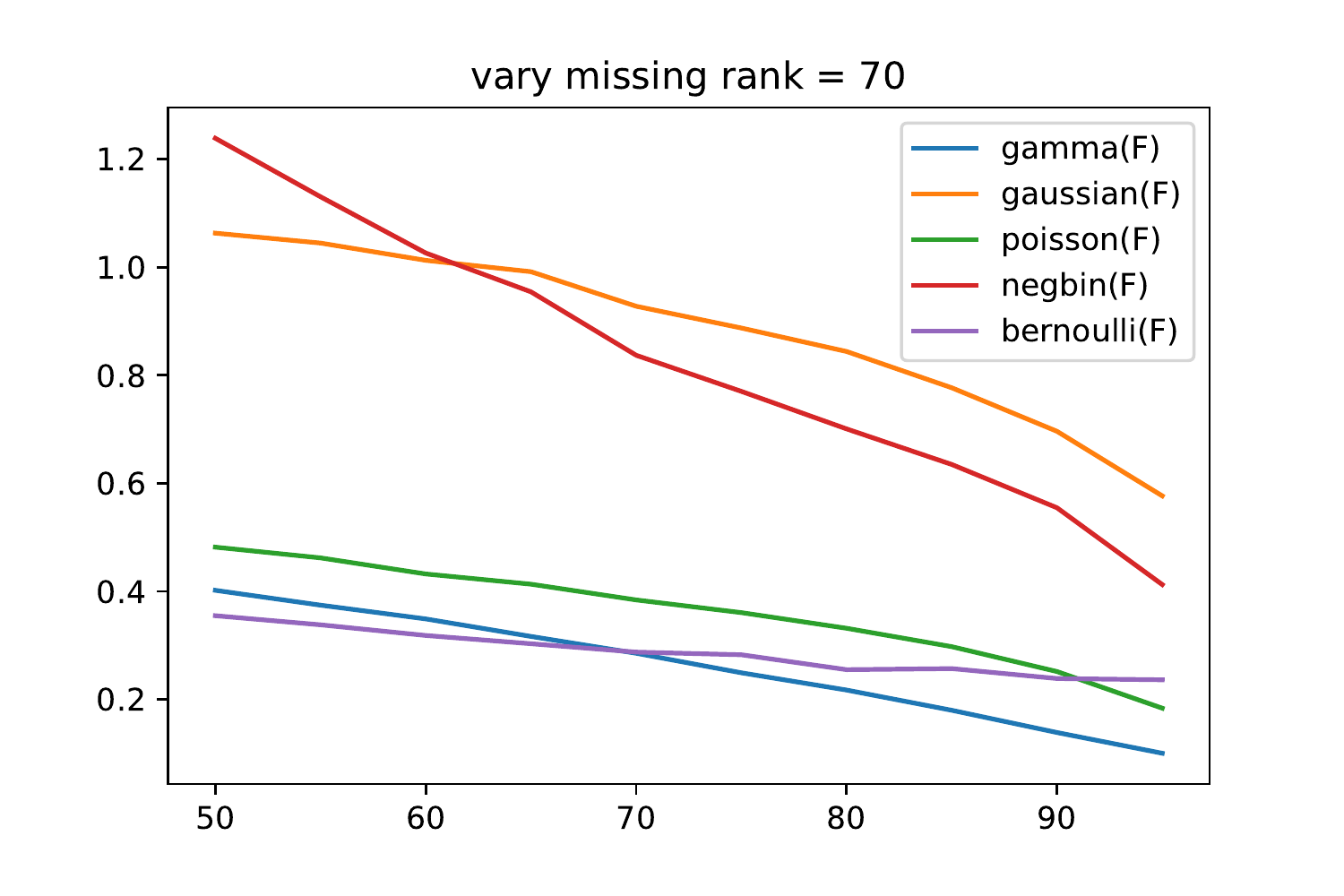}
  \caption{}\label{fig:SSMD70}
\endminipage\hfill
\minipage{0.33\textwidth}
  \includegraphics[width=\linewidth]{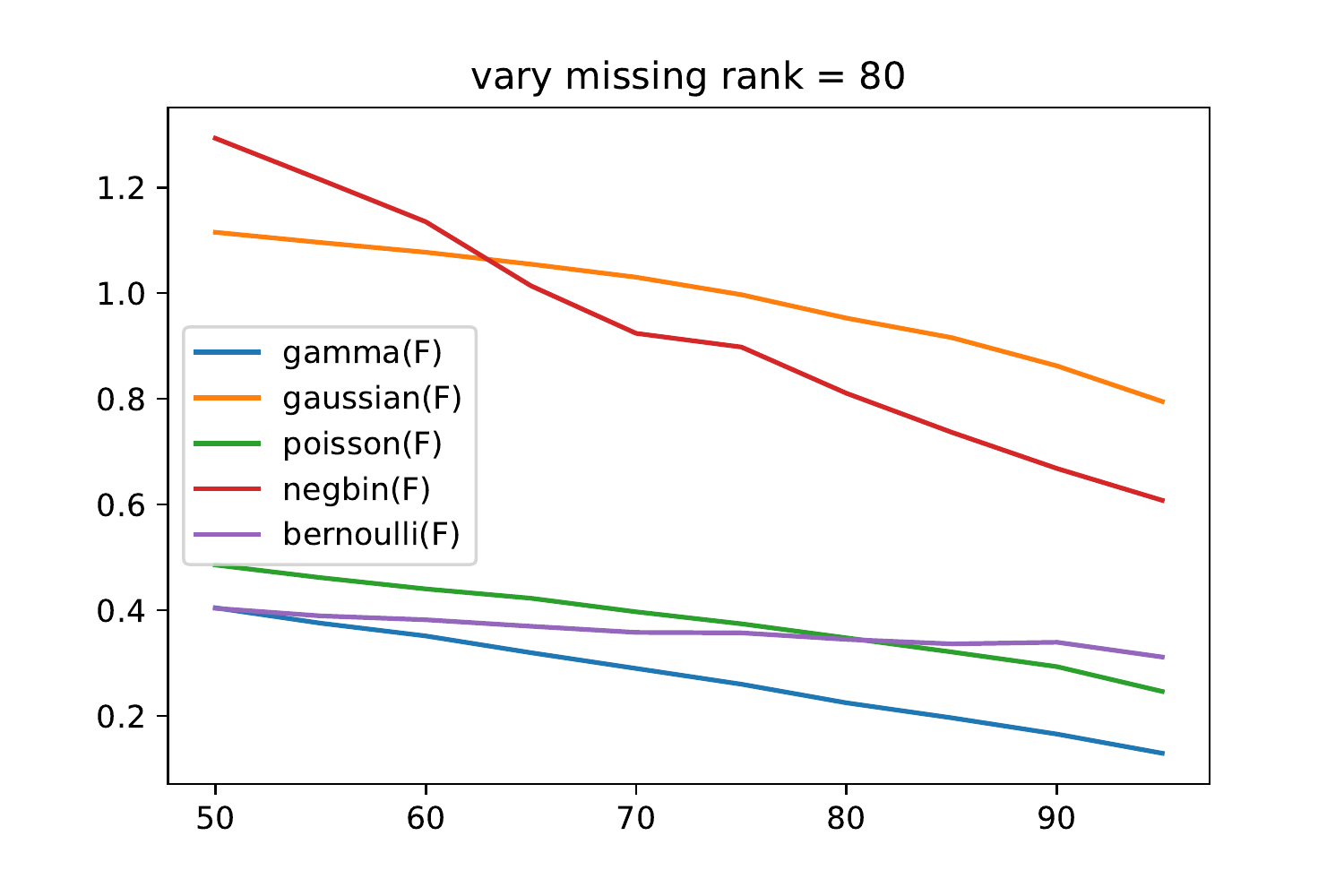}
  \caption{}\label{fig:SSMD80}
\endminipage\hfill 
\end{figure}

\begin{figure}[!htb]
\centering
\includegraphics[height=8cm]{figures/vary_missing_500_avg}
\caption{}\label{fig:SSMD500AVG}
\end{figure}

\paragraph{Small Scale Mixed Data 2}

In this experiment, we test the performance of our algorithm the sampling
rate is fixed at 80\% while changing the input rank of the input matrix.
The resulting figure is in \ref{fig:vary_rank_small}

\begin{figure}[H]
\centering
\includegraphics[height=8cm]{figures/vary_rank_small}
\caption{}\label{fig:vary_rank_small}
\end{figure}

\paragraph{Medium Scale Mixed Data }

In this experiment, we reproduce the same previous evaluation procedures
on medium scaled input. We generate $2000\times2000$ matrix of 5
mixed types(Gaussian, Bernoulli, Poisson, NegBin and Gamma), each
of which could be view as a $2000\times400$ submatrix. We then measure
the performance when holding rank fixed and varying sample rate and
vice versa. The results are in \ref{fig:vary_missing_medium} and
\ref{fig:vary_rank_medium}

\begin{figure}[H]
\minipage{0.4\textwidth}
  \includegraphics[width=\linewidth]{figures/vary_missing_avg_medium}
  \caption{}\label{fig:vary_missing_medium}
\endminipage\hfill
\minipage{0.4\textwidth}
  \includegraphics[width=\linewidth]{figures/vary_rank_medium}
  \caption{}\label{fig:vary_rank_medium}
\endminipage\hfill
\end{figure}

\paragraph{Different Eigen-solvers}

This experiment is designed to test the difference in performance
when different eigen-solvers were used: full eigen-decomposition or
truncated-eigen-decomposition. The result is in \ref{fig:diff_eigen}.
The input matrix is a $500\times500$ mixed typed matrix with each
data type occupying a $500\times100$ sub-matrix. We can see that
when the rank is low, i.e. less than $20\%$ of the corresponding
sub-matrices, the difference between using full and partial eigen
decomposition is small.

\begin{figure}[!htb]
\centering
\includegraphics[height=7cm]{figures/diff_eigen}
\caption{}\label{fig:diff_eigen}
\end{figure}

\paragraph{Observations} The simulation results help verify our theoretical results in that we can see from the plots that 
\begin{itemize}
    \item when the rank is low and fixed, the recovery success is proportional to the sampling rate;
    \item when the sampling rate is fixed, the recovery success is inversely proportional to the rank of the data matrix;
    \item the recovery success when recovering mixed distributed low rank matrices is on par with recovering singly-distributed low rank matrices. 
\end{itemize}
Additionally, we note that although in theory the full eigen-decomposition should be used in order to find out all the positive eigen value/vector pairs, in practice when the matrix is sufficiently low rank, e.g. 10\% of $\min{n, m}$ where $n, m$ respectively refer to row count and column count, using truncated eigen-solver therefore only taking not the full positive spectrum but only the dominate ones actually performs on par with taking the full spectrum. However, we should also note that as rank increases, the truncated eigen-version of the algorithm under performs significantly. 

\section{Concluding Remarks}

From a theoretical point of view we have only obtained an upper bound on the recovery rate. However, many of the previous works have developed a lower bound using information theoretic techniques. It would be interesting to see if a similar 
result could be proved in this general case. Although we have shown that a hybrid of max norm and Schatten norm in the loss function can lead to recovery of the matrix with statistical guarantee, the inequalities between max norm and Schatten norm actually provides a significant bridge in facilitating the final proof. We could not produce a similar result using the same technique without the existence of nuclear norm in the loss function. Hence, an open question is whether we can prove a similar result for max-norm-only loss functions.
\\

\noindent While our paper mainly discusses theoretical results, the numerical implementation counterparts are also worth some brief discussion.  The algorithms developed and analyzed in this article has been implemented in a Julia package, \texttt{MatrixCompletion.jl}\footnote{see https://github.com/jasonsun0310/MatrixCompletion.jl}. To the best of our knowledge, this is the first dedicated package in Julia that address the problem of matrix completion of reasonably large input size that uses convex optimization methods. In addition, \texttt{MatrixCompletion.jl} also provides several features that we deem useful for interested readers who want to get hands on experience with our algorithm.

\paragraph{Automatic Data Type Detection}
In reality it is often unknown that what are the exact distributions of the underlying data. To address this issue, we provided an API that allows the algorithm to automatically detect the best fitting distributed within the supported range and after doing so, also acquire the MLEs of the corresponding parameters. Traditional goodness-and-fit often has less power when the input data size are large. To address this problem, we adopted a different approach combining a simple trivial decision tree and comparing the empirical distribution to its exponential family candidates in terms of moment generating functions. 

\paragraph{Automatic Differentiation and Extensible Loss Function Design}
We acknowledge that besides the loss functions we proposed, there are many other possible candidates within or outside the exponential families could be deemed useful in solving the matrix completion problem. \texttt{MatrixCompletion.jl}'s implementation has taken these factors into consideration. Custom loss functions are possible. Furthermore, we also have bundled automatic differentiation support to help facilitate the implementation of custom loss function by removing the need to manually implement another gradient.

\paragraph{More Classical Algorithms} 
With the help of Github and researchers around the world, we are aiming to make \texttt{MatrixCompletion.jl} a comprehensive library on matrix completion. Currently, we are adding more classical algorithms such as singular value thresholding, manifold optimization based methods. Because of Julia's multiple dispatch system and its good module system, all these algorithms can be implemented under one polymorphic method call, which is straight forward as well as user-friendly.

\printbibliography

\newpage

\section*{Appendix A: Theoretical Results \label{sec:appendix-theoreical-results-proof}}

\subsection*{Precise statement of upper bounds}

Let the collection of matrices $(E_{11}^{\tau},...,E_{n_{1}n_{2}^{\tau}})$
be the canonical basis in the space of matrices of size $n_{1}\times n_{2}^{\tau},$
Let $(\varepsilon_{ij}^{\tau})$ be an i.i.d. Rademacher sequence.
We defined 
\begin{equation}
\Sigma_{R}=(\Sigma_{R}^{1},...,\Sigma_{R}^{\left|\mathcal{T}\right|}),
\end{equation}
where 
\begin{equation}
\Sigma_{R}^{\tau}=\frac{1}{n_{1}N_{2}}\sum_{i,j\in[n_{1}]\times[n_{2}^{\tau}]}\varepsilon_{ij}^{\tau}\delta_{ij}^{\tau}E_{ij}^{\tau}.
\end{equation}
The following lemma provides a bound on the operator norm of $\Sigma_{R}$.
\begin{lem}[Lemma 1 in \cite{alayaCollectiveMatrixCompletion2019}]
There exists an absolute constant $c$ such that 
\begin{equation}
\mathbb{E}\left[\left\Vert \Sigma_{R}\right\Vert \right]\leq c\left(\frac{\sqrt{\mu}+\sqrt{\log(n_{1}\land N_{2})}}{n_{1}N_{2}}\right).
\end{equation}
\end{lem}
Additionally, we let $\left\Vert \right\Vert _{\Pi,F}$ be the weighted
Frobenous norm defined by $\left\Vert A\right\Vert _{\Pi,F}=\sum_{\tau\in\mathcal{T}}\sum_{n_{1}\times n_{2}^{\tau}}\pi_{ij}^{\tau}(A_{ij}^{\tau})^{2}.$

\subsection*{Proof of \ref{thm:2}}

Since by assumption $\Theta\in\mathcal{B}_{_{\infty}}^{n_{1}\times N_{2}}(\gamma)$,
it follows that $\mathcal{L}(\widehat{\Theta}|Y)\leq\mathcal{L}(\Theta\vert Y),$
which expand to 
\begin{equation}
-\frac{1}{n_{1}N_{2}}\sum_{\tau\in\mathcal{T}}\sum_{(i,j)\in\Omega_{\tau}}\delta_{ij}^{\tau}(Y_{ij}^{\tau}\Theta_{ij}-A^{\tau}(\Theta_{ij}))+\left\Vert \Theta\right\Vert _{*,\max}^{\lambda_{*},\lambda_{\max}}\geq-\frac{1}{n_{1}N_{2}}\sum_{\tau\in\mathcal{T}}\sum_{(i,j)\in\Omega_{\tau}}\delta_{ij}^{\tau}(Y_{ij}^{\tau}\widehat{\Theta}_{ij}-A^{\tau}(\widehat{\Theta}_{ij}))+\Vert\widehat{\Theta}\Vert_{*,\max}^{\lambda_{*},\lambda_{\max}},\label{eq:compare_mle}
\end{equation}
which, by rearranging, is equivalent to 
\begin{equation}
\frac{1}{n_{1}N_{2}}\sum_{\tau\in\mathcal{T}}\sum_{(i,j)\in\Omega_{\tau}}\delta_{ij}^{\tau}(A^{\tau}(\widehat{\Theta}_{ij})-Y_{ij}^{\tau}\widehat{\Theta}_{ij})+\Vert\widehat{\Theta}\Vert_{*,\max}^{\lambda_{*},\lambda_{\max}}\leq\frac{1}{n_{1}N_{2}}\sum_{\tau\in\mathcal{T}}\sum_{(i,j)\in\Omega_{\tau}}\delta_{ij}^{\tau}(A^{\tau}(\Theta_{ij})-Y_{ij}^{\tau}\Theta_{ij})+\Vert\Theta\Vert_{*,\max}^{\lambda_{*},\lambda_{\max}}.\label{eq:compare_mle2}
\end{equation}
Now we massage \ref{eq:compare_mle2} into a form that's easier to
work with: 
\begin{equation}
\frac{1}{n_{1}N_{2}}\sum_{\tau\in\mathcal{T}}\sum_{i,j\in\Omega_{\tau}}\delta_{ij}^{\tau}(Y_{ij}^{\tau}(\widehat{\Theta}_{ij}-\Theta_{ij})-(A^{\tau}(\widehat{\Theta}_{ij})-A^{\tau}(\Theta_{ij}))\leq\Vert\Theta\Vert_{*,\max}^{\lambda_{*},\lambda_{\max}}-\Vert\widehat{\Theta}\Vert_{*,\max}^{\lambda_{*},\lambda_{\max}}.
\end{equation}
Unpacking the norms we get
\begin{equation}
\frac{1}{n_{1}N_{2}}\sum_{\tau\in\mathcal{T}}\sum_{i,j\in\Omega_{\tau}}\delta_{ij}^{\tau}\left[(A^{\tau}(\widehat{\Theta}_{ij}^{\tau})-A^{\tau}(\Theta_{ij}^{\tau}))-Y_{ij}^{\tau}(\widehat{\Theta}_{ij}^{\tau}-\Theta_{ij}^{\tau})^{\tau}\right]\leq\lambda_{*}(\left\Vert \Theta\right\Vert _{*}-\Vert\widehat{\Theta}\Vert_{*})+\lambda_{\max}(\left\Vert \Theta\right\Vert _{\max}-\Vert\widehat{\Theta}\Vert_{\max}).\label{eq:compare_mle3}
\end{equation}
Since using the bijection between Bregman divergence and exponential
family we can write 
\begin{equation}
\text{KL}(\widehat{\Theta}_{ij}^{\tau},\Theta_{ij}^{\tau})=A^{\tau}(\widehat{\Theta}_{ij}^{\tau})-A^{\tau}(\Theta_{ij}^{\tau})-(\widehat{\Theta}_{ij}^{\tau}-\Theta_{ij})\nabla A^{\tau}(\Theta_{ij}),
\end{equation}
it follows that 
\begin{equation}
A^{\tau}(\widehat{\Theta}_{ij}^{\tau})-A^{\tau}(\Theta_{ij}^{\tau})=\mathrm{KL}(\widehat{\Theta}_{ij},\Theta_{ij}^{\tau})
\end{equation}
Substitute this back into \ref{eq:compare_mle3}, we get that
\begin{align}
 & \frac{1}{n_{1}N_{2}}\sum_{\tau\in\mathcal{T}}\sum_{i,j\in\Omega_{\tau}}\delta_{ij}^{\tau}\left[(A^{\tau}(\widehat{\Theta}_{ij}^{\tau})-A^{\tau}(\Theta_{ij}^{\tau}))-Y_{ij}^{\tau}(\widehat{\Theta}_{ij}^{\tau}-\Theta_{ij}^{\tau})^{\tau}\right]\\
=\quad & \frac{1}{n_{1}N_{2}}\sum_{\tau\in\mathcal{T}}\sum_{i,j\in\Omega_{\tau}}\delta_{ij}^{\tau}\left[\mathrm{KL}(\widehat{\Theta}_{ij},\Theta_{ij})+(\widehat{\Theta}_{ij}^{\tau}-\Theta_{ij})\nabla A^{\tau}(\Theta_{ij})-Y_{ij}^{\tau}(\widehat{\Theta}_{ij}^{\tau}-\Theta_{ij}^{\tau})^{\tau}\right]\\
\leq\quad & \lambda_{*}(\left\Vert \Theta\right\Vert _{*}-\Vert\widehat{\Theta}\Vert_{*})+\lambda_{\max}(\left\Vert \Theta\right\Vert _{\max}-\Vert\widehat{\Theta}\Vert_{\max}).
\end{align}

Rearranging the terms, we get 
\begin{align}
 & \frac{1}{n_{1}N_{2}}\sum_{\tau\in\mathcal{T}}\sum_{i,j\in[n_{1}]\times[N_{2}]}\delta_{ij}^{\tau}\mathrm{KL}^{\tau}(\widehat{\Theta}_{ij}^{\tau},\Theta_{ij}^{\tau})\\
\leq\quad & \lambda_{*}(\left\Vert \Theta\right\Vert _{*}-\Vert\widehat{\Theta}\Vert_{*})+\lambda_{\max}(\left\Vert \Theta\right\Vert _{\max}-\Vert\widehat{\Theta}\Vert_{\max})+\frac{1}{n_{1}N_{2}}\sum_{\tau\in\mathcal{T}}\sum_{(i,j)\in[n_{1}]\times[N_{2}]}\delta_{ij}^{\tau}(Y_{ij}^{\tau}-\nabla A(\Theta_{ij}^{\tau}))(\widehat{\Theta}_{ij}^{\tau}-\Theta_{ij}^{\tau})\\
\leq\quad & \lambda_{*}(\left\Vert \Theta\right\Vert _{*}-\Vert\widehat{\Theta}\Vert_{*})+\lambda_{\max}(\norm{\Theta}_{\max}-\norm{\widehat{\Theta}}_{\max})+\left\langle \nabla_{\Theta}\mathcal{L}(\Theta\vert Y),\widehat{\Theta}-\Theta\right\rangle \label{eq:proof1-diff1}\\
\leq\quad & \lambda_{*}(\left\Vert \Theta\right\Vert _{*}-\Vert\widehat{\Theta}\Vert_{*})+\lambda_{\max}(\norm{\Theta}_{\max}-\norm{\widehat{\Theta}}_{\max})+\left\Vert \nabla_{\Theta}\mathcal{L}(\Theta\vert Y)\right\Vert \norm{\widehat{\Theta}-\Theta}_{*}\label{eq:proof1-diff2}\\
\leq\quad & \lambda_{*}(\left\Vert \Theta\right\Vert _{*}-\Vert\widehat{\Theta}\Vert_{*})+\lambda_{\max}(\norm{\Theta}_{\max}-\norm{\widehat{\Theta}}_{\max})+\frac{\lambda_{*}}{2}\norm{\widehat{\Theta}-\Theta}_{*}\label{eq:proof1-diff3}\\
\leq\quad & \lambda_{*}(\norm{\mathcal{P}_{\Theta}(\Theta-\widehat{\Theta})}_{*}-\norm{\mathcal{P}_{\Theta}^{\perp}(\Theta-\widehat{\Theta})}_{*})+\frac{\lambda_{*}}{2}\left(\norm{\mathcal{P}_{\Theta}(\Theta-\widehat{\Theta})}_{*}+\norm{\mathcal{P}_{\Theta}^{\perp}(\Theta-\widehat{\Theta})}_{*}\right)+\lambda_{\max}(\norm{\Theta}_{\max}-\norm{\widehat{\Theta}}_{\max})\\
\leq\quad & \frac{3}{2}\lambda_{*}\norm{\mathcal{P}_{\Theta}(\Theta-\widehat{\Theta})}_{*}-\frac{\lambda_{*}}{2}\norm{\mathcal{P}_{\Theta}(\Theta-\widehat{\Theta})}_{*}+\lambda_{\max}(\norm{\Theta}_{\max}-\norm{\widehat{\Theta}}_{\max})\\
\leq\quad & \frac{3}{2}\lambda_{*}\norm{\mathcal{P}_{\Theta}(\Theta-\widehat{\Theta})}_{*}+\lambda_{\max}(\norm{\Theta}_{\max}-\norm{\widehat{\Theta}}_{\max})\\
\leq\quad & \frac{3}{2}\lambda_{*}\sqrt{\text{2rank}(\Theta)}\norm{\Theta-\widehat{\Theta}}_{F}+\lambda_{\max}(\norm{\Theta}_{\max}-\norm{\widehat{\Theta}}_{\max})=\\
\leq\quad & \frac{3}{2}\lambda_{*}\sqrt{\text{2rank}(\Theta)}\norm{\Theta-\widehat{\Theta}}_{F}+\lambda_{\max}\norm{\Theta-\widehat{\Theta}}_{F}\\
=\quad & \left(\frac{3}{2}\lambda_{*}\sqrt{2\text{rank}(\Theta)}+\lambda_{\max}\right)\norm{\Theta-\widehat{\Theta}}_{F},
\end{align}
where we note
\begin{itemize}
\item \ref{eq:proof1-diff1} is because of the fact that 
\begin{align}
\nabla_{\Theta}\mathcal{L}(\Theta,Y) & =\nabla_{\Theta}\left[\sum_{\tau\in\mathcal{T}}\sum_{(i,j)\in[n_{1}]\times[N_{2}]}\frac{1}{n_{1}N_{2}}\delta_{ij}^{\tau}(Y_{ij}^{\tau}\Theta_{ij}^{\tau}-A(\Theta_{ij}^{\tau}))\right]=\sum_{\tau\in\mathcal{T}}\sum_{(i,j)\in[n_{1}]\times[N_{2}]}\left[\delta_{ij}^{\tau}(Y_{ij}-\nabla A(\Theta_{ij}^{\tau}))\right]e_{ij}^{\tau}
\end{align}
where $\{e_{ij}^{\tau}\}$ is the standard basis in $\mathbb{R}^{\left|\mathcal{T}\right|\times n_{1}\times N_{2}}$
\item \ref{eq:proof1-diff2} is due to Cauchy inequality for operator norms
\item \ref{eq:proof1-diff3} is due to the assumption that $\lambda_{*}\geq2\norm{\mathcal{L}_{\Theta}(\Theta\vert Y)}$
\end{itemize}
Then it follows that 
\begin{equation}
\frac{1}{n_{1}N_{2}}\sum_{\tau\in\mathcal{T}}\sum_{i,j\in[n_{1}]\times[N_{2}]}\delta_{ij}^{\tau}\mathrm{KL}^{\tau}(\widehat{\Theta}_{ij}^{\tau},\Theta_{ij}^{\tau})\geq\frac{L_{\gamma}^{2}}{2}\frac{1}{n_{1}N_{2}}\sum_{\tau\in\mathcal{T}}\sum_{i,j\in[n_{1}]\times[N_{2}]}\delta_{ij}^{\tau}(\widehat{\Theta}_{ij}^{\tau}-\Theta)^{2}:=\frac{L_{\gamma}^{2}}{2}\Delta^{2}(\widehat{\Theta}-\Theta).
\end{equation}
So it follows that 
\begin{align}
\Delta^{2}(\widehat{\Theta}-\Theta) & \leq\frac{2}{L_{\gamma}^{2}}\frac{1}{n_{1}N_{2}}\sum_{\tau\in\mathcal{T}}\sum_{i,j\in[n_{1}]\times[N_{2}]}\delta_{ij}^{T}\mathrm{KL}(\widehat{\Theta}_{ij}^{\tau},\Theta_{ij}^{\tau})\leq\left(\frac{3}{L_{\gamma}^{2}}\lambda_{*}\sqrt{2\text{rank}(\Theta)}+\lambda_{\max}\right)\norm{\Theta-\widehat{\Theta}}_{F}.\label{eq:delta_estimate_1-2}
\end{align}

Now we define the threshold $\beta=\frac{946\gamma^{2}\log(n_{1}+N_{2})}{pn_{1}D}$
and distinguish the two following cases:

\paragraph{Case 1. $\boxed{\frac{1}{n_{1}N_{2}}\protect\norm{\widehat{\Theta}-\Theta}_{\Pi,F}<\beta}$}

In this case, the theorem is true.

\paragraph{Case 2. $\boxed{\frac{1}{n_{1}N_{2}}\protect\norm{\widehat{\Theta}-\Theta}_{\Pi,F}\protect\geq\beta}$}

In this case, by  \ref{lem:max-schatten-norm-ineq}, it follows that
\begin{equation}
\norm{\widehat{\Theta}-\Theta}_{*}\leq2\left(\sqrt{8\mathrm{rank}(\Theta)}+\frac{\lambda_{\max}}{\lambda_{*}}\right)\norm{\widehat{\Theta}-\Theta}_{F}.
\end{equation}
Then it follows that $\widehat{\Theta}\in\mathcal{K}(\beta,4(\sqrt{8\mathrm{rank}(\Theta)}+\frac{\lambda_{\max}}{\lambda_{*}})^{2})$,
where 
\begin{equation}
\mathcal{K}(\beta,r):=\left\{ \Xi\in\mathcal{B}_{\infty}(\gamma):\left\Vert \Theta-\Xi\right\Vert _{*}\leq\sqrt{r}\left\Vert \Theta-\Xi\right\Vert _{F}\text{ and }\frac{1}{n_{1}N_{2}}\left\Vert \Xi-\Theta\right\Vert _{\Pi,F}^{2}\geq\beta\right\} .
\end{equation}
Then by  \ref{lem:max-schatten-norm-ineq}, it follows that 
\begin{equation}
\left|\Delta^{2}(\widehat{\Theta},\Theta)-\frac{1}{n_{1}N_{2}}\norm{\widehat{\Theta}-\Theta}_{\Pi,F}^{2}\right|\leq\frac{\norm{\widehat{\Theta}-\Theta}_{\Pi,F}^{2}}{2n_{1}N_{2}}+1392\cdot4\left(\sqrt{8\mathrm{rank}(\Theta)}+\frac{\lambda_{\max}}{\lambda_{*}}\right)^{2}\gamma^{2}(\mathbb{E}[\left\Vert \Sigma_{R}\right\Vert ])^{2}+\frac{5567\gamma^{2}}{n_{1}N_{2}p},
\end{equation}
which after rearrangement becomes
\begin{equation}
\Delta^{2}(\widehat{\Theta},\Theta)\geq\frac{\norm{\widehat{\Theta}-\Theta}_{\Pi,F}^{2}}{2n_{1}N_{2}}-5568\left(\sqrt{8\mathrm{rank}(\Theta)}+\frac{\lambda_{\max}}{\lambda_{*}}\right)^{2}\gamma^{2}(\mathbb{E}[\left\Vert \Sigma_{R}\right\Vert ])^{2}-\frac{5567\gamma^{2}}{n_{1}N_{2}}.\label{eq:delta_estimate_2}
\end{equation}
Then combining \ref{eq:delta_estimate_1-2} and \ref{eq:delta_estimate_2},
it follows that 
\begin{equation}
\frac{\norm{\widehat{\Theta}-\Theta}_{\Pi,F}^{2}}{2n_{1}N_{2}}-5568\left(\sqrt{8\mathrm{rank}(\Theta)}+\frac{\lambda_{\max}}{\lambda_{*}}\right)^{2}\gamma^{2}(\mathbb{E}[\left\Vert \Sigma_{R}\right\Vert ])^{2}-\frac{5567\gamma^{2}}{n_{1}N_{2}}\leq\left(\frac{3}{2}\lambda_{*}\sqrt{2\text{rank}(\Theta)}+\lambda_{\max}\right)\norm{\Theta-\widehat{\Theta}}_{F},
\end{equation}
which after rearranging terms becomes 
\begin{align}
\frac{\norm{\widehat{\Theta}-\Theta}_{\Pi,F}^{2}}{2n_{1}N_{2}} & \leq\underbrace{\left(\frac{3}{L_{\gamma}^{2}}\lambda_{*}\sqrt{2\text{rank}(\Theta)}+\lambda_{\max}\right)}_{(I)}\norm{\Theta-\widehat{\Theta}}_{F}+\underbrace{5568\left(\sqrt{8\mathrm{rank}(\Theta)}+\frac{\lambda_{\max}}{\lambda_{*}}\right)^{2}\gamma^{2}(\mathbb{E}[\left\Vert \Sigma_{R}\right\Vert ])^{2}}_{(II)}+\frac{5567\gamma^{2}}{n_{1}N_{2}p}.
\end{align}

\begin{lem}
\label{lem:term1}The following identity holds:
\begin{equation}
(I)\leq\frac{n_{1}N_{2}}{p}\left(\left(\frac{18\lambda_{*}^{2}}{L_{\gamma}^{4}}+\frac{12\lambda_{*}\lambda_{\max}}{L_{\gamma}^{2}}\right)\text{rank}(\Theta)+\lambda_{\max}^{2}\right)+\frac{1}{4n_{1}N_{2}}\norm{\Theta-\widehat{\Theta}}_{\Pi,F}^{2}.
\end{equation}
\end{lem}
\begin{proof}[proof of \ref{lem:term1}]
Let $\mathcal{H}:=\left(\frac{3}{L_{\gamma}^{2}}\lambda_{*}\sqrt{2\text{rank}(\Theta)}+\lambda_{\max}\right),$
then it follows we can rewrite the term as 
\begin{align}
\mathcal{H}\norm{\Theta-\widehat{\Theta}}_{F} & =\left(\frac{\sqrt{2n_{1}N_{2}}}{\sqrt{p}}\mathcal{H}\right)\left(\frac{\sqrt{p}}{\sqrt{2n_{1}N_{2}}}\norm{\Theta-\widehat{\Theta}}_{F}\right)\\
 & \leq\frac{1}{2}\left(\frac{2n_{1}N_{2}}{p}\mathcal{H}^{2}\right)+\frac{1}{2}\left(\frac{p}{2n_{1}N_{2}}\norm{\Theta-\widehat{\Theta}}_{F}^{2}\right)\\
 & \leq\frac{n_{1}N_{2}}{p}\left(\frac{3}{L_{\gamma}^{2}}\lambda_{*}\sqrt{2\text{rank}(\Theta)}+\lambda_{\max}\right)^{2}+\frac{1}{4n_{1}N_{2}}\norm{\Theta-\widehat{\Theta}}_{\Pi,F}^{2}\\
 & =\frac{n_{1}N_{2}}{p}\left(\frac{18\lambda_{*}^{2}\text{rank}(\Theta)}{L_{\gamma}^{4}}+\frac{6\lambda_{*}\lambda_{\max}}{L_{\gamma}^{2}}\sqrt{2\text{rank}(\Theta)}+\lambda_{\max}^{2}\right)+\frac{1}{4n_{1}N_{2}}\norm{\Theta-\widehat{\Theta}}_{\Pi,F}^{2}\\
 & \leq\frac{n_{1}N_{2}}{p}\left(\left(\frac{18\lambda_{*}^{2}}{L_{\gamma}^{4}}+\frac{12\lambda_{*}\lambda_{\max}}{L_{\gamma}^{2}}\right)\text{rank}(\Theta)+\lambda_{\max}^{2}\right)+\frac{1}{4n_{1}N_{2}}\norm{\Theta-\widehat{\Theta}}_{\Pi,F}^{2}.
\end{align}
\end{proof}
\begin{lem}
\label{lem:term2}The following identity holds:
\begin{equation}
(II)\leq\frac{n_{1}N_{2}}{p}\left[44544 \ \mathrm{rank}(\Theta)+89088 \ \mathrm{rank}(\Theta)\frac{\lambda_{\max}}{\lambda_{*}}+\frac{\lambda_{\max}^{2}}{\lambda_{*}^{2}}\right]\gamma^{2}(\mathbb{E}[\left\Vert \Sigma_{R}\right\Vert ])^{2}
\end{equation}
\end{lem}
\begin{proof}[Proof of \ref{lem:term2}]
Note that 
\begin{align}
5568\left(\sqrt{8\mathrm{rank}(\Theta)}+\frac{\lambda_{\max}}{\lambda_{*}}\right)^{2}\gamma^{2}(\mathbb{E}[\left\Vert \Sigma_{R}\right\Vert ])^{2} & =\left[5568\left(8 \ \text{rank}(\Theta)+2\sqrt{8\text{rank}(\Theta)}\frac{\lambda_{\max}}{\lambda_{*}}+\frac{\lambda_{\max}^{2}}{\lambda_{*}^{2}}\right)\right]\gamma^{2}(\mathbb{E}[\left\Vert \Sigma_{R}\right\Vert ])^{2}\\
 & \leq\frac{n_{1}N_{2}}{p}\left[44544 \ \text{rank}(\Theta)+89088 \ \text{rank}(\Theta)\frac{\lambda_{\max}}{\lambda_{*}}+\frac{\lambda_{\max}^{2}}{\lambda_{*}^{2}}\right]\gamma^{2}(\mathbb{E}[\left\Vert \Sigma_{R}\right\Vert ])^{2}.
\end{align}
\end{proof}
Now using \ref{lem:term1},\ref{lem:term2}, we have 
\begin{align}
 & \frac{1}{4n_{1}N_{2}}\left\Vert \Theta-\widehat{\Theta}\right\Vert _{\Pi,F}^{2}\\
\leq & \frac{n_{1}N_{2}}{p}\left(\left(\frac{18\lambda_{*}^{2}}{L_{\gamma}^{4}}+\frac{12\lambda_{*}\lambda_{\max}}{L_{\gamma}^{2}}\right)\text{rank}(\Theta)+\lambda_{\max}^{2}\right)\\
 & +\frac{n_{1}N_{2}}{p}\left[44544 \ \text{rank}(\Theta)+89088 \ \text{rank}(\Theta)\frac{\lambda_{\max}}{\lambda_{*}}+\frac{\lambda_{\max}^{2}}{\lambda_{*}^{2}}\right]\gamma^{2}(\mathbb{E}[\left\Vert \Sigma_{R}\right\Vert ])^{2}+\frac{5567\gamma^{2}}{n_{1}N_{2}p}\\
\leq & \frac{n_{1}N_{2}}{p}\Bigg[\left(\frac{18\lambda_{*}^{2}}{L_{\gamma}^{4}}+\frac{12\lambda_{*}\lambda_{\max}}{L_{\gamma}^{2}}\right)\text{rank}(\Theta)+\lambda_{\max}^{2}\\
 & \quad\quad+\left(44544 \ \text{rank}(\Theta)+89088 \ \text{rank}(\Theta)\frac{\lambda_{\max}}{\lambda_{*}}+\frac{\lambda_{\max}^{2}}{\lambda_{*}^{2}}\right)\gamma^{2}(\mathbb{E}[\left\Vert \Sigma_{R}\right\Vert ])^{2}\Bigg]+\frac{5567\gamma^{2}}{n_{1}N_{2}p}\\
\leq & \frac{n_{1}N_{2}}{p}\Bigg[\text{rank}(\Theta)\left(\frac{c_{1}}{L_{\gamma}^{4}}\left(\lambda_{*}^{2}+\lambda_{*}\lambda_{\max}L_{\gamma}^{2}\right)+\left(c_{2}+c_{3}\frac{\lambda_{\max}}{\lambda_{*}}\right)\gamma^{2}(\mathbb{E}[\left\Vert \Sigma_{R}\right\Vert ])^{2}\right)\\
 & \quad\quad+\lambda_{\max}^{2}+\frac{\lambda_{\max}^{2}}{\lambda_{*}^{2}}\gamma^{2}(\mathbb{E}[\left\Vert \Sigma_{R}\right\Vert ])^{2}\Bigg]+\frac{5567\gamma^{2}}{n_{1}N_{2}p}\\
\leq & \frac{Cn_{1}N_{2}}{p}\left[\text{rank}(\Theta)\left(\frac{1}{L_{\gamma}^{4}}\left(\lambda_{*}^{2}+\lambda_{*}\lambda_{\max}L_{\gamma}^{2}\right)+\left(1+\frac{\lambda_{\max}}{\lambda_{*}}+\frac{\lambda_{\max}^{2}}{\lambda_{*}^{2}}\right)\gamma^{2}(\mathbb{E}[\left\Vert \Sigma_{R}\right\Vert ])^{2}\right)+\lambda_{\max}^{2}\right]+\frac{5567\gamma^{2}}{n_{1}N_{2}p}\\
\leq & \frac{C}{p}\left[n_{1}N_{2}\left(\text{rank}(\Theta)\left(\frac{1}{L_{\gamma}^{4}}\left(\lambda_{*}^{2}+\lambda_{*}\lambda_{\max}L_{\gamma}^{2}\right)+\left(1+\frac{\lambda_{\max}}{\lambda_{*}}+\frac{\lambda_{\max}^{2}}{\lambda_{*}^{2}}\right)\gamma^{2}(\mathbb{E}[\left\Vert \Sigma_{R}\right\Vert ])^{2}\right)+\lambda_{\max}^{2}\right)+\frac{\gamma^{2}}{n_{1}N_{2}}\right]
\end{align}
Therefore, the inequality $a+b\leq2(a\lor b)$ for $a,b\in\mathbb{R}$
yields 
\begin{align}
 & \frac{1}{n_{1}N_{2}}\norm{\Theta-\widehat{\Theta}}_{\Pi,F}^{2}\\
\leq & \frac{2C}{p}\max\Bigg\{2\max\Bigg\{ n_{1}N_{2}\text{rank}(\Theta)\left(\frac{1}{L_{\gamma}^{4}}\left(\lambda_{*}^{2}+\lambda_{*}\lambda_{\max}L_{\gamma}^{2}\right)+\left(1+\frac{\lambda_{\max}}{\lambda_{*}}+\frac{\lambda_{\max}^{2}}{\lambda_{*}^{2}}\right)\gamma^{2}(\mathbb{E}[\left\Vert \Sigma_{R}\right\Vert ])^{2}\right)\\
 & \quad+n_{1}N_{2}\text{rank}(\Theta)\lambda_{\max}^{2},\frac{\gamma^{2}}{n_{1}N_{2}}\Bigg\},\frac{\gamma^{2}\log(n_{1}+N_{2})}{n_{1}N_{2}}\Bigg\}\\
\leq & \frac{4C}{p}\max\Bigg\{\max\Bigg\{ n_{1}N_{2}\text{rank}(\Theta)\left(\frac{1}{L_{\gamma}^{4}}\left(\lambda_{*}^{2}+\lambda_{*}\lambda_{\max}L_{\gamma}^{2}\right)+\left(1+\frac{\lambda_{\max}}{\lambda_{*}}+\frac{\lambda_{\max}^{2}}{\lambda_{*}^{2}}\right)\gamma^{2}(\mathbb{E}[\left\Vert \Sigma_{R}\right\Vert ])^{2}\right)\\
 & \quad+n_{1}N_{2}\text{rank}(\Theta)\lambda_{\max}^{2},\frac{\gamma^{2}}{n_{1}N_{2}}\Bigg\},\frac{\gamma^{2}\log(n_{1}+N_{2})}{n_{1}N_{2}}\Bigg\}\\
\leq & \frac{C_{*}}{p}\max\Bigg\{ n_{1}N_{2}\text{rank}(\Theta)\left(\frac{1}{L_{\gamma}^{4}}\left(\lambda_{*}^{2}+\lambda_{*}\lambda_{\max}L_{\gamma}^{2}\right)+\left(1+\frac{\lambda_{\max}}{\lambda_{*}}+\frac{\lambda_{\max}^{2}}{\lambda_{*}^{2}}\right)\gamma^{2}(\mathbb{E}[\left\Vert \Sigma_{R}\right\Vert ])^{2}+\lambda_{\max}^{2}\right),\\
 & \quad\max\Bigg\{\frac{\gamma^{2}}{n_{1}N_{2}},\frac{\gamma^{2}\log(n_{1}+N_{2})}{n_{1}N_{2}}\Bigg\}\Bigg\}\\
\leq & \frac{C_{*}}{p}\max\Bigg\{ n_{1}N_{2}\text{rank}(\Theta)\left(\frac{1}{L_{\gamma}^{4}}\left(\lambda_{*}^{2}+\lambda_{*}\lambda_{\max}L_{\gamma}^{2}\right)+\left(1+\frac{\lambda_{\max}}{\lambda_{*}}+\frac{\lambda_{\max}^{2}}{\lambda_{*}^{2}}\right)\gamma^{2}(\mathbb{E}[\left\Vert \Sigma_{R}\right\Vert ])^{2}\right)+\lambda_{\max}^{2},\frac{\gamma^{2}\log(n_{1}+N_{2})}{n_{1}N_{2}}\Bigg\},
\end{align}
where the second inequality follows from $\max(a,b)\leq\max(a,\eta\cdot b)$
for $\eta>1$ and the third inequality follow commutativity of the
$\max$ function. This completes the proof of \ref{thm:2}.  \hfill $\square$

\subsection*{Proof of  \ref{thm:3}}
For ease of notation, we let 
\begin{equation}
\mathcal{H}=\left(\frac{1}{L_{\gamma}^{4}}\left(\lambda_{*}^{2}+\lambda_{*}\lambda_{\max}L_{\gamma}^{2}+L_{\gamma}^{4}\lambda_{\max}^{2}\right)+\left(1+\frac{\lambda_{\max}}{\lambda_{*}}+\frac{\lambda_{\max}^{2}}{\lambda_{*}^{2}}\right)\gamma^{2}(\mathbb{E}[\left\Vert \Sigma_{R}\right\Vert ])^{2}\right)
\end{equation}

Since we let 
\begin{equation}
\lambda_{*}=2c\left(\frac{(U_{\gamma}\cup K)(\sqrt{n_{1}\lor N_{2}}+(\log(n_{1}\lor N_{2}))^{3/2})}{n_{1}N_{2}}\right)\text{ and }\lambda_{\max}\leq\kappa\lambda_{*},
\end{equation}
it follows that 
\begin{align}
\frac{1}{L_{\gamma}^{4}}\left(\lambda_{*}^{2}+\lambda_{*}\lambda_{\max}L_{\gamma}^{2}+L_{\gamma}^{4}\lambda_{\max}^{2}\right) & \leq\frac{1}{L_{\gamma}^{4}}\left(4c^{2}\frac{(U_{\gamma}\cup K)^{2}(\sqrt{n_{1}\lor N_{2}}+(\log(n_{1}\lor N_{2}))^{3/2})^{2}}{\left(n_{1}N_{2}\right)^{2}}(1+\kappa L_{\gamma}^{2}+\kappa^{2}L_{\gamma}^{4})\right)\\
 & \leq\frac{c_{1}(1+\kappa L_{\gamma}^{2}+\kappa^{2}L_{\gamma}^{4})}{L_{\gamma}^{4}}\left[(U_{\gamma}\cup K)^{2}\frac{(\sqrt{n_{1}\lor N_{2}}+(\log(n_{1}\lor N_{2}))^{3/2})^{2}}{\left(n_{1}N_{2}\right)^{2}}\right].
\end{align}
And that
\begin{align}
\left(1+\frac{\lambda_{\max}}{\lambda_{*}}\right)\gamma^{2}(\mathbb{E}[\left\Vert \Sigma_{R}\right\Vert ])^{2} & \leq\left[1+\kappa\right]\gamma^{2}\cdot c_{\Sigma}\left(\frac{\sqrt{n_{1}\lor N_{2}}+\sqrt{\log(n_{1}\lor N_{2})}}{n_{1}N_{2}}\right)^{2}\\
 & \leq\frac{(\sqrt{n_{1}\lor N_{2}}+(\log(n_{1}\lor N_{2}))^{3/2})^{2}}{\left(n_{1}N_{2}\right)^{2}}c_{\Sigma}\left(\kappa+1\right)\gamma^{2}.
\end{align}
And that 
\begin{align}
\frac{\lambda_{\max}^{2}}{\lambda_{*}^{2}}\gamma^{2}(\mathbb{E}[\left\Vert \Sigma_{R}\right\Vert ])^{2} & \leq\kappa^{2}\cdot c_{\Sigma}^{2}\left(\frac{\sqrt{n_{1}\lor N_{2}}+\sqrt{\log(n_{1}\lor N_{2})}}{n_{1}N_{2}}\right)^{2}\gamma^{2}\\
 & \leq c_{\Sigma}^{2}\kappa^{2}\gamma^{2}\left(\frac{(\sqrt{n_{1}\lor N_{2}}+(\log(n_{1}\lor N_{2}))^{3/2})^{2}}{\left(n_{1}N_{2}\right)^{2}}\right).
\end{align}
Therefore, it follows that 
\begin{align}
\mathcal{H} & \leq\frac{(\sqrt{n_{1}\lor N_{2}}+(\log(n_{1}\lor N_{2}))^{3/2})^{2}}{\left(n_{1}N_{2}\right)^{2}}\left(\frac{c_{1}(1+\kappa L_{\gamma}^{2}+\kappa^{2}L_{\gamma}^{4})}{L_{\gamma}^{4}}(U_{\gamma}\lor K)^{2}+c_{\Sigma}(\kappa+1)\gamma^{2}+c_{\Sigma}^{2}\kappa^{2}\gamma^{2}+4\right)\\
 & \leq\frac{(\sqrt{n_{1}\lor N_{2}}+(\log(n_{1}\lor N_{2}))^{3/2})^{2}}{\left(n_{1}N_{2}\right)^{2}}\left((U_{\gamma}\lor K)^{2}\frac{c_{1}(1+\kappa L_{\gamma}^{2}+\kappa^{2}L_{\gamma}^{4})}{L_{\gamma}^{4}}+\gamma^{2}(1+c_{\Sigma}\kappa+c_{\Sigma}^{2}\kappa^{2})\right)\\
 & \leq\frac{C(\sqrt{n_{1}\lor N_{2}}+(\log(n_{1}\lor N_{2}))^{3/2})^{2}}{\left(n_{1}N_{2}\right)^{2}}\left((U_{\gamma}\lor K)^{2}\frac{(1+\kappa L_{\gamma}^{2}+\kappa^{2}L_{\gamma}^{4})}{L_{\gamma}^{4}}+\gamma^{2}(1+\kappa+\kappa^{2})\right)\\
 & \leq C\left(\frac{n_{1}\lor N_{2}}{(n_{1}N_{2})^{2}}+\frac{\log^{3}(n_{1}\lor N_{2})}{\left(n_{1}N_{2}\right)^{2}}+\frac{2\sqrt{n_{1}\lor N_{2}}(\log(n_{1}\lor N_{2}))^{3/2}}{\left(n_{1}N_{2}\right)^{2}}\right)\left((U_{\gamma}\lor K)^{2}\frac{(1+\kappa L_{\gamma}^{2}+\kappa^{2}L_{\gamma}^{4})}{L_{\gamma}^{4}}+\gamma^{2}(1+\kappa+\kappa^{2})\right)\\
 & \leq C\left(\frac{n_{1}\lor N_{2}}{(n_{1}N_{2})^{2}}+\frac{\log^{3}(n_{1}\lor N_{2})}{\left(n_{1}N_{2}\right)^{2}}\right)\left((U_{\gamma}\lor K)^{2}\frac{(1+\kappa L_{\gamma}^{2}+\kappa^{2}L_{\gamma}^{4})}{L_{\gamma}^{4}}+\gamma^{2}(1+\kappa+\kappa^{2})\right).
\end{align}
 Hence it follows that 
\begin{align}
 & \frac{1}{n_{1}N_{2}}\norm{\Theta-\widehat{\Theta}}_{\Pi,F}\\
\leq & \frac{C}{p}\max\left\{ n_{1}N_{2}\left[\text{rank}(\Theta)\mathcal{H}+\lambda_{\max}^{2}\right],\frac{\gamma^{2}\log(n_{1}+N_{2})}{n_{1}N_{2}}\right\} \\
\leq & \frac{\tilde{C}}{p}\max\left\{ \text{rank}(\Theta)\left(\frac{n_{1}\lor N_{2}}{n_{1}N_{2}}+\frac{\log^{3}(n_{1}\lor N_{2})}{n_{1}N_{2}}\right)\left((U_{\gamma}\lor K)^{2}\frac{(1+\kappa L_{\gamma}^{2}+\kappa^{2}L_{\gamma}^{4})}{L_{\gamma}^{4}}+\gamma^{2}(1+\kappa+\kappa^{2})\right),\frac{\gamma^{2}\log(n_{1}+N_{2})}{n_{1}N_{2}}\right\} \\
= & \frac{\tilde{C}\text{rank}(\Theta)}{p}\left(\frac{n_{1}\lor N_{2}}{n_{1}N_{2}}+\frac{\log^{3}(n_{1}\lor N_{2})}{n_{1}N_{2}}\right)\left((U_{\gamma}\lor K)^{2}\frac{(1+\kappa L_{\gamma}^{2}+\kappa^{2}L_{\gamma}^{4})}{L_{\gamma}^{4}}+\gamma^{2}(1+\kappa+\kappa^{2})\right)\\
= & \frac{\tilde{C}\text{rank}(\Theta)(n_{1}\lor N_{2})}{pn_{1}N_{2}}\left(1+\frac{\log^{3}(n_{1}\lor N_{2})}{n_{1}\lor N_{2}}\right)\left((U_{\gamma}\lor K)^{2}\frac{(1+\kappa L_{\gamma}^{2}+\kappa^{2}L_{\gamma}^{4})}{L_{\gamma}^{4}}+\gamma^{2}(1+\kappa+\kappa^{2})\right).
\end{align}
Also, using the fact the $p\left\Vert A\right\Vert _{F}\leq\left\Vert A\right\Vert _{|\Pi,F},$
for any matrix $A,$ it follows that 
\begin{align}
\frac{1}{n_{1}N_{2}}\norm{\Theta-\widehat{\Theta}}_{F} & \leq\frac{\tilde{C}\text{rank}(\Theta)}{p^{2}}\left(\frac{n_{1}\lor N_{2}}{n_{1}N_{2}}+\frac{\log^{3}(n_{1}\lor N_{2})}{n_{1}N_{2}}\right)\left((U_{\gamma}\lor K)^{2}\frac{(1+\kappa L_{\gamma}^{2}+\kappa^{2}L_{\gamma}^{4})}{L_{\gamma}^{4}}+\gamma^{2}(1+\kappa+\kappa^{2})\right)\\
 & =\frac{\tilde{C}\text{rank}(\Theta)(n_{1}\lor N_{2})}{p^{2}n_{1}N_{2}}\left(1+\frac{\log^{3}(n_{1}\lor N_{2})}{n_{1}\lor N_{2}}\right)\left((U_{\gamma}\lor K)^{2}\frac{(1+\kappa L_{\gamma}^{2}+\kappa^{2}L_{\gamma}^{4})}{L_{\gamma}^{4}}+\gamma^{2}(1+\kappa+\kappa^{2})\right).
\end{align}
This completes the proof of \ref{thm:3} \hfill $\square$

\clearpage

\clearpage

\section*{Appendix B: Technical Lemmas \label{sec:appendix-lemmas}}
\begin{lem}
\label{lem:projection contraction}Let $\mathcal{H}$ be a Hilbert
space and $\mathcal{P}$ be an orthogonal operator. Then $\left\Vert \mathcal{P}(f)\right\Vert \leq\left\Vert f\right\Vert .$
\end{lem}
\begin{proof}
Note that by Cauchy Schwartz inequality, we have 
\begin{equation}
\left\Vert \mathcal{P}(f)\right\Vert ^{2}=\left\langle \mathcal{P}(f),\mathcal{P}(f)\right\rangle =\left\langle \mathcal{P}(f),f\right\rangle \leq\left\Vert \mathcal{P}(f)\right\Vert \left\Vert f\right\Vert .
\end{equation}
The result follows by dividing both size by $\left\Vert \mathcal{P}(f)\right\Vert .$
\end{proof}
\begin{lem}
\label{lem:lp norm equivalence}For $1\leq p<q,$ the following inequality
holds 
\begin{equation}
\left\Vert x\right\Vert _{q}\leq\left\Vert x\right\Vert _{p}\leq n^{\frac{1}{p}-\frac{1}{q}}\left\Vert x\right\Vert _{q}
\end{equation}
for $x\in\mathbb{R}^{n}.$
\end{lem}
\begin{proof}
We first show that $\left\Vert x\right\Vert _{q}\leq\left\Vert x\right\Vert _{p}.$
Without loss of generality, it suffices to assume that $\left\Vert x\right\Vert _{p}=1$
since $\left\Vert x\right\Vert _{q}\leq\left\Vert x\right\Vert _{p}$
if and only if $\left\Vert \frac{x}{\left\Vert x\right\Vert _{p}}\right\Vert _{q}\leq\left\Vert \frac{x}{\left\Vert x\right\Vert _{p}}\right\Vert _{p}=1.$
For ease of notation, let $z=x/\left\Vert x\right\Vert _{p}.$ Note
that $\left\Vert z\right\Vert _{q}\leq1\implies z_{i}\leq1$ for all
$i\in1,...,n.$ Now since $x^{q}\leq x^{p}$ for all $x\in(0,1)$,
it follows that 
\begin{equation}
\left\Vert z\right\Vert _{q}=\left(\sum_{i=1}^{n}\left|z_{i}\right|^{q}\right)^{\frac{1}{q}}\leq\left(\sum_{i=1}^{n}\left|z_{i}\right|^{p}\right)^{1/q}=\left\Vert z\right\Vert _{p}^{1/q}=1.
\end{equation}
The result follows by multiplying both sides by $\left\Vert x\right\Vert _{p}.$

Next, we show that $\left\Vert x\right\Vert _{p}\leq n^{1/p-1/q}\left\Vert x\right\Vert _{q}.$
This follows from Holder's inequality which states that for $r>1,$
\begin{equation}
\sum_{i=1}^{n}\left|a_{i}\right|\left|b_{i}\right|\leq\left(\sum_{i=1}^{n}\left|a_{i}\right|^{r}\right)^{\frac{1}{r}}\left(\sum_{i=1}^{n}\left|b_{i}\right|^{\frac{r}{r-1}}\right)^{1-\frac{1}{r}}.\label{eq:holder}
\end{equation}
Apply \ref{eq:holder} to $a_{i}=\left|x_{i}\right|^{p},b_{i}=1$
and $r=\frac{q}{p}>1$ and we get 
\begin{equation}
\sum_{i=1}^{n}\left|x_{i}\right|^{p}1\leq\left(\sum_{i=1}^{n}\left(\left|x_{i}\right|^{p}\right)^{\frac{q}{p}}\right)^{\frac{p}{q}}\left(\sum_{i=1}^{n}1^{\frac{q}{q-p}}\right)^{1-\frac{p}{q}}=\left(\sum_{i=1}^{n}\left|x_{i}\right|^{p}\right)^{\frac{p}{q}}n^{1-\frac{p}{q}}.
\end{equation}
Taking the $p$-th root on both sides yields 
\begin{equation}
\left\Vert x\right\Vert _{p}=\left(\sum_{i=1}^{n}\left|x_{i}\right|^{p}\right)^{1/p}\leq\left[\left(\sum_{i=1}^{n}\left|x_{i}\right|^{q}\right)^{\frac{p}{q}}n^{1-\frac{p}{q}}\right]^{1/p}=\left(\sum_{i=1}^{n}\left|x_{i}\right|^{q}\right)^{\frac{1}{q}}\left(n^{1-\frac{p}{q}}\right)^{\frac{1}{p}}=\left\Vert x\right\Vert _{q}n^{\frac{1}{p}-\frac{1}{q}}.
\end{equation}
\end{proof}
\begin{lem}
\label{lem:schatten and frob ineq}Let $A\in\mathrm{Mat}_{\mathbb{R}}(m\times n)$, 
then the following inequality holds:
\begin{equation}
\left\Vert A\right\Vert _{*}\leq\sqrt{\mathrm{rank}(A)}\left\Vert A\right\Vert _{F}.
\end{equation}
\end{lem}
\begin{proof}
Let $U\Sigma V^{*}=A$ be the singular value decomposition of $A.$
Note that $\left\Vert A\right\Vert _{*}=\sum_{i=1}^{n}\sigma_{i}(A)=\sum_{i=1}^{r}\Sigma_{i,i}=\left\Vert \text{diag}(\Sigma)\right\Vert _{\ell_{1}}.$
On the other hand, note that 
\begin{equation}
\left\Vert A\right\Vert _{2}=\text{tr}(A^{T}A)=\text{tr}(V\Sigma U^{*}U\Sigma V^{*})=\text{tr}(V\Sigma^{2}V^{*})=\text{tr}(\Sigma^{2}V^{*}V)=\text{tr}(\Sigma^{2})=\sum_{i=1}^{r}\Sigma_{i,i}^{2}=\left\Vert \text{diag}(\Sigma)\right\Vert _{\ell_{2}}
\end{equation}
Then, the result follows from an application of \ref{lem:lp norm equivalence}
to $\text{diag}(\Sigma).$
\end{proof}
\global\long\def\wt#1{\widetilde{#1}}%

\begin{lem}
Let $A,B$ be compatible matrices and 
\begin{equation}
U\Sigma V^{*}=\begin{bmatrix}U & \widetilde{U}\end{bmatrix}\begin{bmatrix}\Sigma & 0\\
0 & 0
\end{bmatrix}\begin{bmatrix}V\\
\wt V
\end{bmatrix}=A=\sum_{i=1}^{\text{rank}(A)}\sigma_{i}(A)u_{i}v_{i}^{*}
\end{equation}
 be the fat-version of singular value decomposition of $A.$ Let 
\begin{equation}
T_{A}=\left\langle u_{k}x^{*},yv_{k}^{*}\vert k=1,...r,x\in\mathbb{R}^{n_{2}},y\in\mathbb{R}^{n_{1}}\right\rangle ,
\end{equation}
be the generating set of rank $r$ matrices spanned by $A$'s singular
vectors. Let $\mathcal{P}_{T}(\cdot)$ be the orthogonal projection
onto $T$. Then the following (in)equalities hold:

\begin{enumerate}
\item $\mathcal{P}_{T_{A}}(B)=P_{U}B+BP_{V}-P_{U}BP_{V}=UU^{*}B+BVV^{*}-UU^{*}BVV^{*}$,
\item $\mathcal{P}_{T_{A}^{\perp}}(B)=(\mathcal{I}-\mathcal{P}_{T})(B)=(I_{n_{1}}-P_{U})X(I_{n_{2}}-P_{V})$,
\item $\mathrm{rank}(\mathcal{P}_{T_{A}}(B))\leq2\mathrm{rank}(A)$,
\item $\left\Vert \mathcal{P}_{T_{A}}(B)\right\Vert _{*}\leq\sqrt{2\mathrm{rank}(A)}\left\Vert B\right\Vert _{F}$
for compatible real matrices $A$ and $B$, and
\item $\left\Vert A\right\Vert _{*}-\left\Vert B\right\Vert _{*}\leq\norm{\mathcal{P}_{T_{A}}(A-B)}_{*}-\norm{\mathcal{P}_{T_{A}}^{\perp}(A-B)}_{*}.$
\end{enumerate}
\end{lem}
\begin{proof}

\begin{enumerate}
\item Suppose $x\in T$, then 
\begin{align}
x & =\sum_{i\in\left|I\right|,\left|I\right|<\infty}\alpha_{i}u_{i}x_{i}^{*}+\sum_{j\in\left|J\right|,\left|J\right|<\infty}\beta_{j}yv_{j}^{*}\tag{ for some \ensuremath{v_{j}\in}\{\ensuremath{v_{1}},...,\ensuremath{v_{k}}\},\ensuremath{u_{i}\in}\{\ensuremath{u_{1}},\ensuremath{\ldots},\ensuremath{u_{k}}\}}\\
 & =\sum_{i=1}^{r}\tilde{\alpha}_{i}u_{i}\tilde{x}_{i}^{*}+\sum_{j=1}^{r}\tilde{\beta}_{j}\tilde{y}_{i}v_{j}^{*}=\sum_{i=1}^{r}\tilde{\alpha}_{i}u_{i}\bigg(\sum_{j=1}^{n_{2}}\theta_{1,j}v_{j}^{*}\bigg)+\sum_{j=1}^{r}\tilde{\beta}_{j}\bigg(\sum_{i=1}^{n_{1}}\theta_{2,i}u_{i}\bigg)v_{j}^{*}\\
 & =\sum_{i=1}^{r}\sum_{j=1}^{n_{2}}\tilde{\alpha}_{i}\tilde{\theta}_{1,j}u_{i}v_{j}^{*}+\sum_{j=1}^{r}\sum_{i=1}^{n_{1}}\tilde{\beta}_{j}\theta_{2,i}u_{i}v_{j}^{*}=\sum_{i=1}^{r}\sum_{j=1}^{n_{2}}\gamma_{1,i,j}u_{i}v_{j}^{*}+\sum_{i=1}^{n_{1}}\sum_{j=1}^{r}\gamma_{2,i,j}u_{i}v_{j}^{*}.
\end{align}
Then it follows that 
\begin{equation}
x\in\text{span}\{u_{i}v_{j}^{*}\vert1\leq i\leq r,1\leq j\leq n_{2}\text{ or }1\leq i\leq n_{1},1\leq j\leq r\},
\end{equation}
the other direction follows using the same argument. Therefore, 
\begin{equation}
T_{A}=\text{span}\{u_{i}v_{j}^{*}\vert1\leq i\leq r,1\leq j\leq n_{2}\text{ or }1\leq i\leq n_{1},1\leq j\leq r\}.
\end{equation}
Now we calculate the projection onto $T$, $\mathcal{P}_{T}:$ using
the projection formula we have that for any matrix $B\in\mathrm{Mat}_{\mathbb{R}}(n_{1},n_{2})$,

\begin{align}
\mathcal{P}_{T_{A}}(B) & =\sum_{1\leq i\leq r,1\leq j\leq n_{2}\text{ or }1\leq i\leq n_{1},1\leq j\leq r}\left\langle B,u_{i}v_{j}^{*}\right\rangle u_{i}v_{j}^{*}\\
 & =\underbrace{\sum_{i=1}^{r}\sum_{j=1}^{r}\left\langle B,u_{i}v_{j}^{*}\right\rangle u_{i}v_{j}^{*}}_{\mathrm{(I)}}+\underbrace{\sum_{i=1}^{r}\sum_{j=r+1}^{n_{2}}\left\langle B,u_{i}v_{j}^{*}\right\rangle u_{i}v_{j}^{*}}_{\mathrm{(II)}}+\underbrace{\sum_{i=r+1}^{n_{1}}\sum_{j=1}^{r}\left\langle B,u_{i}v_{j}^{*}\right\rangle u_{i}v_{j}^{*}}_{\mathrm{(III)}}.
\end{align}
We analyze it term by term:

\begin{align}
\text{[I]} & =\sum_{i=1}^{r}\sum_{j=1}^{r}\text{tr}(X^{T}u_{i}v_{j}^{*})u_{i}v_{j}^{*}=\sum_{i=1}^{r}\sum_{j=1}^{r}\text{tr}(v_{j}u_{i}^{*}X)u_{i}v_{j}^{*}=\sum_{i=1}^{r}\sum_{j=1}^{r}\text{tr}(u_{i}^{*}Xv_{j})u_{i}v_{j}^{*}=\sum_{i=1}^{r}\sum_{j=1}^{r}u_{i}u_{i}^{*}Xv_{j}v_{j}^{*}\\
 & =UU^{\top}XVV^{\top}=P_{U}XP_{V}.\\
\text{[II]} & =\sum_{i=1}^{r}\sum_{j=r+1}^{n_{2}}\left\langle X,u_{i}v_{j}^{*}\right\rangle u_{i}v_{j}^{*}=\left(\sum_{i=1}^{r}u_{i}u_{i}^{*}\right)X\left(\sum_{j=r+1}^{n_{2}}v_{j}v_{j}^{*}\right)=UU^{\top}X(I-VV^{\top})=P_{U}XP_{V^{\perp}}\\
\mathrm{[III]} & =\sum_{i=r+1}^{n_{1}}\sum_{j=1}^{r}\left\langle X,u_{i}v_{j}^{*}\right\rangle u_{i}v_{j}^{*}=\left(\sum_{i=r+1}^{n_{1}}u_{i}u_{i}^{*}\right)X\left(\sum_{j=1}^{r}v_{j}v_{j}^{*}\right)=(I-UU^{\top})XVV^{\top}=P_{U^{\perp}}XP_{V}.
\end{align}

Combined the terms and we get the
\begin{equation}
\mathcal{P}_{T_{A}}(B)=P_{U}BP_{V}+P_{U}BP_{V^{\perp}}+P_{U^{\top}}BP_{V}=P_{U}B+BP_{V}-P_{U}XP_{V}=UU^{*}B+BVV^{*}-UU^{*}BVV^{*}
\end{equation}
as desired.
\item This is because of by orthogonal decomposition, we have $\mathrm{Mat}_{\mathbb{R}}(n_{1},n_{2})=T\oplus T^{\perp},$
we have 
\begin{equation}
\mathcal{P}_{T_{A}^{\perp}}(B)=(\mathcal{I}-\mathcal{P}_{T}(B))=B-P_{U}B-BP_{V}+P_{U}BP_{V}=(I-P_{U})B(I-P_{V}).
\end{equation}
\item Note that 
\begin{align}
\text{rank}(\mathcal{P}_{T_{A}}(B)) & =\text{rank}(UU^{*}B+BVV^{*}-UU^{*}BVV^{*})=\text{rank}(UU^{*}B+(I-UU^{*})BVV^{*})\\
 & \leq\text{rank}(UU^{*}B)+\text{rank}((I-UU^{*})BVV^{*})\leq\text{rank}(U)+\text{rank}(V)\leq2\text{rank}(A),
\end{align}
where the second to last inequality follows by keeping applying the
basic inequality $\mathrm{rank}(AB)\leq\min\{\text{rank}(A),\mathrm{rank}(B)\}.$
\item Note that by  \ref{lem:schatten and frob ineq} and part-(3)
\begin{align}
\left\Vert \mathcal{P}_{T_{A}}(B)\right\Vert _{*} & \leq\sqrt{\text{rank}(\mathcal{P}_{T_{A}}(B))}\left\Vert \mathcal{P}_{T_{A}}(B)\right\Vert _{F}\leq\sqrt{2\text{rank}(A)}\left\Vert \mathcal{P}_{T_{A}}(B)\right\Vert _{F}=\sqrt{2\text{rank}(A)}\left|\left\langle \mathcal{P}_{T_{A}}(B),\mathcal{P}_{T_{A}}(B)\right\rangle \right|^{1/2}\\
 & =\sqrt{2\text{rank}(A)}\left|\left\langle \mathcal{P}_{T_{A}}^{*}\mathcal{P}_{T_{A}}(B),B\right\rangle \right|^{1/2}=\sqrt{2\text{rank}(A)}\left|\left\langle \mathcal{P}_{T_{A}}(B),B\right\rangle \right|^{1/2}\leq\sqrt{2\text{rank}(A)}\left\Vert B\right\Vert ,
\end{align}

\noindent where the last inequality follows from \ref{lem:projection contraction}.
\item Note that 
\begin{align}
\left\Vert B\right\Vert _{*} & =\left\Vert A+B-A\right\Vert _{*}=\left\Vert A+\mathcal{P}_{T_{A}}^{\perp}(B-A)+\mathcal{P}_{T_{A}}(B-A)\right\Vert _{*}\geq\left\Vert A+\mathcal{P}_{T_{A}}^{\perp}(B-A)\right\Vert _{*}-\left\Vert \mathcal{P}_{T_{A}}(B-A)\right\Vert _{*}.
\end{align}

\begin{claim}
\label{claim:pythagran of nuclear norm}$\left\Vert A+\mathcal{P}_{T_{A}}^{\perp}(B-A)\right\Vert _{*}=\left\Vert A\right\Vert _{*}+\left\Vert \mathcal{P}_{T_{A}}^{\perp}(B-A)\right\Vert _{*}.$
\end{claim}
\begin{proof}
Let $\widetilde{U}\widetilde{\Sigma}\widetilde{V}^{*}=(B-A)$ be the
singular value decomposition of $B-A.$ Then note that 
\begin{align}
A+\mathcal{P}_{T_{A}^{\perp}}(B-A) & =U\Sigma V^{*}+\left[\widetilde{U}\widetilde{U}^{*}U_{B-A}\right]\Sigma_{B-A}\left[V_{B-A}^{*}\wt V\wt V^{*}\right].\label{eq:svd-qr}
\end{align}
Let $Q_{1}R_{1}=\wt U^{*}U_{B-A}$ and $Q_{2}R_{2}=\widetilde{V}^{*}V_{B-A}^{*}$
be two thin QR decompositions, then it follows that 
\begin{equation}
\left[\widetilde{U}\widetilde{U}^{*}U_{B-A}\right]\Sigma_{B-A}\left[V_{B-A}^{*}\wt V\wt V^{*}\right]=\widetilde{U}Q_{1}R_{1}\Sigma_{B-A}R_{2}^{*}Q_{2}^{*}\widetilde{V}^{*}=(\widetilde{U}Q_{1})\widetilde{\Sigma}_{B-A}(\widetilde{V}Q_{2})^{*},
\end{equation}
where the last equality follows the fact that the product of an upper
triangular and lower triangular matrix is a diagonal matrix. We substitute
this equation back to \ref{eq:svd-qr} we get 
\begin{equation}
A+\mathcal{P}_{T_{A}^{\perp}}(B-A)=\begin{bmatrix}U & \widetilde{U}Q_{1}\end{bmatrix}\begin{bmatrix}\Sigma & 0\\
0 & \widetilde{\Sigma}_{B-A}
\end{bmatrix}\begin{bmatrix}V & \widetilde{V}Q_{2}\end{bmatrix}^{*}.
\end{equation}
since ($\widetilde{U}Q_{1})^{*}(\widetilde{U}Q_{1})=I$, and $(\widetilde{V}Q_{2})^{*}(\widetilde{V}Q_{2})=I$
and the columns in $U,\widetilde{U}Q_{1}$ and in $V,\widetilde{V}Q_{2}$
are orthogonal to each other by construction, it follows that 
\begin{equation}
\left\Vert A+\mathcal{P}_{T_{A}^{\perp}}(B-A)\right\Vert _{*}=\left\Vert \text{diag}(\Sigma)\right\Vert _{\ell_{1}}+\Vert\text{diag}(\widetilde{\Sigma}_{B-A})\Vert_{\ell_{1}}=\left\Vert A\right\Vert _{*}+\left\Vert \mathcal{P}_{T_{A}^{\perp}}(B-A)\right\Vert _{*}
\end{equation}
as desired.
\end{proof}
\begin{claim}
An application of  \ref{claim:pythagran of nuclear norm} yields that
\begin{equation}
\left\Vert B\right\Vert _{*}\geq\left\Vert A\right\Vert _{*}+\left\Vert \mathcal{P}_{T_{A}^{\perp}}(B-A)\right\Vert -\left\Vert \mathcal{P}_{T_{A}}(B-A)\right\Vert ,
\end{equation}
which after rearrangement gives the desired inequality.
\end{claim}
\end{enumerate}
\end{proof}
The following result could be found in plenty of standard Banach space textbooks, see for example, \cite{brezisFunctionalAnalysisSobolev2011a}.
\begin{lem}
Let $f$ be a function with continuous second partial derivative defined
on an open convex set $U\in\mathbb{R}^{n}.$ Then for any $x$ and
$x_{0}\in U$, the following identity holds. 
\begin{equation}
f(x)=f(x_{0})+\left\langle \nabla f(x_{0}),(x-x_{0})\right\rangle +\frac{1}{2}(x-x_{0})^{T}\nabla_{f}^{2}(x_{0}+c(x-x_{0}))(x-x_{0})
\end{equation}
for $c\in(0,1).$
\end{lem}
\begin{lem}
Under\label{lem:bregman-frob-ineq}  \ref{assu:curvature}, it follows
that 
\begin{equation}
L_{\gamma}^{2}(x-y)\leq2d_{A}^{\tau}(x,y)\leq U_{\gamma}^{2}(x-y)^{2}.
\end{equation}
\end{lem}
\begin{proof}
By definition 
\begin{equation}
d_{A}^{\tau}(x,y)=A(x)-A(y)-\left\langle \nabla A(y),x-y\right\rangle =\frac{1}{2}(x-y)^{2}\nabla_{A}^{2}(\xi)
\end{equation}
for some $\xi\in(x,y).$ Then by  \ref{assu:curvature}, we see that
the result holds. 
\end{proof}
\begin{prop}[Corollary 3.3 in \cite{bandeiraSharpNonasymptoticBounds2016}]
Let $W$ be the $n\times m$ rectangular matrix whose entries $W_{ij}$
are independent centered bounded random variables. Then there exists
a universal constant $c$ such that 
\begin{equation}
\mathbb{E}[\left\Vert W\right\Vert ]\leq c(\kappa_{1}\lor\kappa_{2}+\kappa_{*}\sqrt{\log(n\land m)}),
\end{equation}
where 
\begin{equation}
\kappa_{1}=\max_{i\in[n]}\sqrt{\sum_{j\in[m]}\mathbb{E}[W_{ij}^{2}]},\quad\kappa_{2}=\max_{j\in[m]}\sqrt{\sum_{i\in[n]}\mathbb{E}[W_{ij}^{2}]},\quad\text{and }\max_{(i,j)\in[n]\times[m]}\left|W_{ij}\right|.
\end{equation}
\end{prop}
\begin{lem}
Let $A,B\in\mathcal{B}_{\left\Vert \right\Vert _{\infty}}(\gamma).$
If \label{lem:max-schatten-norm-ineq}
\begin{equation}
\ell(A\vert Y)+\lambda_{*}\left\Vert A\right\Vert _{*}+\lambda_{\max}\left\Vert A\right\Vert _{\max}\leq\ell(B\vert Y)+\lambda_{*}\left\Vert B\right\Vert _{*}+\lambda_{\max}\left\Vert B\right\Vert _{\max},
\end{equation}
then the following inequalities hold:
\end{lem}
\begin{enumerate}
\item $\left\Vert \mathcal{P}_{B}^{\perp}(A-B)\right\Vert _{*}\leq3\left\Vert \mathcal{P}_{B}(A-B)\right\Vert +2\frac{\lambda_{\max}}{\lambda_{*}}\left\Vert A-B\right\Vert _{F}.$
\item $\left\Vert A-B\right\Vert _{*}=2\left(\sqrt{8\mathrm{rank}(B)}+\frac{\lambda_{\max}}{\lambda_{*}}\right)\left\Vert A-B\right\Vert _{F}$.
\end{enumerate}
\begin{proof}
\begin{enumerate}
\item First, note that 
\begin{equation}
\lambda_{\max}(\left\Vert A\right\Vert _{\infty}-\left\Vert B\right\Vert _{\infty})=-\lambda_{\max}(\left\Vert B\right\Vert _{\infty}-\left\Vert A\right\Vert _{\infty})\leq-\lambda_{\max}(\left\Vert B-A\right\Vert _{\max})
\end{equation}
Note that 
\begin{align}
\ell(B\vert Y)-\ell(A\vert Y) & \geq\lambda_{*}\left(\left\Vert A\right\Vert _{*}-\left\Vert B\right\Vert _{*}\right)+\lambda_{\max}\left(\left\Vert A\right\Vert _{\max}-\left\Vert B\right\Vert _{\max}\right)\\
 & \geq\lambda_{*}\left(\left\Vert \mathcal{P}_{B}^{\perp}(A-B)\right\Vert _{*}-\left\Vert \mathcal{P}_{B}(A-B)\right\Vert _{*}\right)+\lambda_{\max}\left(\left\Vert A\right\Vert _{\max}-\left\Vert B\right\Vert _{\max}\right)\\
 & \geq\lambda_{*}\left(\left\Vert \mathcal{P}_{B}^{\perp}(A-B)\right\Vert _{*}-\left\Vert \mathcal{P}_{B}(A-B)\right\Vert _{*}\right)-\lambda_{\max}\left|\left\Vert A\right\Vert _{\max}-\left\Vert B\right\Vert _{\max}\right|\label{eq:oracle_ineq1}\\
 & \geq\lambda_{*}\left(\left\Vert \mathcal{P}_{B}^{\perp}(A-B)\right\Vert _{*}-\left\Vert \mathcal{P}_{B}(A-B)\right\Vert _{*}\right)-\lambda_{\max}\left\Vert A-B\right\Vert _{F}.\label{eq:oracle_ineq2}
\end{align}
where \ref{eq:oracle_ineq1} follows from the fact that $\lambda_{\max}\geq0$
and the fact that $x\geq-\left|x\right|$ for any $x\in\mathbb{R}$
and \ref{eq:oracle_ineq2} follows by \ref{lem:max-fro-ineq}. On
the other hand, it follows from convexity that 
\begin{equation}
\ell(A\vert Y)\geq\ell(B\vert Y)+\left\langle \nabla\ell(B\vert Y),B-A\right\rangle \implies\ell(B\vert Y)-\ell(A|Y)\leq\left\langle \nabla\ell(B\vert A),B-A\right\rangle .
\end{equation}
An application of operator norm Cauchy Schwartz inequality yield and
\begin{equation}
\ell(B\vert Y)-\ell(A|Y)\leq\left\Vert \nabla\ell(B|A)\right\Vert \left\Vert B-A\right\Vert _{*}\leq\frac{\lambda_{*}}{2}\left\Vert B-A\right\Vert _{*},\label{eq:oracle_ineq3}
\end{equation}
where the second inequality follows from the assumption stated in
the theorem. Combining the \ref{eq:oracle_ineq2} and \ref{eq:oracle_ineq3},
we get that 
\begin{equation}
\lambda_{*}\left(\left\Vert \mathcal{P}_{B}^{\perp}(A-B)\right\Vert _{*}-\left\Vert \mathcal{P}_{B}(A-B)\right\Vert _{*}\right)-\lambda_{\max}\left\Vert A-B\right\Vert _{F}\leq\frac{\lambda_{*}}{2}\left\Vert B-A\right\Vert _{*}=,
\end{equation}
which by rearranging, becomes 
\begin{equation}
\frac{1}{2}\left\Vert \mathcal{P}_{B}^{\perp}(A-B)\right\Vert _{*}\leq\frac{3}{2}\left\Vert \mathcal{P}_{B}(A-B)\right\Vert _{*}+\frac{\lambda_{\max}}{\lambda_{*}}\left\Vert A-B\right\Vert _{F},
\end{equation}
which is equivalent to
\begin{equation}
\left\Vert \mathcal{P}_{B}^{\perp}(A-B)\right\Vert _{*}\leq3\left\Vert \mathcal{P}_{B}(A-B)\right\Vert _{*}+2\frac{\lambda_{\max}}{\lambda_{*}}\left\Vert A-B\right\Vert _{F},
\end{equation}

\item A direct application of part (1) yields that 
\begin{align}
\left\Vert A-B\right\Vert _{*} & =\left\Vert \mathcal{P}_{B}^{\perp}(A-B)\right\Vert _{*}+\left\Vert \mathcal{P}_{B}(A-B)\right\Vert _{*}\\
 & \leq3\left\Vert \mathcal{P}_{B}(A-B)\right\Vert _{*}+2\frac{\lambda_{\max}}{\lambda_{*}}\left\Vert A-B\right\Vert _{F}+\left\Vert \mathcal{P}_{B}(A-B)\right\Vert _{*}\\
 & \leq4\left\Vert \mathcal{P}_{B}(A-B)\right\Vert _{*}+2\frac{\lambda_{\max}}{\lambda_{*}}\left\Vert A-B\right\Vert _{F}\\
 & \leq4\sqrt{2\text{rank}(B)}\left\Vert A-B\right\Vert _{*}+2\frac{\lambda_{\max}}{\lambda_{*}}\left\Vert A-B\right\Vert _{F}\\
 & \leq4\sqrt{2\text{rank}(B)}\left\Vert A-B\right\Vert _{F}+2\frac{\lambda_{\max}}{\lambda_{*}}\left\Vert A-B\right\Vert _{F}\\
 & =2\left(\sqrt{8\text{rank}(B)}+\frac{\lambda_{\max}}{\lambda_{*}}\right)\left\Vert A-B\right\Vert _{F}
\end{align}
\end{enumerate}
\end{proof}
\begin{lem}
\label{lem:max-fro-ineq}We have $\left|\left\Vert A\right\Vert _{\max}-\left\Vert B\right\Vert _{\max}\right|\leq\left\Vert A-B\right\Vert _{F}.$
\end{lem}
\begin{proof}
Note that 
\begin{equation}
\left|\left\Vert A\right\Vert _{\max}-\left\Vert B\right\Vert _{\max}\right|\leq\left\Vert A-B\right\Vert _{\max}\leq\left\Vert A-B\right\Vert _{F}.
\end{equation}
Another way to see it is that $\left\Vert A-B\right\Vert _{\max}$
is the maximum of the $L^{2}$ row norms, where as $\left\Vert A-B\right\Vert _{F}$
is the sum of all rows' $L^{2}$ norms.
\end{proof}
\begin{lem}[Appendix A.1 in \cite{alayaCollectiveMatrixCompletion2019}]
\label{lem:klopp_key_lemma}Let $\beta=\frac{946\gamma^{2}\log(n_{1}+N_{2})}{pn_{1}N_{2}}.$
Then for all $\Xi\in\mathcal{K}(\beta,r),$ it follows that 
\begin{equation}
\left|\Delta^{2}(\Xi,\Theta)-\frac{1}{n_{1}N_{2}}\left\Vert \Xi-\Theta\right\Vert _{\Pi,F}^{2}\right|\leq\frac{\left\Vert \Xi-\Theta\right\Vert _{\Pi,F}^{2}}{2n_{1}N_{2}}+1392r\gamma^{2}(\mathbb{E}[\left\Vert \Sigma_{R}\right\Vert ])^{2}+\frac{5567\gamma^{2}}{n_{1}N_{2}}.
\end{equation}
\end{lem}
\begin{lem}[Lemma 2 in \cite{alayaCollectiveMatrixCompletion2019}]
Let Assumption 2 holds. Then then there exists an absolute constant
$c$ such that with probability $1-4/(n_{1}+N_{2})$, we have that
\begin{equation}
\left\Vert \ell(\Theta|Y)\right\Vert \leq c\left(\frac{(U_{\gamma}\cup K)(\sqrt{n_{1}\lor N_{2}}+(\log(n_{1}\lor N_{2}))^{3/2}}{n_{1}N_{2}}\right).
\end{equation}
\end{lem}
\begin{prop}
Let $A\in\mathrm{Mat}_{\mathbb{R}}(n_{1},n_{2}),$ then $\left\Vert A\right\Vert _{2\rightarrow\infty}=\max_{i=1}^{n_{1}}\left\Vert A^{*}e_{i}\right\Vert _{\ell_{2}},$
i.e. it is the maximum of the row $\ell_{2}$ norm of $A$.
\end{prop}
\begin{proof}
First, note that the equality clearly holds when $A=0.$ So without
loss of generality, we can assume that $A\neq0.$ Note that 

\begin{align}
\left\Vert A\right\Vert _{2,\infty} & =\sup_{\left\Vert x\right\Vert _{2}}\left\Vert Ax\right\Vert _{\infty}=\sup_{\left\Vert x\right\Vert _{2}}\max_{1\leq i\leq n_{1}}\left\langle Ax,e_{i}\right\rangle =\sup_{\left\Vert x\right\Vert _{2}=1}\max_{1\leq i\leq n_{1}}\left\langle x,A^{*}e_{i}\right\rangle \\
 & \leq\sup_{\left\Vert x\right\Vert _{2}=1}\max_{1\leq i\leq n_{1}}\left\Vert x\right\Vert _{2}\left\Vert A^{*}e_{i}\right\Vert _{2}=\max_{1\leq i\leq n_{1}}\left\Vert A^{*}e_{i}\right\Vert _{2}.
\end{align}
On the other hand, let $\dagger$ be row number of $A$ that has the
largest row $\ell_{2}$ norm (in case of duplicate, pick the first
one). In other words, $\dagger=\text{argmax}_{1\leq i\leq n_{1}}\left\Vert A_{i}^{*}e_{i}\right\Vert _{\ell_{2}}$.
Note 
\begin{equation}
\left\Vert A\right\Vert _{2,\infty}=\sup_{\left\Vert x\right\Vert _{2}}\max_{1\leq i\leq n_{1}}\left\langle Ax,e_{i}\right\rangle \geq\left\langle A\frac{A^{*}e_{\dagger}}{\left\Vert A^{*}e_{\dagger}\right\Vert },e_{\dagger}\right\rangle =\frac{1}{\left\Vert A^{*}e_{\dagger}\right\Vert }\left\langle A^{*}e_{\dagger},A^{*}e_{\dagger}\right\rangle =\left\Vert A^{*}e_{\dagger}\right\Vert _{\ell_{2}}=\max_{1\leq i\leq n_{1}}\left\Vert A_{i}^{*}e_{i}\right\Vert .
\end{equation}
And the proof is completed.
\end{proof}
\begin{lem}
Let $M\in\mathbb{R}^{n\times m},$ then it follows that 
\begin{equation}
\left\Vert M\right\Vert _{\max}\leq\left\Vert M\right\Vert _{2\rightarrow\infty}\leq\left\Vert M\right\Vert _{F}.
\end{equation}
\end{lem}
\begin{proof}
Since $\left\Vert M\right\Vert _{2\rightarrow\infty}$ is the maximum
of the row $\ell_{2}$ nor\textbackslash ms of $M$, and $\left\Vert M\right\Vert _{F}$
is the sum of all row $\ell_{2}$ norms, the inequality clearly holds
and it suffices to establish the first part of the inequality. Recall
from that by definition 
\begin{equation}
\left\Vert M\right\Vert _{\max}=\min_{U,V\text{ s.t.}\ M=UV^{T}}\left\Vert U\right\Vert _{2,\infty}\left\Vert V\right\Vert _{2,\infty}.
\end{equation}
Note that $M$ has a trivial decomposition $M=M\cdot I,$ where $I\in\mathbb{R}^{n\times n};$
it follows that 
\begin{equation}
\min_{U,V\text{ s.t.}\ M=UV^{T}}\left\Vert U\right\Vert _{2,\infty}\left\Vert V\right\Vert _{2,\infty}\leq\left\Vert M\right\Vert _{2,\infty}\left\Vert I\right\Vert _{2,\infty}=\left\Vert M\right\Vert _{2,\infty},
\end{equation}
and the result follows as desired.
\end{proof}
\begin{lem}[Lemma 3.3 in \cite{fangMaxnormOptimizationRobust2018}]
\label{lem:kkt helper}Consider the optimization problem
\begin{equation}
\min_{z\in\mathbb{R}^{d}}\beta\left\Vert z\right\Vert _{\infty}+\frac{1}{2}\left\Vert c-z\right\Vert _{2}^{2}.
\end{equation}
Assume that $c_{1}\geq c_{2}\geq\ldots\geq c_{d}\geq0.$ The solution
to the problem has the following closed form: 
\begin{equation}
z^{*}=(t^{*},\ldots,t^{*},c_{k^{*}+1},\ldots,c_{d})^{T},
\end{equation}
where $t^{*}=\frac{1}{k^{*}}\sum_{i=1}^{k^{*}}(c_{i}-\beta)$ and
$k^{*}$ is the index such that $c_{k^{*}+1}<\frac{1}{k^{*}}(\sum_{i=1}^{k}c_{i}-\beta)\leq c_{k^{*}}.$
If no such $k^{*}$ exists, then $z^{*}=(t^{*},\ldots,t^{*})^{T},$
where $t^{*}=\frac{1}{d}\sum_{i=1}^{d}(c_{i}-\beta).$
\end{lem}
\begin{lem}[Negative Binomial Moments]
\label{lem:nb-moments} Let X be a random variable such that $X\sim\mathrm{NB}(r,p).$
Then 
\end{lem}
\begin{proof}
There are many ways to prove this fact. Here we use the standard factorial
moment trick. Note that 
\begin{align}
\mathbb{E}[X] & =\sum_{k=0}^{\infty}k\frac{\Gamma(k+r)}{k!\Gamma(r)}p^{r}(1-p)^{k}\\
 & =(1-p)\sum_{k=1}^{\infty}\frac{\Gamma(k+r)}{(k-1)!\Gamma(r)}p^{r}(1-p)^{k-1}\\
 & =(1-p)\sum_{j=0}^{\infty}\frac{\Gamma(j+1+r)}{j!\Gamma(r)}p^{r}(1-p)^{j}\\
 & =(1-p)\sum_{j=0}^{\infty}\frac{\Gamma(j+r)(j+r)}{j!\Gamma(r)}p^{r}(1-p)^{j}\\
 & =(1-p)\left[j\left(\sum_{j=0}^{\infty}\frac{\Gamma(j+r)}{j!\Gamma(r)}p^{r}(1-p)^{j}\right)+r\left(\sum_{j=0}^{\infty}\frac{\Gamma(j+r)}{j!\Gamma(r)}p^{r}(1-p)^{j}\right)\right]\\
 & =(1-p)(\mathbb{E}[X]+r),
\end{align}
Solving yields $\mathbb{E}[X]=\frac{r(1-p)}{p}.$ To calculate
$\mathrm{Var}(X),$ we first compute 
\begin{align}
\mathbb{E}[X^{2}] & =\sum_{k=0}^{\infty}k^{2}\frac{\Gamma(k+r)}{k!\Gamma(r)}p^{r}(1-p)^{k}\\
 & =\sum_{k=0}^{\infty}[k(k-1)+k]\frac{\Gamma(k+r)}{k!\Gamma(r)}p^{r}(1-p)^{k}\\
 & =\sum_{k=0}^{\infty}k(k-1)\frac{\Gamma(k+r)}{k!\Gamma(r)}p^{r}(1-p)^{k}+\sum_{k=0}^{\infty}k\frac{\Gamma(k+r)}{k!\Gamma(r)}p^{r}(1-p)^{k}\\
 & =(1-p)^{2}\sum_{j=0}^{\infty}\frac{\Gamma(j+2+r)}{j!\Gamma(r)}p^{r}(1-p)^{j}+\mathbb{E}[X]\\
 & =(1-p)^{2}\sum_{j=0}^{\infty}\frac{\Gamma(j+r)(j+1+r)(j+r)}{j!\Gamma(r)}p^{r}(1-p)^{j}+\mathbb{E}[X]\\
 & =(1-p)^{2}\left[\sum_{j=0}^{\infty}(j^{2}+r^{2}+2jr+j+r)\frac{\Gamma(j+r)}{j!\Gamma(r)}p^{r}(1-p)^{j}\right]+\mathbb{E}[X]\\
 & =(1-p)^{2}(\mathbb{E}[X^{2}]+r^{2}+2r\mathbb{E}[X]+\mathbb{E}[X]+r)+\mathbb{E}[X]\\
 & =(1-p)^{2}\mathbb{E}[X^{2}]+[2r(1-p)^{2}+(1-p)^{2}+1]\mathbb{E}[X]+(1-p)^{2}(r^{2}-r),
\end{align}
which after arrangement and some bit of algebra yields that $\mathbb{E}[X^{2}]=\frac{r(p^{2}r-2pr-p+r+1)}{p^{2}.}.$
As a result, we have that 
\begin{equation}
\mathrm{Var}(X)=\mathbb{E}[X^{2}]-(\mathbb{E}[X])^{2}=\frac{r(p^{2}r-2pr-p+r+1)}{p^{2}.}-\frac{r^{2}(1-p)^{2}}{p^{2}}=\frac{r(p-1)}{p^{2}}.
\end{equation}
\end{proof}
\begin{lem}[Negative Binomial mean parametrization]
\begin{doublespace}
 \label{lem:mean-parametrization-negative-binomial}Alternatively,
we can parametrize by its mean in the following way: a random variable
$X$ is a negative binomial random variable with mean $\mu$ and number
of success $r$ if and only if it has the following p.m.f 
\begin{equation}
\mathbb{P}(X=k)=\frac{\Gamma(k+r)}{\Gamma(r)k!}\left(\frac{r}{\mu+r}\right)^{r}\left(\frac{\mu}{\mu+r}\right)^{k}.
\end{equation}
\end{doublespace}

\vspace{-\parskip}
\end{lem}
\begin{proof}
This could be directly verified using \ref{lem:nb-moments}. Alternatively,
one could rewrite \ref{eq:extended-nb-pmf} in its exponential family
canonical form and invoke using the gradient forward map properties. 
\end{proof}
\end{document}